\documentclass[twoside]{article}

\usepackage[accepted]{aistats2025}

\usepackage{graphicx} 
\usepackage[round]{natbib}

\usepackage{titlesec}
\usepackage{graphicx} 
\usepackage[greek,british]{babel}
\usepackage[usenames,dvipsnames]{xcolor} 
\usepackage[utf8]{inputenc}
\usepackage[T1]{fontenc}
\usepackage{nicefrac}
\usepackage{microtype}     
\usepackage{algorithm}
\usepackage{algorithmic}
\usepackage{amsfonts,amsmath,amssymb,amsthm}
\usepackage{mathtools}
\usepackage{hyperref}
\usepackage[normalem]{ulem}
\usepackage[capitalise,nameinlink]{cleveref}
\usepackage{tikz}
\usepackage{enumitem}

\hypersetup{
    colorlinks = true,
    linkcolor = RedViolet,
    citecolor=NavyBlue
    }

\DeclareMathOperator*{\argmax}{arg\,max}
\DeclareMathOperator*{\argmin}{arg\,min}

\DeclareMathOperator*{\arginf}{arg\,inf}

\newtheorem{theorem}{Theorem}
\newtheorem{lemma}{Lemma}
\newtheorem{corollary}{Corollary}

\newtheorem{proposition}{Proposition}

\newtheorem{definition}{Definition}

\newtheorem{assumption}{Assumption}

\newtheorem{example}{Example}[section]

\makeatletter
\newtheoremstyle{case}
  {} 
  {} 
  {} 
  {} 
  {\bfseries} 
  {.} 
  { } 
  {} 
\makeatother
\theoremstyle{case}
\newtheorem{case}{Case}

\definecolor{salmon}{RGB}{250,128,114}
\definecolor{ballblue}{rgb}{0.13, 0.67, 0.8}

\newcommand{\refalgbyname}[2]{\hyperref[#1]{\texttt{\textbf{#2}}}}
\newcommand{\algoname}{\refalgbyname{alg:TaSFG}{TaS-FG}}

\title{Pure Exploration with Feedback Graphs}
\author{Alessio Russo}
\date{September 2024}

\begin{document}
\twocolumn[

\aistatstitle{Pure Exploration with Feedback Graphs}

\aistatsauthor{ Alessio Russo\textsuperscript{1} \And Yichen Song\textsuperscript{1} \And  Aldo Pacchiano\textsuperscript{1,2}}

\aistatsaddress{ \textsuperscript{1}Boston University\quad \textsuperscript{2} Broad Institute of MIT and Harvard}]
\begin{abstract}
We study the sample complexity of pure exploration in an online learning problem with a feedback graph. This graph dictates the feedback available to the learner, covering scenarios between full-information, pure bandit feedback, and settings with no feedback on the chosen action. While variants of this problem have been investigated for regret minimization, no prior work has addressed the pure exploration setting, which is the focus of our study. We derive an instance-specific lower bound on the sample complexity of learning the best action with fixed confidence, even when the feedback graph is unknown and stochastic, and present unidentifiability results for Bernoulli rewards. Additionally, our findings reveal how the sample complexity scales with key graph-dependent quantities. Lastly, we introduce \algoname{} (Track and Stop for Feedback Graphs), an asymptotically optimal algorithm, and demonstrate its efficiency across different graph configurations.
\end{abstract}

\section{INTRODUCTION}
\begin{figure*}[t]
    \centering

\begin{minipage}[t]{0.19\textwidth}
\centering
\begin{tikzpicture}[
  vertex/.style = {circle, draw, minimum size=0.3cm},
  loop/.style = {looseness=3, in=60, out=120, min distance=7mm},
  every edge/.style = {draw, thick}
]

\node[vertex] (A1) at (0,0) {A};
\node[vertex] (B1) at (1,-0.5) {B};
\node[vertex] (C1) at (0.5,-1.5) {C};
\node[vertex] (D1) at (-0.5,-1.5) {D};
\node[vertex] (E1) at (-1,-0.5) {E};

\foreach \v in {A1,B1,C1,D1,E1}
  \draw[->, loop] (\v) to (\v);
\end{tikzpicture}

    \label{fig:example_graph_bandit}
\end{minipage}
\hfill
\begin{minipage}[t]{0.19\textwidth}
\centering
\vspace{-2.1cm}
\begin{tikzpicture}[
  vertex/.style = {circle, draw, minimum size=0.3cm},
  loop/.style = {looseness=3, in=60, out=120, min distance=7mm},
  every edge/.style = {draw, thick},
]
\node[vertex] (A1) at (0,0) {A};
\node[vertex] (B1) at (1,0) {B};

\foreach \v in {B1}
  \draw[->] (A1) -- (\v);
\draw[->, loop] (A1) to (A1);
\end{tikzpicture}
    \label{fig:example_apple_tasting}
\end{minipage}
\hfill
\begin{minipage}[t]{0.19\textwidth}
\centering
\begin{tikzpicture}[
  vertex/.style = {circle, draw, minimum size=0.3cm},
  loop/.style = {looseness=3, in=60, out=120, min distance=7mm},
  every edge/.style = {draw, thick}
]
\node[vertex] (A1) at (0,0) {A};
\node[vertex] (B1) at (1,-0.5) {B};
\node[vertex] (C1) at (0.5,-1.5) {C};
\node[vertex] (D1) at (-0.5,-1.5) {D};
\node[vertex] (E1) at (-1,-0.5) {E};

\foreach \v in {B1,C1,D1,E1}
  \draw[->] (A1) -- (\v);
\draw[->, loop] (A1) to (A1);
\end{tikzpicture}
    \label{fig:example_graph_2}
\end{minipage}
\hfill
\begin{minipage}[t]{0.19\textwidth}
\centering
\begin{tikzpicture}[
  vertex/.style = {circle, draw, minimum size=0.3cm},
  loop/.style = {looseness=3, in=60, out=120, min distance=7mm},
  every edge/.style = {draw, thick}
]
\node[vertex] (A1) at (0,0) {A};
\node[vertex] (B1) at (1,-0.5) {B};
\node[vertex] (C1) at (0.5,-1.5) {C};
\node[vertex] (D1) at (-0.5,-1.5) {D};
\node[vertex] (E1) at (-1,-0.5) {E};

\draw[->] (A1) to (B1);
\draw[->] (A1) to (E1);

\draw[->] (B1) to (A1);
\draw[->] (B1) to (C1);

\draw[->] (C1) to (B1);
\draw[->] (C1) to (D1);

\draw[->] (D1) to (C1);
\draw[->] (D1) to (E1);

\draw[->] (E1) to (D1);
\draw[->] (E1) to (A1);

\end{tikzpicture}
    \label{fig:example_ring_graph}
\end{minipage}
\hfill
\begin{minipage}[t]{0.19\textwidth}
\centering
\begin{tikzpicture}[
  vertex/.style = {circle, draw, minimum size=0.3cm},
  loop/.style = {looseness=3, in=60, out=120, min distance=7mm},
  every edge/.style = {draw, thick}
]
\node[vertex] (A1) at (0,0) {A};
\node[vertex] (B1) at (1,-0.5) {B};
\node[vertex] (C1) at (0.5,-1.5) {C};
\node[vertex] (D1) at (-0.5,-1.5) {D};
\node[vertex] (E1) at (-1,-0.5) {E};

\foreach \v in {B1,C1,D1,E1}
  \draw[->] (A1) -- (\v);

\foreach \v in {A1,C1,D1,E1}
  \draw[->] (B1) -- (\v);

\foreach \v in {B1,C1,D1,A1}
  \draw[->] (E1) -- (\v);

\foreach \v in {B1,E1,D1,A1}
  \draw[->] (C1) -- (\v);

\foreach \v in {B1,E1,C1,A1}
  \draw[->] (D1) -- (\v);

\end{tikzpicture}
    \label{fig:example_graph_full_feedback}
\end{minipage}

    \caption{Examples of feedback graphs. From left to right: (1) bandit feedback; (2) apple tasting; (3) revealing action; (4) ring; (5) loopless clique.}
    \label{fig:example_graphs}
\end{figure*}
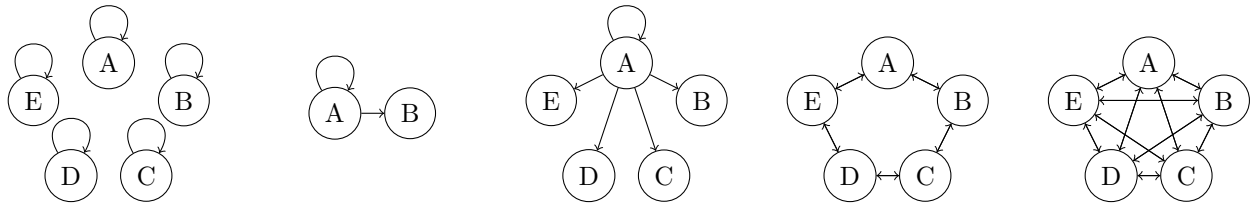
Online learning in a stochastic scenario is a sequential decision making problem in which, at each round, the learner chooses one action (arm) out of a finite set of actions, and observes some feedback depending on the setting \citep{lattimore2020bandit}. In the bandit setting the learner observes  a random reward distributed according to the distribution of the corresponding action \citep{robbins1952some}, while in the full information setting the random rewards of all the arms are observed \citep{littlestone1994weighted}.

Feedback graphs \citep{mannor2011bandits} extend these two settings by defining what type of feedback is available to the learner, effectively bridging scenarios between full-information and pure bandit feedback. The feedback is specified by a directed graph $G$, with the actions as its vertices, and the edges defining what feedback is revealed to the learner when an action is chosen. 

Feedback graphs for online learning have been extensively investigated in the  context of  regret minimization, a framework used to minimize the hindsight loss relative to optimal strategies \citep{lattimore2020bandit}. Under this framework, the problem has been studied in the case where the learner is \emph{self-aware} or not self-aware.

The self-aware scenario is a specific case where  all the vertices have self-loops, and therefore the learner can observe the reward of the chosen action \citep{mannor2011bandits, alon2017nonstochastic,arora2019bandits,lykouris2020feedback,rouyer2022near,marinov2022stochastic,kocak2023online}. On the other hand, there are some problems where the learner is not self-aware, such as in the apple tasting problem \citep{helmbold2000apple}. Another closely-related problem where the learner is not self-aware is the \emph{revealing action} problem \citep{cesa2006prediction}, in which there exists a special action that allows  the learner to observe full feedback, while  other actions have no feedback. 
This more general  setting has also been studied under the regret minimization framework  for both the  stochastic and adversarial regimes \citep{alon2015online,cohen2016online,chen2021understanding,kong2022simultaneously,eldowa2023minimax,zhang2024efficient,zhang2024practical}.

Feedback graphs have also been analyzed under the \emph{informed} and the \emph{uninformed} settings \citep{alon2015online}. In the former case the graph is revealed to the learner prior to each decision \citep{arora2019bandits,marinov2022stochastic,zhang2024practical}, while  in the latter (and harder) case the graph is unknown at decision time \citep{alon2017nonstochastic,zhang2024efficient}.

 Several recent works have considered bandits with stochastic feedback graphs. For instance, \citep{cortes2020online} and \citep{esposito2022learning} study regret minimization under  stochastic feedback graphs. \citep{dai2024can} study how stochastic feedback graphs drive user impacts in online platforms.

Prior work, to the best of our knowledge, has mostly focused  on regret minimization, whereas the pure exploration problem with fixed confidence \citep{paulson1964sequential, bechhofer1968sequential,bubeck2011pure,kaufmann2016complexity}, where the goal is to determine the best action with  a given confidence level $\delta$, has not been thoroughly  investigated in the feedback graph setting.  In \citep{du2021combinatorial} the authors investigate the pure‐exploration problem in combinatorial bandits with partial‐linear feedback. However, while their approach is quite general and can be applied to informed settings, it does not exploit the structure of the graph to drive exploration.  \citet{chen2024interpolating} study deterministic graphs where the arms are partitioned into $T$ different groups, such that the
pull of an arm results in an observation of all the arms in its group. In this setting, they prove a $(\epsilon,0.05)$-PAC lower bound of order $O\left(\sum_{i=1}^T \log(m_i+1)/\epsilon^2\right)$, where $m_i$ is the size of the $i$-group. Nonetheless, their work does not address the more general pure exploration problem with stochastic feedback graphs.

To address this gap in the literature, we study the pure exploration problem for general stochastic graphs, which may or may not be self-aware. In our model, each edge in the graph has an associated probability of providing feedback, introducing uncertainty about whether feedback is observed from a particular edge. Additionally, we study  both the informed scenario, where the graph structure is known and the learner knows the set of edges that were activated, and the uninformed scenarios, where the graph is unknown and the learner does not know which edge has been activated \footnote{ We note that in the adversarial literature, the terms \emph{informed} and \emph{uninformed} are  used in a slightly different way--typically, they distinguish between observing the feedback graph before versus after making a decision, respectively. }.

Using tools from Best Arm Identification (BAI) \citep{garivier2016optimal}, we derive an instance-specific lower bound on the sample complexity of learning the best action with fixed confidence $\delta\in (0,1)$, and present unidentifiability results for Bernoulli rewards in the uninformed case.  Additionally, our findings reveal how the sample complexity scales with key graph-dependent quantities, such as the independence number and the number of self-loops in a graph. Lastly, we introduce \algoname{} (Track and Stop for Feedback Graphs), an asymptotically optimal algorithm, and demonstrate its efficiency across different graph configurations\footnote{Code repository: \url{https://github.com/rssalessio/Pure-Exploration-with-Feedback-Graphs}}.

\section{PROBLEM SETTING}
We now briefly explain the graph structure, and some graph-specific quantities, and then explain the BAI setting with feedback graphs.

\subsection{Graphs} We indicate by $G$ a generic directed graph with vertices $V=[K]$, where $[K]=\{1,\dots, K\}$ denotes the set of the first $K$ integers. We denote by $G\in [0,1]^{K\times K}$ the matrix of weights of the edges, where $G_{u,v}$ indicates the weight of the edge $(u,v)\in V^2$. Therefore, $E=\{(u,v)\in V^2: G_{u,v}>0\}$ is the set of edges in $G$.

 For every $v\in V$ let $N_{in}(v) = \{v'\in V: (v',v)\in E\}$ be the \emph{in-neighborhood} of $v$ in $G$. Similarly, we  let $N_{out}(v) = \{v'\in V: (v,v')\in E\}$ be the \emph{out-neighborhood} of $v$ in $G$. If a vertex $v$ has a self-loop, that is $(v,v)\in E$, then $v\in N_{in}(v)$ and $v\in N_{out}(v)$.  We now define the concept of observability.
\begin{definition}[Graph observability]
    In a directed graph $G$ with vertices $V$ we say that $v\in V$ is observable if $N_{in}(v)\neq \emptyset$ . A vertex $v\in V$ is strongly observable if $\{v\}  \subseteq N_{in}(v)$ , or $ V\setminus\{v\}  \subseteq  N_{in}(v)$ or both conditions hold. A vertex is weakly observable if it is observable but not strongly.
We also let $ W(G), SO(G)$ be, respectively, the set of  weakly observable and strongly observable vertices. Lastly, a graph $G$ is observable (resp. strongly observable) if all its vertices are observable (resp. strongly observable). A graph is weakly observable if it is observable but not strongly.
\end{definition}
We also define the notion of \emph{domination} and of \emph{set independence.}
\begin{definition}[Domination and set independence]
In a directed graph $G$ with vertices $V$ we say that $D\subseteq V$ dominates $W \subseteq V$ (and we write $W\ll D$)  if for any $w\in W$ there exists $d\in D$ such that $w\in N_{out}(d)$.
We also say that $I \subseteq V$ is an independent set if it is a collection of vertices with no edges connecting any pair of them.
\end{definition}

Lastly, we define some graph-dependent quantities.
\begin{definition}[Graph-dependent quantities]\label{def:graph_dependent_quantities}
In a directed graph $G$ with vertices $V$ we let:
\begin{itemize}
    \item $\delta(G)=|D(G)|$ be the weak domination number of $G$, which is the size of the smallest set $D(G)\in \argmin_{D\subseteq V: W(G)\ll D} |D|$ that dominates the set of weakly observable vertices $W(G)$. 
    \item  $\alpha(G)=|I(G)|$ be the independence number of $G$, which is the largest possible number of vertices in an independent set of $G$. Formally, $I(G)\in{\cal I}(G)\coloneqq  \argmax_{I\subseteq V: \forall (u,v)\in I^2,u\neq v, G_{u,v}=0} |I|$.
    \item $\sigma(G)= |L(G)|$ be the  number of vertices with a self-loop, where  $L(G)=\{v\in V: \{v\} \subseteq N_{in}(v)\}$ is the set of vertices with a self-loop.
\end{itemize}
\end{definition}
\begin{example}
    In \cref{fig:example_graphs} we have that: (1) the bandit feedback is strongly observable, with $\alpha(G)=5$; (2) apple tasting is also strongly observable, with $\alpha(G)=1$; (3) revealing action is weakly observable, with $\delta(G)=1, \alpha(G)=4$; (4) ring graph, which is weakly observable, with $\delta(G)=3, \alpha(G)=2$; (5) loopless clique is strongly observable, with $\alpha(G)=1$. In addition to the graphs in \cref{fig:example_graphs}, we also have that: (6) the union of bandit feedback and revealing action is called \emph{loopy star}, a strongly observable graph with $\alpha(G)=1$; (7) the full feedback graph is strongly observable with $\alpha(G)=1$.
\end{example}

\subsection{Best Arm Identification with Feedback Graphs}\label{subsec:bai_setting}
We consider models $\nu=\{G,(\nu_u)_{u\in V}\}$ that consist of a graph $G\in [0,1]^{K\times K}$, with $|V|=K$ vertices, where each vertex $u\in V$  is characterised by $\nu_u$, a   probability distribution with average value $\mu_u$. 

To investigate the problem, without loss of generality, we analyse the particular case of distributions that belong to a canonical exponential family with one parameter \citep{efron2022exponential}:
\begin{equation}
    \frac{{\rm d}\nu_u}{{\rm d} \rho}(x) = \exp(\theta_u x- b(\theta_u)) \eqqcolon f_u(x),
\end{equation}
 where $f_u$ is the associated density, $\theta_u\in\Theta\subset \mathbb{R}$ is the canonical  parameter (which satisfies $\eta_u=\theta(\mu_u)$ for some mapping $\theta:\mathbb{R}\to\Theta$), $\rho$ is some dominating measure and $b:\Theta\to\mathbb{R}$ is a convex, twice-differentiable function. The mean of a distribution is denoted by $\dot b(\theta)\coloneqq \frac{{\rm d}b}{{\rm d}\theta}(\theta)$, which satisfies $\dot b(\theta_u)=\mu_u$. This class of distribution includes the Binomial distribution with $n$ samples, the Poisson distribution, Gaussians with known variance, and others (see also \citep{cappe2013kullback, efron2022exponential} for more details). This assumption has been previously used in pure exploration problems, see for example \citep{garivier2016optimal,degenne2019pure,degenne2019non}.

In the following, with some abuse of notation, we interchangeably write $\nu=\{G, (\mu_u)_{u\in V}\}$, in the sense that a model is identified with the graph $G$ and the mean value of the arms $(\mu_u)_{u\in V}$.

\paragraph{Settings.} At each time-step $t=1,2,\dots$, the agent chooses a vertex $V_t\in V$, and observes a collection of independent draws $Z_t\coloneqq\{Z_{t,u}\}_{u\in V}$, where $Z_{t,u}=Y_{t,(V_t,u)} R_{t,u}$, with $Y_{t,(V_t,u)}$ being distributed as ${\rm Ber}(G_{V_t,u})$ and the reward $R_{t,u}$  is drawn from  $ \nu_u$ \footnote{When both random variables follow Bernoulli distribution, one can  expect unidentifiability issues if the graph is unknown.}. We focus on the \emph{uninformed} and \emph{informed} settings.
\begin{definition}[Uninformed setting]
    In the uninformed setting   the learner {does not know} the graph nor which edge is activated at each time-step $t$. In other words, at time $t$ after choosing $V_t=v$ the learner does not know $E_t\coloneqq \{u\in V:Y_{t, (V_t,u)}=1\}$.
\end{definition}
 A simpler case is the informed setting.
\begin{definition}[Informed setting]
    In the informed setting the learner at each time-step $t$ knows the graph or which edge was activated after choosing  $V_t=v$, i.e., the set  $E_t$ is revealed to the learner.
\end{definition}

We denote by $ \nu_{v,u}$ the product distribution of $Z_{t,u}$ when $V_t=v$. For simplicity, we also indicate by $Z_{v,u}$ an i.i.d. sample from $\nu_{v,u}$. Lastly, we indicate by $\mathbb{P}_\nu$ (resp. $\mathbb{E}_\nu$) the probability law under $\nu$ of the observed rewards.

The goal of the learner is to identify the reward associated to the \emph{best vertex} $a^\star(\nu)=\argmax_{u\in V} \mu_u$, also known as the \emph{best action} or \emph{best arm} (in the following we also write $a^\star$ whenever it's clear from the context). To ensure that the problem is well-defined, for simplicity, we require that  the graph is fully observable. 
\begin{assumption}\label{assump:nontrivial_problem}
    We assume  $\nu$ to be observable.
\end{assumption}
We remark that the extension to a non-fully observable graph is straightforward, as long as there are $2$ observable vertices (otherwise the problem is ill-defined). Observe that we do not require to have self-loops in the model compared to some of the  previous settings  studied in the literature \citep{kocak2023online,mannor2011bandits}.

\paragraph{Algorithm and objective.} We consider a broad class of algorithms for the learner that consist of:
\begin{enumerate}
    \item \emph{a sampling rule}, which determines, based on past observations, which  vertex  is chosen at
time $t$; that is, $V_t$ is ${\cal F}_{t-1}$-measurable, with ${\cal F}_t=\sigma(V_1,Z_1,V_2,\dots, V_t, Z_t)$ in the uninformed case, while in the informed one ${\cal F}_t = \sigma(V_1,E_1, Z_1, \dots, V_t,E_t, Z_t)$.
\item \emph{a stopping rule} $\tau$ that stops the algorithm when sufficient evidence has been gathered to identify the optimal vertex. It is a stopping time with respect to $({\cal F}_t)_t$  satisfying $\mathbb{P}_\nu(\tau<\infty)=1$.
\item  \emph{a recommendation rule} $\hat a_\tau \in V$ that returns the estimated optimal vertex, and $\hat a_\tau$ is a ${\cal F}_{\tau}$-measurable random variable.
\end{enumerate}
We focus on the fixed-confidence setting, with a risk parameter\footnote{\label{footnote_remark}Note  that $\delta$ is the confidence parameter, while  $\delta(G)$ is the weak domination parameter.} $\delta\in (0,1)$, which entails devising an probably-correct estimator of the best vertex.
\begin{definition}[$\delta$-PC Algorithm]
    We say that an algorithm ${\tt Alg}$ is $\delta$-PC (Probably Correct) if, for any model $\nu$ satisfying \cref{assump:nontrivial_problem}, we have $\mathbb{P}_\nu(\tau<\infty, \hat a_\tau \neq a^\star(\nu)) < \delta$.
\end{definition}

The goal in this setting is to obtain a $\delta$-PC algorithm that requires, on average, the minimum number of draws $\mathbb{E}_\nu[\tau]$. Therefore we study the minimum achievable sample complexity $\mathbb{E}_\nu[\tau]$ by any $\delta$-PC algorithm.

\paragraph{Notation.} In the following, we denote by ${\rm KL}(P,Q)$ the KL-divergence between two distributions $P$ and $Q$, and by ${\rm kl}(x,y)=x\ln(x/y) + (1-x)\ln((1-x)/(1-y))$ the Bernoulli KL-divergence between two Bernoulli distributions of parameters $x$ and $y$ respectively.
For distributions  $P,Q$ belonging to the canonical exponential family with one parameter, with canonical parameters $\theta,\theta'$ respectively, we have ${\rm KL}(P,Q)=b(\theta')-b(\theta)-\dot{b}(\theta)(\theta'-\theta)$.
We also define a generalized version of the Jensen-Shannon divergence as $I_\alpha(P,Q)=\alpha {\rm KL}(P, \alpha P+ (1-\alpha)Q) + (1-\alpha) {\rm KL}(Q,\alpha P + (1-\alpha) Q)$  with $\alpha\in [0,1]$.
The sub-optimality gap in a vertex $u\in V$ is defined as $\Delta_u\coloneqq \mu_{a^\star}-\mu_u$, and the minimum gap as $\Delta_{\rm min}=\min_{a\neq a^\star} \Delta_a$.

Finally, we let $N_v(t)$  be the number of times a vertex $v\in V$ has been chosen up to time-step $t$ by ${\tt Alg}$ (to not be confused with the in/out-neighborhoods $N_{in}(v)$ and  $N_{out}(v)$), thus $N_v(t)=N_v(t-1)+\mathbf{1}_{\{V_t=v\}}$, with $N_v(0)=0$.  We also indicate by $N_{v,u}(t)$  the number of times edge $(v,u)$ was activated after choosing $v$ up to time-step $t$, thus $N_{v,u}(t)=N_{v,u}(t-1)+\mathbf{1}_{\{V_t=v, Y_{t,(v,u)}=1\}}$ with $N_{v,u}(0)=0$. Similarly, we denote by $M_u(t)$ the number of times we observed a reward from vertex $u$ up to time-step $t$. Hence, one can write  $M_u(t) = \sum_{v\in V} N_{v,u}(t)$.

 \section{SAMPLE COMPLEXITY LOWER BOUNDS}
The recipe to derive instance-specific sample complexity lower bounds is based on a \emph{change of measure} argument. This argument allows to derive an instance-dependent quantity $T^\star(\nu)$,  also known as \emph{characteristic time} (and its inverse $(T^\star(\nu))^{-1}$ is the \emph{information rate}), that  permits to lower bound the sample complexity of an algorithm. Change of measure arguments have a long history \citep{wald1947sequential,lorden1971procedures,lai1981asymptotic,lai1985asymptotically}, and have been applied to find  lower bounds for regret minimization \citep{combes2014unimodal,garivier2019explore} and best-arm identification  \citep{garivier2016optimal}.
We use this technique to derive the sample complexity lower bound in both the uninformed and informed settings.

\subsection{Lower Bound in the Uninformed Setting}
\label{subsec:lb_uninformed}
The uninformed case, while seemingly daunting, admits a separation in behavior between the class of \emph{continuous} and \emph{discrete} rewards.

We find that for Bernoulli rewards the best vertex is \emph{unidentifiable}, i.e.,  we cannot reject the null hypothesis that a certain vertex is optimal, no matter how much data is gathered. This result comes from the fact that an agent cannot discern between low-probability edges linked to high-reward vertices, and vice-versa. 

For the continuous case, however, we find a different behavior. The intuition is that an agent  can conclude, with almost sure certainty, that an edge is not activated if  zero reward is observed from that edge.

\subsubsection{Lower bound for continuous rewards in the uninformed setting}
 Recall the definition of the generalized Jensen-Shannon divergence $I_\alpha(P,Q)\coloneqq\alpha {\rm KL}(P, \alpha P+ (1-\alpha)Q) + (1-\alpha) {\rm KL}(Q,\alpha P + (1-\alpha) Q)$  for two distributions $P,Q$ and $\alpha\in [0,1]$.

For continuous rewards, we obtain the following instance-specific sample complexity lower bound, which is proved in \cref{subsec:proof_lb}.
 \begin{theorem}\label{thm:lb_general}
     For any $\delta$-PC algorithm and any model $\nu$  with reward distributions $\{\nu_u\}_{u\in V}$ with continuous support, satisfying \cref{assump:nontrivial_problem}, we have that
     \begin{equation}
         \mathbb{E}_\nu[\tau]\geq T^\star(\nu){\rm kl}(\delta,1-\delta),
     \end{equation}
     where 
          \begin{equation*}
    \begin{aligned}
   (T^\star(\nu))^{-1} =& \sup_{\omega\in \Delta(V)} \min_{u\neq a^\star} (m_u +m_{a^\star} )I_{\frac{m_{a^\star}}{m_u+m_{a^\star}}}(\nu_{a^\star},\nu_u)\\
   &\hbox{ s.t. } m_u = \sum_{v\in N_{in}(u)}\omega_v G_{v,u} \quad \forall u\in V.
    \end{aligned}
\end{equation*}
 \end{theorem}
In the following we also denote by $\omega^\star = \arginf_{\omega\in \Delta(V)} T(\omega;\nu)$ the optimal solution, where 
\begin{equation}
   \begin{aligned}
   T(\omega;\nu)^{-1}=& \min_{u\neq a^\star} (m_u +m_{a^\star} )I_{\frac{m_{a^\star}}{m_u+m_{a^\star}}}(\nu_{a^\star},\nu_u),\\
   &\hbox{ s.t. } m=G^\top \omega,
    \end{aligned}
\end{equation}
where $m=\begin{bmatrix}m_1&\dots& m_K\end{bmatrix}^\top$,  similarly $\omega$, are in vector form.

\paragraph{Discussion and scaling.} The characteristic time in \cref{thm:lb_general} displays some similarities to the  characteristic time found in classical BAI \citep{garivier2016optimal}. First, as one would expect, the amount of evidence $ (m_u +m_{a^\star} )I_{\frac{m_{a^\star}}{m_u+m_{a^\star}}}(\nu_{a^\star},\nu_u)$ does not depend directly on the vertex selection rate $\omega$, but on the observation rate $m$, which depends  on the edge activation probabilities.

Secondly, to gain a better intuition of the above quantities, we can focus on the Gaussian case where $\nu_u = {\cal N}(\mu_u,\lambda^2)$, with $\lambda> 0$. For this particular choice,  $T^\star(\nu)$ is the solution to the following convex problem:
          \begin{equation}
    \begin{aligned}
   T^\star(\nu) =&   \inf_{\omega\in \Delta(V)} \max_{u\neq a^\star}\left(m_u^{-1} + m_{a^\star}^{-1}\right) \frac{2\lambda^2}{\Delta_u^2}\hbox{ s.t. } m = G^\top \omega.
    \end{aligned}
\end{equation}
This expression  allows us to gain a better understanding of the scaling   of $T^\star(\nu)$, as shown in the next  two propositions (which are proved in \cref{app:subsec_scaling}).
\begin{proposition}
    \label{prop:scaling_weakly_observable_gaussian}
    Consider an observable model $\nu=(\{\nu_u\}_u, G)$ with  Gaussian  rewards $\nu_u={\cal N}(\mu_u, \lambda^2)$.
    If $\delta(G)+\sigma(G)>0$, we can upper bound $T^\star$ as
    \begin{equation}
        T^\star(\nu) \leq \frac{4\left(\delta(G)+\sigma(G)-\left\lfloor \frac{\sigma(G)}{\alpha(G)+1}\right \rfloor\right)\lambda^2}{\min_{u\neq a^\star} \min(\bar G_u, \bar G_{a^\star}) \Delta_u^2},
    \end{equation}
    where $\bar G_u$ for any $u\in V$ is defined as \[
    \bar G_u \coloneqq \max\left(  \max_{v\in D(G)} G_{v,u},\min_{v\in L(G): G_{v,u}>0}G_{v,u} \right ).\]
\end{proposition}
Instead, for the loopless clique, we obtain the following result.
\begin{proposition}\label{prop:scaling_looplessclique}
     For an observable  $\nu=(\{\nu_u\}_u, G)$ with  Gaussian  rewards $\nu_u={\cal N}(\mu_u, \lambda^2)$ satisfying $\delta(G)+\sigma(G)=0$ (i.e., the loopless clique), we have
    \begin{equation}
        T^\star(\nu) \leq \frac{4\bar G\lambda^2}{\Delta_{\rm min}^2},
    \end{equation}
    where $\bar G \coloneqq  \min_{v,w: v\neq w} \max_{u\neq a^\star}\frac{1}{G_{v,w}(u)} + \frac{1}{G_{v,w}(a^\star)} $ and $ G_{v,w}(u)\coloneqq G_{v,u} + G_{w,u}$.
\end{proposition}
For  both results we see how  the sample complexity does not  scale with $K$, but depends on the structural properties of $G$. For example,  \cref{prop:scaling_weakly_observable_gaussian} exhibits a scaling in $\sigma(G)$ for bandit feedback, and $\sigma(G)/2$ for  a complete graph (i.e., there is an edge between any pair of vertices) with self-loops. Nevertheless, we find it challenging to enhance this scaling without using additional graph-specific parameters.
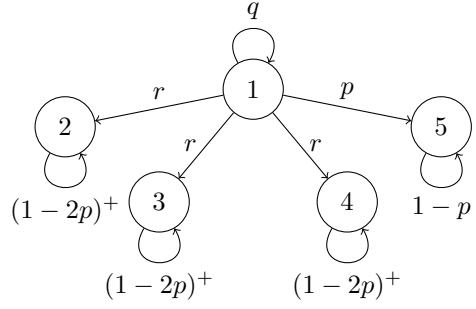
\begin{figure}[t]
    \centering

\begin{tikzpicture}[
  vertex/.style = {circle, draw, minimum size=0.8cm},
  loop/.style = {looseness=3, in=60, out=120, min distance=7mm},
  every edge/.style = {draw, thick}
]
\node[vertex] (A1) at (0,0) {$1$};
\node[vertex] (B1) at (2.5,-0.5) {$5$};
\node[vertex] (C1) at (1.25,-1.5) {$4$};
\node[vertex] (D1) at (-1.25,-1.5) {$3$};
\node[vertex] (E1) at (-2.5,-0.5) {$2$};

\draw[->] (A1) -- node[above] {$p$}  (B1);
\draw[->] (A1) -- node[right] {$r$}  (C1);
\draw[->] (A1) -- node[left] {$r$}  (D1);
\draw[->] (A1) -- node[above] {$r$}  (E1);

\draw[->, loop] (A1) to node[above] {$q$}  (A1);
\draw[->, loop] (B1) to [out=240,in=300] node[below] {$1-p$}  (B1);
\draw[->, loop] (C1) to [out=240,in=300]  node[below] {$(1-2p)^+$} (C1);
\draw[->, loop] (D1) to [out=240,in=300] node[below] {$(1-2p)^+$}  (D1);
\draw[->, loop] (E1) to [out=240,in=300] node[below] {$(1-2p)^+$}  (E1);
\end{tikzpicture}
    
    \caption{Loopy star graph. To each edge is associated an activation probability (obs. that $(x)^+ = \max(x,0)$). }
    \label{fig:loopy_star_graph}
\end{figure}
\begin{example}[The loopy star]
    We study the  graph in \cref{fig:loopy_star_graph} with Gaussian rewards, where $\lambda=1$, $\mu_5=1$ and   $\mu_u=0.5, u
\in \{1,
    \dots,4\}$. Notably, this graph  is the union of a  bandit feedback graph and revealing action graph. This example is relevant in adversarial regret minimization: removing any self-loop changes the minimax regret from $\tilde \Theta(\sqrt{\alpha(G) T})$ to $\tilde \Theta(T^{2/3})$ \citep{alon2015online}.
The graph depends on the parameters $(p,q,r)$, where the number of self-loops decreases as $p$ increases. 
In \cref{fig:example_characteristic_time_loopystar} we show the characteristic time $T^\star(\nu)$ (in solid lines) of the loopy star graph for different values of $(p,q)$ with $r=1/4$.  
 In particular, there is no sharp transition as in the regret minimization setting when removing self-loops: since  exploitation is unnecessary, selecting a less promising node is acceptable as long as it yields useful information.
Additionally, the only visible change happens at  $p\approx0.1$ (and $q=1$ is fixed), where it is no longer convenient for the algorithm to sample the vertices $\{2,3,4\}$, and focuses only on vertices $1$ and $5$. 
In the figure we also plot (in dashed lines) $\|G^\top \omega^\star\|_2$, an indication of the observation frequency of the vertices. We see how this quantity is directly correlated with the characteristic time $T^\star(\nu)$.
\end{example}

\subsubsection{A heuristic solution}\label{subsec:heuristic_sol}
In general, it is difficult to guess what the optimal solution $\omega^\star$ may be.  \cite{garivier2016optimal} show that $\omega_u \propto 1/\Delta_u^2$ is almost optimal for Gaussian rewards (with $\Delta_{a^\star} = \Delta_{\rm min}$) in multi-armed bandit problems. Is it the same also for  feedback graphs?

Taking inspiration from \cite{garivier2016optimal}, we propose that $m_u \propto  1/\Delta_u^2$, where $m=G^\top \omega$. 
Define then the vector $\Delta^{-2}\coloneqq (1/\Delta_u^2)_{u\in V}$. From a geometrical perspective, we can maximize the similarity $m^\top \Delta^{-2}$, or rather, $\omega^\top G \Delta^{-2}$.  In the classical Euclidean space this is achieved by $\omega \propto G \Delta^{-2}$.  To obtain a distribution, we project $G\Delta^{-2}$ to the nearest distribution $\omega_{\rm heur}$ in the KL sense, and obtain $\omega_{\rm heur} \coloneqq G\Delta^{-2}/\|G\Delta^{-2}\|_1$ (see also \cref{subsubsec:app_heuristic_sol}).

Such allocation $\omega_{\rm heur}$ makes intuitive  sense: the probability of selecting a vertex $u$ is proportional to $ \sum_{v\in V} G_{uv} \Delta_{v}^{-2}$, thus assigning higher preference to vertices that permit the learner to sample actions with small sub-optimality gaps.
\begin{figure}[t]
    \centering
    \includegraphics[width=\linewidth]{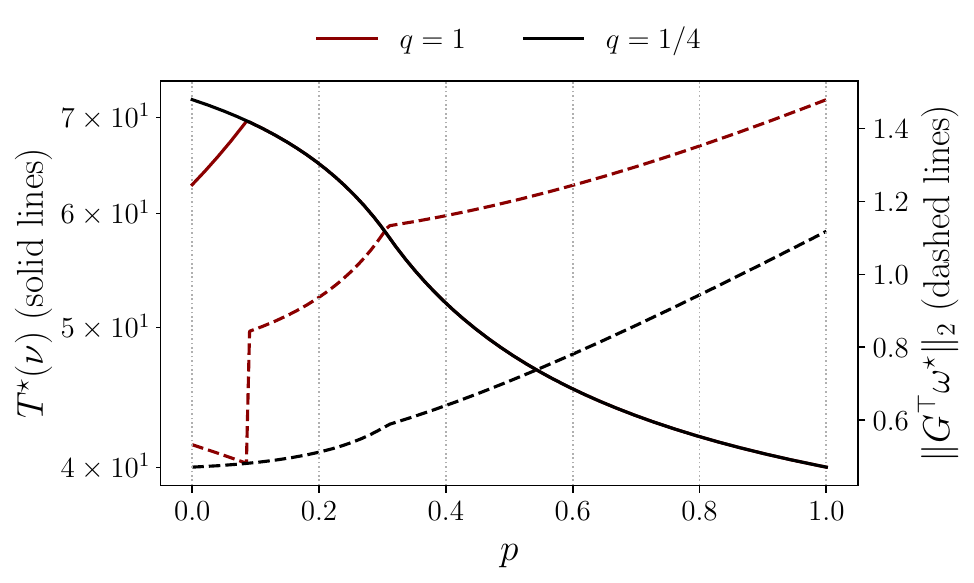}
    \caption{Loopy star example with $r=1/4$.
    The solid lines depict $T^\star(\nu)$ for $q=1$ and $q=1/4$ for different values of $p$. Similarly, on the right axis, the dashed lines show $\|G^\top \omega^\star\|_2$, which indicates the amount of information gathered per time-step.}
    \label{fig:example_characteristic_time_loopystar}
\end{figure}
 
 For this heuristic allocation $\omega_{\rm heur}$ we can provide the following upper bound on its scaling (we refer  the reader to \cref{subsubsec:app_heuristic_sol} for a proof).
\begin{proposition}\label{prop:scaling_heruistic}
    For an observable model $\nu=(\{\nu_u\}_u, G)$ with Gaussian random rewards $\nu_u={\cal N}(\mu_u, \lambda^2)$ we can upper bound $T(\omega_{\rm heur};\nu)$  as
    \begin{equation*}
        T^\star (\nu) \leq T(\omega_{\rm heur};\nu)\leq 4\lambda^2\frac{\|G\Delta^{-2}\|_1}{\sigma_{\min}(G)^2}  ,
    \end{equation*}
    where $\sigma_{\min}(G)$ is the minimum singular value of $G$.
\end{proposition}
Observing the following  upper bound $||G\Delta^{-2}||_1 \leq  K^{3/2}\sigma_{\max}(G)/\Delta_{\min}^2$ where $\sigma_{\rm max}(G)$ denotes the maximum singular value of $G$, we conclude that   in the worst case scenario $T(\omega_{\rm heur};\nu) \leq O\left( \frac{K^{3/2}\sigma_{\max}(G)}{\Delta_{\min}^2\sigma_{\min}(G)^2}\right)$. 
Observe that this scaling can be improved in $K$ for sparse solutions that efficiently use the graph structure. We discuss in \cref{subsubsec:app_heuristic_sol} such an approach that scales according to the minimum number of vertices that dominate the graph.

Another simple solution is the sparse allocation $\omega_u = \mathbf{1}_{\{u\in {\cal G}\}}/|{\cal G}|$, where ${\cal G}=\argmax_{v} (G\Delta^{-2})_v$. This is an efficient allocation, since it scales as $O\left(\frac{|{\cal G}|}{\Delta_{\min}^2 \max_{u\in {\cal G}}\min_v G_{u,v}} \right)$. However, such solution is admissible only if ${\cal G}$ dominates the graph, which ultimately depends on the structure of $G$.

\subsubsection{The curious case of Bernoulli rewards}
We conclude the study of sample complexity lower bounds in the uninformed setting by examining  the case of Bernoulli rewards, finding that it is generally impossible to estimate the best vertex in this case.  If the learner does not know whether the graph activations are deterministic and the rewards follow a Bernoulli distribution, a zero outcome could arise either because the reward of the chosen arm is truly zero or because the edge responsible for providing feedback was not activated.  Without knowing which edge was activated,  the learner cannot discern between the randomness of the reward and that of the edge. This intuition is formalized in the following result, which shows that the information rate is zero (the proof is given in \cref{subsubsec:sample_complexity_discrete_case}).
\begin{proposition}\label{prop:unidentifiability_bernoulli}
    Under \cref{assump:nontrivial_problem}, if $(\nu_u)_{u\in V}$ are Bernoulli distributions with  parameters $(\mu_u)_{u\in V}$, then $a^\star$ is unidentifiable, in the sense that $(T^\star(\nu))^{-1}=0$.
\end{proposition}

\subsection{Lower Bound in the Informed Setting}
When the set of activated edges $E_t$ is revealed to the learner at each time-step $t$, we can show that the lower bound in \cref{thm:lb_general} also holds for Bernoulli rewards. Furthermore, the result also applies to the case where the graph probabilities $G$ are known a-priori.
\begin{theorem}
\label{thm:lb_informed}
        Consider a $\delta$-PC algorithm and a model $\nu$  satisfying \cref{assump:nontrivial_problem}. If  the set $E_t$ is revealed to the learner for each $t\geq1$ (or $G$ is known), then
     \begin{equation}
         \mathbb{E}_\nu[\tau]\geq T^\star(\nu){\rm kl}(\delta,1-\delta),
     \end{equation}
     where $T^\star(\nu)$ is  as in \cref{thm:lb_general}.
\end{theorem}
 See the proof in \cref{subsubsec:informed_setting_appendix}. Henceforth,  similar remarks from the uninformed case also apply here.

\section{\algoname{} ALGORITHM}\label{subsec:algorithm}
In this section we propose \algoname{}  (Track and Stop for Feedback Graphs), an algorithm inspired by Track and Stop (TaS, \cite{garivier2016optimal}) that is  asymptotically optimal as $\delta \to 0$ with respect to $T^\star(\nu)$ for both the \emph{informed} and \emph{uninformed} cases. The algorithm consists of: (1) the model estimation procedure and recommender rule; (2) the sampling rule, dictating which vertex to select at each time-step; (3) the stopping rule, defining when enough evidence has been collected to identify the best vertex with sufficient confidence, and therefore to stop the algorithm.
\subsection{Estimation Procedure and Recommender Rule}\label{subsec:estimation_procedure}
The algorithm maintains a maximum likelihood estimate $\hat \nu(t)= (\hat G(t),\hat \mu(t))$ of the model.
 Using these estimates we define the estimated optimal vertex at time $t$ as $\hat a_t=\argmax_a \hat \mu_a(t)$, and the estimated sub-optimality gap in $u\neq \hat a_t$ as $\hat \Delta_u(t) = \hat \mu_{\hat a_t}(t) - \hat \mu_u(t)$, and $\hat \Delta_{\hat a_t}(t)= \hat \Delta_{\min}(t) \coloneqq \min_{u\neq \hat a_t}\hat \Delta_{u}(t)$.  The recommender rule at the stopping time $\tau$ is defined as $\hat a_\tau = \argmax_{a\in V} \hat \mu_a(\tau)$. Observe that we differentiate   between the informed and the uninformed cases.

\paragraph{Informed case.}
In the informed case, the model parameters can simply be estimated as $\hat \mu_u(t) = 
\frac{1}{M_u(t)}\sum_{n=1}^t R_{n,u}$, with $M_u(t)=\sum_{v\in V} N_{v,u}(t)$ and $N_{v,u}(t) = N_{v,u}(t-1)+\mathbf{1}_{\{V_t=v, u\in E_t\}}$. If the graph is unknown, the estimator of $G$ is
$
\hat G_{v,u}(t)=  \frac{N_{v,u}(t)}{N_{v}(t)}.
$

\paragraph{Uninformed case.} In the uninformed case, we focus on continuous rewards. We employ the fact that observing zero reward has measure zero, and therefore one can define $N_{v,u}(t) = N_{v,u}(t-1)+\mathbf{1}_{\{V_t=v, Z_t\neq 0\}}$ to obtain an unbiased estimator of the number of times edge $(v,u)$ was activated. Henceforth, one can define the estimator  $\hat \mu(t),\hat G_{v,u}(t)$ as in the observable case.

\subsection{Sampling Rule}
Interestingly, to design an algorithm with minimal sample complexity, we can look at 
 the solution $\omega^\star = \arginf_{\omega\in \Delta(V)} T(\omega;\nu)$.
 
The solution $\omega^\star$ provides the best  proportion of  draws, that is, an algorithm selecting a vertex $u\in V$ with probability $\omega_u^\star$ matches the lower bound in \cref{thm:lb_general} and is therefore optimal with respect to $T^\star(\nu)$.  Therefore, an idea is to ensure that $N_t/t$ tracks $\omega^\star$, where $N_t$ is the visitation vector  $N(t)\coloneqq \begin{bmatrix}
    N_{1}(t) &\dots &N_K(t)
\end{bmatrix}^\top$.

However, the instance $\nu$ is initially unknown. Common algorithms in the best arm identification literature \citep{garivier2016optimal,kaufmann2016complexity}, or in best policy identification problems \citep{al2021navigating,russo2024multi},   track an estimated optimal allocation $\omega^\star(t)=\arginf_{\omega\in \Delta(V)} T(\omega;\hat \nu(t))$ using the current  estimate of the model $\hat \nu(t)=(\hat G(t),\{\hat \mu_u(t)\}_u)$ (see line \ref{algo:compute_omega_star_t} in \cref{alg:TaSFG}). This technique is also known as \emph{certainty-equivalence} principle \citep{jedra2020optimal}. 

To apply this  principle, the learner needs to at-least make sure to use a sampling rule guaranteeing that $\hat \nu(t) \to \nu$. Common techniques employ some form of \emph{forced exploration} to ensure convergence, for example by introducing uniform noise in the action selection process.  In general, there are two families of tracking procedures: the C-tracking and D-tracking rules introduced in \citep{garivier2016optimal}:
\begin{itemize}
    \item \emph{C-tracking}: compute  $\tilde \omega^\star(t)$, which is the $\ell_\infty$ projection  of $\omega^\star(t)$ on $\Delta_t(V)=\{\omega\in \Delta(V): \forall u \in V, \omega_u \geq \epsilon_t\}$ for some $\epsilon_t>0$. Then the action selection is $u_t = \argmin_{u\in V} N_{t}(u) - \sum_{n=1}^t \tilde \omega_{u}^\star(t)$. 
    \item \emph{D-tracking}: if there is a vertex $u$ with $N_{u}(t) \leq \sqrt{t}-K/2$ then choose $u_t=u$. Otherwise, choose the vertex $u_t =\argmin_{u\in V} N_{u}(t) - t \omega_u^\star(t)$.  Compared to C-tracking, this rule enjoys deterministic guarantees on the visitation frequency of order $\sqrt{t}$.
\end{itemize}
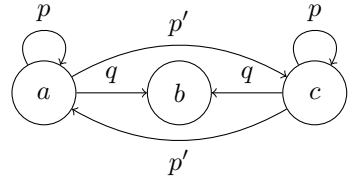
\begin{figure}[t]
    \centering

\begin{tikzpicture}[
  vertex/.style = {circle, draw, minimum size=0.85cm},
  loop/.style = {looseness=3, in=60, out=120, min distance=7mm},
  every edge/.style = {draw, thick}
]
\node[vertex] (A1) at (-1.8,0) {$a$};
\node[vertex] (B1) at (0,0) {$b$};
\node[vertex] (C1) at (1.8,0) {$c$};

\draw[->, loop] (A1) to node[above] {$p$} (A1);
\draw[->] (A1) to node[above] {$q$} (B1);
\draw[->, bend left] (A1) to node[above] {$p'$} (C1);

\draw[->, loop] (C1) to node[above] {$p$} (C1);
\draw[->] (C1) to node[above] {$q$} (B1);
\draw[->, bend left] (C1) to node[below] {$p'$} (A1);

\end{tikzpicture}

    \caption{Example of symmetric feedback graph where the solution set $C^\star(\nu)=\arginf_{w\in \Delta(V)} T(w;\nu)$ is not unique for $
    \mu_a=\mu_c$ and $a^\star(\mu)= b$.}
    \label{fig:example_equal_solutions}
\end{figure}
\paragraph{Uniqueness of the solution $\omega^\star$.}  These approaches guarantee that $N_t/t$ approximately track $\omega^\star(t)$, which ultimately should converge to $\omega^\star$. Unfortunately, for the problem studied in this manuscript, the optimal solution $\omega^\star$ \emph{may not be  unique}, i.e., the set $C^\star(\nu) = \arginf_{\omega\in \Delta(V)} T(\omega;\nu)$ may admit more than one optimal solution. This may be an issue if we want $N_t/t$ to track $\omega^\star$. An example  of graph with non-unique solution is the following one. 
\begin{example}[Multiple optimal allocations]
    Consider the example in \cref{fig:example_equal_solutions} with Gaussian rewards $\left(\mathcal{N}(\mu_u,1)\right)_{u\in\{a,b,c\}}$. This is an example of symmetric graph where vertex $b$ is optimal, and the solution set $C^\star(\nu) = \arginf_{\omega\in \Delta(V)} T(\omega;\nu)$ is actually a convex set. In fact, for any $x\in [0,1]$ the solution  $\omega(x)=\begin{bmatrix}x & 0 & 1-x
    \end{bmatrix}$ is optimal. As an example, for $p=p'=1/2,  q=1$ and $\mu_a=\mu_c=0, \mu_b=1$ we obtain $T^\star(\nu)=T(\omega;\nu)=6$ for any $\omega\in \{\omega(x): x\in [0,1]\}$.
\end{example}
In literature, some strategies have been proposed to deal with the aforementioned issue. \cite{russo2023sample}  regularize $T^\star(\omega;\nu)$ to make it strongly convex in $\omega$, and obtain a sample complexity upper bound that is asymptotically optimal in $\delta$ and the regularization parameter. In \citep{jedra2020optimal,degenne2019pure} they use an approach inspired by  \cite{Bonsall1963CB}'s Maximum theorem and the fact that $C^\star(\nu)$ is a convex set. Using this fact, they show that with an appropriate tracking rule it is possible to ensure $\lim_{t\to\infty} \inf_{\omega\in C^\star(\nu)}\left\|\omega-N(t)/t\right\|_\infty = 0$. 

 Based on this result, \cite{jedra2020optimal}  propose an averaged D-tracking  sampling rule for linear multi-armed bandit problem, while \cite{degenne2019pure}  opt for a C-tracking procedure claiming that D-tracking may fail to converge in classical bandit problems. Nonetheless, we can guarantee convergence using an averaged variant of D-tracking  similarly to \citep{jedra2020optimal}.
\begin{proposition}\label{prop:tracking}
    Let $S_t= \{u\in V: N_u(t) < \sqrt{t}-K/2\}$. The averaged D-tracking rule, defined as
    \begin{equation}
        V_t \in \begin{cases}
            \argmin_{u\in S_t} N_u(t) & S_t\neq \emptyset\\
            \argmin_{u\in V} N_u(t) - \sum_{n=1}^t \omega_u^\star(n)& \hbox{otherwise}
        \end{cases},
    \end{equation}
    ensures that
    $
   \lim_{t\to\infty}\inf_{\omega\in C^\star(\nu)}\| N(t)/t -\omega\|_\infty 
 \to 0 
    $ a.s.
\end{proposition}
The proof for the above proposition can be found in \cref{subsec:app_sampling_rule}.

\subsection{Stopping Rule and Sample Complexity Guarantees}
The stopping rule determines when enough evidence has been collected to determine the optimal action with a prescribed confidence level. The problem of determining when to stop can be framed as a statistical hypothesis testing problem \citep{chernoff1959sequential}, where we are testing between $K$ different hypotheses $({\cal H}_u: (\mu_u > \max_{v\neq u}\mu_v))_{u\in V}$.

We consider the following statistic in line \ref{algo:stopping_rule} of \cref{alg:TaSFG}
\begin{equation}\label{eq:statistic_stopping_rule}
L(t) = t T(N(t)/t; \hat \nu(t))^{-1},
\end{equation}
which is a Generalized Likelihood Ratio Test (GLRT), similarly as in \citep{garivier2016optimal}. Comparing with \cref{thm:lb_general}, one needs to stop as soon as  $L(t)\geq {\rm kl}(\delta,1-\delta) \sim \ln(1/\delta)$. However, to  account for the random fluctuations, a more natural threshold is $\beta(t,\delta)=\ln((1+\ln(t))/\delta)$.
Several thresholds have been proposed in literature \citep{magureanu2014lipschitz,garivier2016optimal}.
We employ the following threshold
\begin{equation}\label{eq:threshold}
    \beta(t,\delta)\coloneqq2  {\cal C}_{\rm exp}\left(\frac{\ln\left(\frac{K-1}{\delta}\right)}{2}\right) + 6\ln(1+\ln(t)),
\end{equation}
where ${\cal C}_{\rm exp}(x)\approx x+4\ln(1+x+\sqrt{2x})$ for $x\geq 5$. A precise definition of ${\cal C}_{\rm exp}(x)$ can be found in \cite[Theorem 7]{kaufmann2021mixture}, or in \cref{subsec:app_stopping_rule}. For this threshold, we obtain the following guarantee, proved in \cref{subsec:app_stopping_rule}.
\begin{proposition}\label{prop:threshold_error_rate}
    Using the threshold function defined in \cref{eq:threshold} guarantees that \algoname{} is $\delta$-PC, i.e.
    \[\mathbb{P}_\nu(\tau<\infty, \hat a_\tau \neq a^\star(\mu)) \leq \delta.\]
\end{proposition}

Additionally, we have the following  sample complexity optimality guarantees, proved in \cref{subsec:algorithm_as_sample_complexity_app}.
\begin{theorem}\label{thm:sample_complexity}
    For all $\delta\in (0,1/2)$, \algoname{} (1) terminates a.s. $\mathbb{P}_\nu(\tau<\infty)=1$; (2) is a.s. asymptotically optimal $\mathbb{P}_\nu\left(\limsup_{\delta \to 0} \frac{\tau}{\ln(1/\delta)} \leq T^\star(\nu)\right) = 1$; (3) is optimal in expectation $\limsup_{\delta \to 0} \frac{\mathbb{E}_\nu\left[ \tau\right]}{\ln(1/\delta)} \leq T^\star(\nu)$.
\end{theorem}
\paragraph{Heuristic \algoname{}.} Now consider \algoname{} with the heuristic solution $\omega_{\rm heur}(t) \coloneqq \frac{\hat G(t)\hat\Delta^{-2}(t)}{\|\hat G(t)\hat\Delta^{-2}(t)\|_1}$ (where $\hat G(t)$ is the estimated graph at time $t$, and $\hat \Delta^{-2}(t) = (1/\hat\Delta_u^2(t))_{u\in V}$). That is, we let $\omega^\star(t)=\omega_{\rm heur}(t)$ in line ~\ref{algo:compute_omega_star_t} of \cref{alg:TaSFG}. Then we obtain the following guarantees, proved in \cref{subsec:algorithm_as_sample_complexity_app}.
\begin{corollary}\label{thm:sample_complexity_heuristic}
     \algoname{} with $\omega^\star(t) = \omega_{\rm heur}(t)$ is $\delta$-PC, and guarantees that (1) the algorithm stops a.s.; (2) $\mathbb{P}_\nu\left(\limsup_{\delta \to 0} \frac{\tau}{\ln(1/\delta)} \leq T(\omega_{\rm heur}; \nu)\right)=1$ and (3) $\limsup_{\delta \to 0} \frac{\mathbb{E}_\nu[\tau]}{\ln(1/\delta)} \leq T(\omega_{\rm heur};\nu)$.
\end{corollary}

\begin{algorithm}[t]
   \caption{\algoname{} (TaS for Feedback Graphs)}
   \label{alg:TaSFG}
\begin{algorithmic}[1]
   \STATE {\bfseries Input:} confidence $\delta$.
     \STATE Set $t\gets 1$
   \WHILE{$L(t) <\beta(\delta,t)$} \label{algo:stopping_rule}
   \STATE Compute $\omega^\star(t)=\arginf_{\omega\in \Delta(V)}T(\omega;\hat\nu(t))$. \label{algo:compute_omega_star_t}
   \STATE Select $V_t$ according to the D-tracking rule in \cref{prop:tracking} and observe  $Z_t$.
   \STATE Update statistics $N(t),M(t),\hat \mu(t), \hat G(t)$ as explained in \cref{subsec:estimation_procedure} and set $t\gets t+1$.
   \ENDWHILE
   \STATE {\bf Return} $\hat{a}_{\tau} = \argmax_{a\in V} \hat\mu_a(\tau)$
\end{algorithmic}
\end{algorithm}
\begin{figure*}[t]
    \centering
    \includegraphics[width=\linewidth]{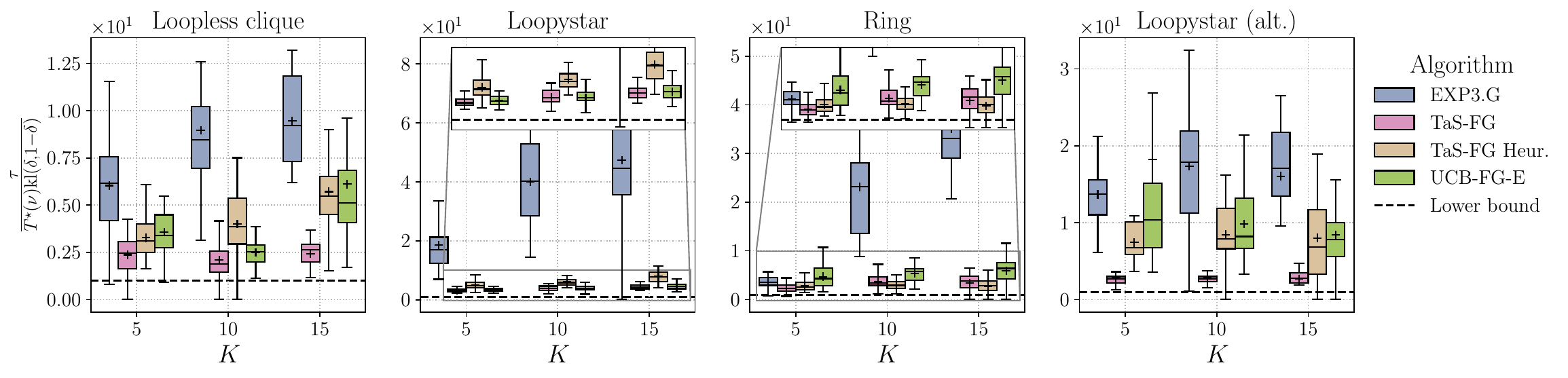}
    \caption{Box plots of the normalized sample complexity $\frac{\tau}{T^\star(\nu){\rm kl}(\delta,1-\delta)}$ for $\delta=e^{-7}$ over $100$ seeds. Boxes indicate the interquartile range, while the median   and mean values are, respectively, the solid  line and the $+$  sign in black.}
    \label{fig:boxplot_results_summary}
\end{figure*}
\section{NUMERICAL RESULTS}
In this section, we present the numerical results  of \algoname{} and other algorithms on various graph configurations, with varying graph sizes.
Due to lack of space, we only present results for $\delta\approx10^{-3}$, and  we refer the reader to \cref{sec:app_numerical_results} for more numerical results, including evaluation on different values of $\delta$.

\paragraph{Algorithms.} We compare \algoname{} to heuristic \algoname{} (i.e., \algoname{} with the heuristic solution proposed in \cref{subsec:heuristic_sol}) and to {\sc EXP3.G} \citep{alon2015online}, an algorithm for regret minimization in the adversarial case (which   is similar to most algorithms designed to handle feedback graphs). We also compare with two variants of {\sc UCB}: (1) {\sc UCB-FG-E}, which acts greedily with respect to the upper confidence bound of $\left(\hat G(t)\hat \mu(t)\right)$; and (2) {\sc UCB-FG-V}, which selects $\argmax_v \hat G_{v, \hat a^{\rm ucb}(t)}^{\rm ucb}(t)$, where $\hat G^{\rm ucb}(t), \hat a^{\rm ucb}(t)$ are, respectively, the UCB estimates of $G$ and $a^\star$ at time $t$.
 In our experiments, we used the same stopping rule for all algorithms, as the guarantees of \cref{prop:threshold_error_rate} hold for any sampling rule. However, we observed that the rate $\beta(t,\delta)$ given in \cref{eq:threshold} is relatively conservative. Therefore, we opted for a less conservative threshold that still achieves satisfactory error rates, that is $\beta(t,\delta)=\ln(1/\delta)+3 \ln(1+2\ln(t))$.

\paragraph{Results.} The main sample complexity results are depicted in \cref{fig:boxplot_results_summary} (results are omitted for {\sc UCB-FG-V} due to its much larger sample complexity; we report those in \cref{sec:app_numerical_results}). We evaluated these algorithms on several graphs: {\bf (1)} the loopystar (\cref{fig:loopy_star_graph}) with $(p,q,r)=(1/5,1/4,1/4)$ and {\bf (2)} an alternative version (denoted by `Loopystar (alt.)' in \cref{fig:boxplot_results_summary}) where $\mu_1=1$ is optimal and $\mu_u=0.5$ for all other vertices, with $(p,q,r)=(0,1/4, 1/[8(K-1)])$; {\bf (3)} the ring and {\bf (4)} loopless clique graphs (see, respectively, \cref{fig:ring_graph_details}  and \cref{fig:looplessclique_graph_details} in \cref{sec:app_numerical_results} for more details). Overall, we observe the good performance of \algoname{} and heuristic \algoname{}, while {\sc EXP3.G}  is not as performant.
Notably,  {\sc UCB-FG-E} also seems to achieve good sample complexity, comparable to heuristic \algoname{}, although it may not perform as well when $(G\mu)_{a^\star}$ is small in comparison to that of other vertices.
\section{CONCLUSIONS}
In this paper we characterized the sample complexity of identifying the best action in an online learning problem with a feedback graph. In this setting, the graph dictates what feedback the learner observes, with the additional complexity that the graph may be unknown and have stochastic activations (a.k.a. \emph{uninformed setting}). For such a setting, we showed that for Bernoulli rewards the best action remains unidentifiable. In contrast, for continuous rewards, we derived an instance-specific lower bound applicable to all $\delta$-PC algorithms, a result that also extends to the informed setting where the graph or activations are known.
Building on this lower bound, we introduced \algoname{}, an algorithm achieving asymptotically optimal sample complexity in both  settings, for which we demonstrated its efficiency numerically on several graph configurations.

\section*{Acknowledgments}
The authors are pleased to acknowledge that the computational work reported on in this paper was
performed on the Shared Computing Cluster, which is administered by Boston University’s Research
Computing Services.

\bibliographystyle{apalike}
\bibliography{ref}
\section*{Checklist}

 \begin{enumerate}

 \item For all models and algorithms presented, check if you include:
 \begin{enumerate}
   \item A clear description of the mathematical setting, assumptions, algorithm, and/or model. [{\color{blue}Yes}], we present a description of the setting, as well as a description of the assumptions (see also \cref{subsec:bai_setting} and \cref{assump:nontrivial_problem}).
   \item An analysis of the properties and complexity (time, space, sample size) of any algorithm. [{\color{blue}Yes}], we provide a sample complexity analysis of the algorithm in \cref{subsec:algorithm}.
   \item (Optional) Anonymized source code, with specification of all dependencies, including external libraries. [{\color{blue}Yes}], in the appended {\tt README.md} file we included a description of all the needed dependecies as well as the anonymized source code.
 \end{enumerate}

 \item For any theoretical claim, check if you include:
 \begin{enumerate}
   \item Statements of the full set of assumptions of all theoretical results. [{\color{blue}Yes}], we include in  our results the full set of assumptions.
   \item Complete proofs of all theoretical results. [{\color{blue}Yes}], we provide in the appendix a complete proofs of all the theoretical results presented.
   \item Clear explanations of any assumptions. [{\color{blue} Yes}], we provide in \cref{subsec:bai_setting} an explanation of all the assumptions used.     
 \end{enumerate}

 \item For all figures and tables that present empirical results, check if you include:
 \begin{enumerate}
   \item The code, data, and instructions needed to reproduce the main experimental results (either in the supplemental material or as a URL). [{\color{blue} Yes}], in the appended {\tt README.md} file we describe how to reproduce the main experimental results.
   \item All the training details (e.g., data splits, hyperparameters, how they were chosen). [{\color{blue} Yes}], in \cref{sec:app_numerical_results} we describe the training details.
         \item A clear definition of the specific measure or statistics and error bars (e.g., with respect to the random seed after running experiments multiple times). [{\color{blue} Yes}], we provide a measure of variability in our results, and provide its definition .
         \item A description of the computing infrastructure used. (e.g., type of GPUs, internal cluster, or cloud provider). [{\color{blue} Yes}], in \cref{sec:app_numerical_results} we provide a description of the computing infrastructure used.
 \end{enumerate}

 \item If you are using existing assets (e.g., code, data, models) or curating/releasing new assets, check if you include:
 \begin{enumerate}
   \item Citations of the creator If your work uses existing assets. [{\color{blue} Yes}], for all libraries that we used we cited those libraries appropriately (see also \cref{sec:app_numerical_results}).
   \item The license information of the assets, if applicable. [{\color{blue} Yes}], we provide license information of the assets used, as well as the license information of our code in the appended material.
   \item New assets either in the supplemental material or as a URL, if applicable. [{\color{blue} Yes}], we provide new code to solve the optimization problems, and run the numerical results presented in this manuscript.
   \item Information about consent from data providers/curators. [{\color{gray} Not Applicable}], we use assets for which consent is not needed.
   \item Discussion of sensible content if applicable, e.g., personally identifiable information or offensive content. [{\color{gray} Not Applicable}], we not use sensible content.
 \end{enumerate}

 \item If you used crowdsourcing or conducted research with human subjects, check if you include:
 \begin{enumerate}
   \item The full text of instructions given to participants and screenshots. [{\color{gray} Not Applicable}], we did not use crowdsourcing nor  conducted research with human subjects.
   \item Descriptions of potential participant risks, with links to Institutional Review Board (IRB) approvals if applicable. [{\color{gray} Not Applicable}], we did not use crowdsourcing nor  conducted research with human subjects.
   \item The estimated hourly wage paid to participants and the total amount spent on participant compensation. [{\color{gray} Not Applicable}], we did not use crowdsourcing nor  conducted research with human subjects.
 \end{enumerate}

 \end{enumerate}

\onecolumn
\appendix

\aistatstitleappendix{Pure Exploration with Feedback Graphs\\(Supplementary Material)}

\tableofcontents
\addcontentsline{toc}{section}{Supplementary Material}
\newpage
\section{Additional Numerical Results}\label{sec:app_numerical_results}
In this section of the appendix we describe the graphs used in the numerical results; the details of the algorithms used and exhibits additional numerical results.

\subsection{Graphs Details}
Here we briefly describe the graphs used in the numerical results. Also refer to the code for more details.

\begin{figure}[h]
    \centering
     \resizebox{0.3\textwidth}{!}{ 
\begin{minipage}[t]{0.33\textwidth}
\centering
\begin{tikzpicture}[
  vertex/.style = {circle, draw, minimum size=0.5cm},
  loop/.style = {looseness=3, in=60, out=120, min distance=7mm},
  every edge/.style = {draw, thick}
]
\node[vertex] (A1) at (0,0) {$\mu_1$};
\node[vertex] (B1) at (2.5,-0.5) {$\mu_5$};
\node[vertex] (C1) at (1.25,-1.5) {$\mu_4$};
\node[vertex] (D1) at (-1.25,-1.5) {$\mu_3$};
\node[vertex] (E1) at (-2.5,-0.5) {$\mu_2$};

\draw[->] (A1) -- node[above] {$p$}  (B1);
\draw[->] (A1) -- node[right] {$r$}  (C1);
\draw[->] (A1) -- node[left] {$r$}  (D1);
\draw[->] (A1) -- node[above] {$r$}  (E1);

\draw[->, loop] (A1) to node[above] {$q$}  (A1);
\draw[->, loop] (B1) to [out=240,in=300] node[below] {$1-p$}  (B1);
\draw[->, loop] (C1) to [out=240,in=300]  node[below] {$(1-2p)^+$} (C1);
\draw[->, loop] (D1) to [out=240,in=300] node[below] {$(1-2p)^+$}  (D1);
\draw[->, loop] (E1) to [out=240,in=300] node[below] {$(1-2p)^+$}  (E1);
\end{tikzpicture}
    
    \caption{Loopy star graph.}
    \label{fig:loopy_star_graph_app}
\end{minipage}
}
\hfill
\resizebox{0.3\textwidth}{!}{ 
\begin{minipage}[t]{0.33\textwidth}
\centering
\begin{tikzpicture}[
  vertex/.style = {circle, draw, minimum size=0.7cm},
  loop/.style = {looseness=3, in=60, out=120, min distance=7mm},
  every edge/.style = {draw, thick}
]
\node[vertex] (A1) at (0,0) {$\mu_1$};
\node[vertex] (B1) at (2.25,-2) {$\mu_2$};
\node[vertex] (C1) at (0, -4) {$\mu_3$};
\node[vertex] (D1) at (-2.25,-2) {$\mu_4$};

\draw[->, bend right=20] (A1) to node[left] {$p$} (B1);
\draw[->, bend right=20] (A1) to node[left] {$1-p$} (D1);
\draw[->, bend right=20] (B1) to node[right] {$1-p$} (A1);
\draw[->, bend right=20] (B1) to node[left] {$p$} (C1);
\draw[->, bend right=20] (C1) to node[right] {$p$} (D1);
\draw[->, bend right=20] (C1) to node[right] {$1-p$} (B1);

\draw[->, bend right=20] (D1) to node[right] {$p$} (A1);
\draw[->, bend right=20] (D1) to node[left] {$1-p$} (C1);
\end{tikzpicture}

    \caption{Ring graph.}
    \label{fig:ring_graph_details}
\end{minipage}}
\hfill
\resizebox{0.3\textwidth}{!}{ 
\begin{minipage}[t]{0.33\textwidth}
\centering

\begin{tikzpicture}[
  vertex/.style = {circle, draw, minimum size=0.7cm},
  loop/.style = {looseness=3, in=60, out=120, min distance=7mm},
  every edge/.style = {draw, thick}
]
\node[vertex] (A1) at (0,0) {$\mu_1$};
\node[vertex] (B1) at (2.25,-2) {$\mu_2$};
\node[vertex] (C1) at (0, -4) {$\mu_3$};
\node[vertex] (D1) at (-2.25,-2) {$\mu_4$};

\draw[->, bend right=20] (A1) to (B1);
\draw[->, bend right=20] (A1) to (D1);
\draw[->, bend right=20] (B1) to  (A1);
\draw[->, bend right=20] (B1) to (C1);
\draw[->, bend right=20] (C1) to  (D1);
\draw[->, bend right=20] (C1) to  (B1);

\draw[->, bend right=20] (D1) to  (A1);
\draw[->, bend right=20] (D1) to (C1);

\draw[->, bend left=10] (D1) to  (B1);
\draw[->, bend left=10] (B1) to  (D1);

\draw[->, bend left=10] (A1) to  (C1);
\draw[->, bend left=10] (C1) to   (A1);
\end{tikzpicture}

    \caption{Loopless clique graph. }
    \label{fig:looplessclique_graph_details}
    \end{minipage}}
\end{figure}
\subsubsection{Loopy star graph}  
The loopy star graph (\cref{fig:loopy_star_graph_app}) is the only graph with self-loops in our experiments. However, it depends on several parameter $(p,q,r)$ that affect the underlying topology. The rewards are Gaussianly distributed, with variance $1$. The best arm has average reward $1$, while sub-optimal arms have average reward $0.5$.
\begin{figure}[h]
    \centering
    \includegraphics[width=.7\linewidth]{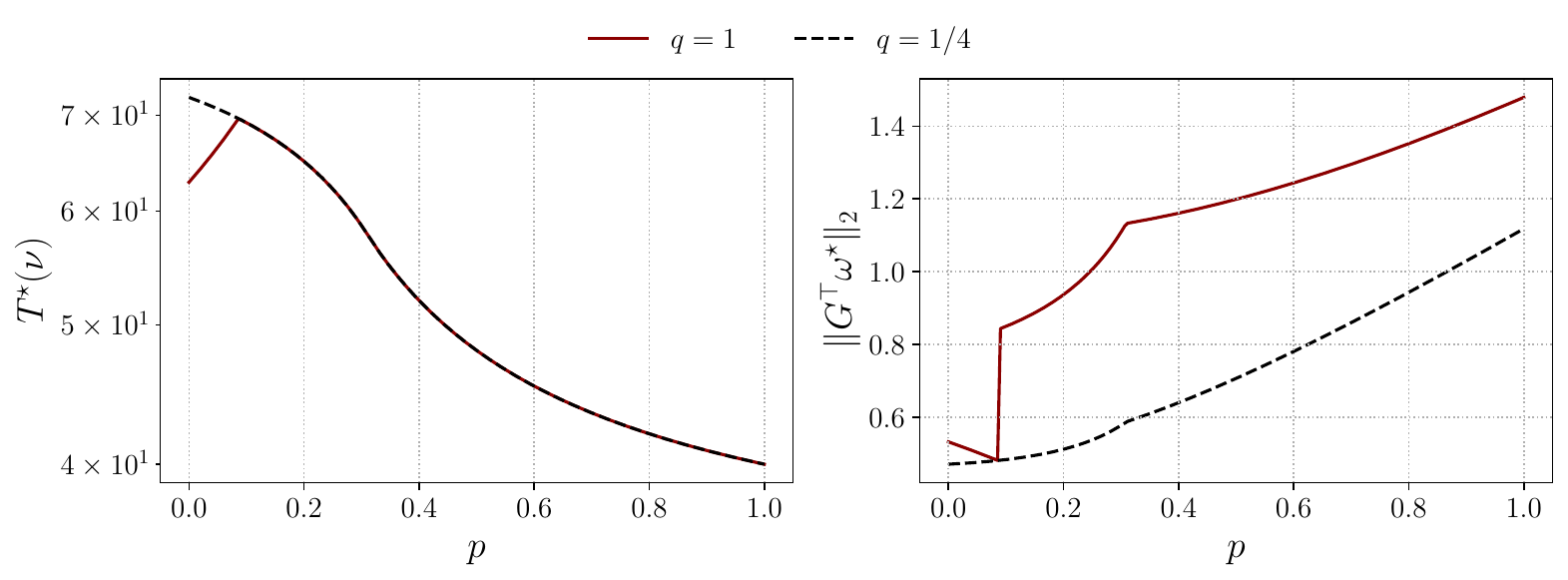}
    \caption{On the left: characteristic time of the loopy star graph with $K=5,r=0.5$ for different values of $p,q$ (the best arm is $v=5$). On the right: plot of $\|m^\star\|$ as a function of $\omega^\star$.}
    \label{fig:loopystar_characteristic_time_app}

    \includegraphics[width=0.4\linewidth]{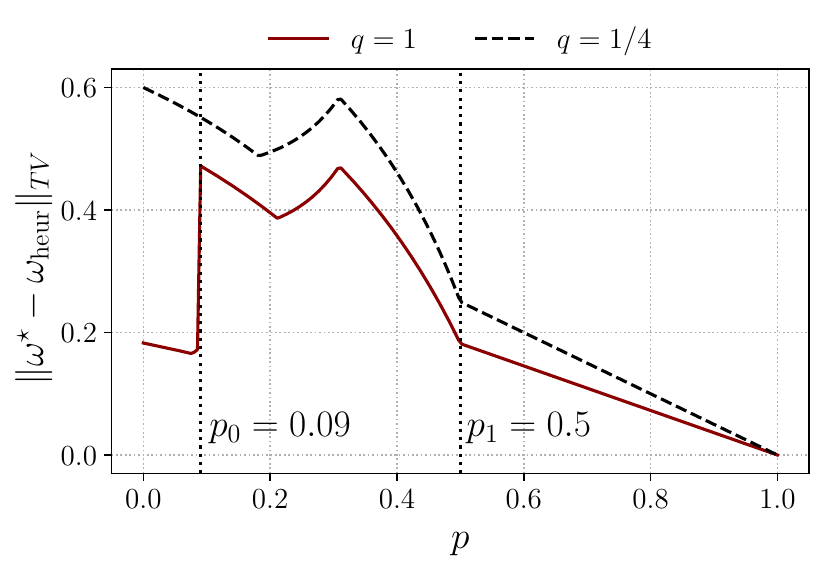}
    \caption{Difference between $\omega^\star$ and $\omega_{\rm heur}$ in the loopy star graph (same setting as in \cref{fig:loopystar_characteristic_time_app}).}
    \label{fig:approx_sol_loopystar}
\end{figure}

We simulated two settings:
\begin{enumerate}
    \item Main setting: we set $p=1/5, q=1/4, r=1/5$. Hence, self-loops may bring more information. The best arm is $v=5$. In \cref{fig:loopystar_characteristic_time_app} and \cref{fig:approx_sol_loopystar} we depict $T^\star(\nu),\|G^\top \omega^\star\|$ and the difference $\|\omega^\star - \omega_{\rm heur}\|$ for this setting. Notably, for increasing values of $p$ the approximate solution $\omega_{\rm heur}$ converges to the optimal solution.
    \item Alternative setting: we set $p=0, q=1/4, r=\frac{1-2q}{4(K-1)}$. Hence, it is not worth for the agent to choose $v=1$, and the optimal arm is $v=1$. In \cref{fig:loopystaralt_characteristic_time_app} and \cref{fig:approx_sol_loopystaralt} we depict $T^\star(\nu),\|G^\top \omega^\star\|$ and the difference $\|\omega^\star - \omega_{\rm heur}\|$ for this setting. Also in this case  for increasing values of $p$ the approximate solution $\omega_{\rm heur}$ converges to the optimal solution.
\end{enumerate}

\begin{figure}[H]
    \centering
    \includegraphics[width=.7\linewidth]{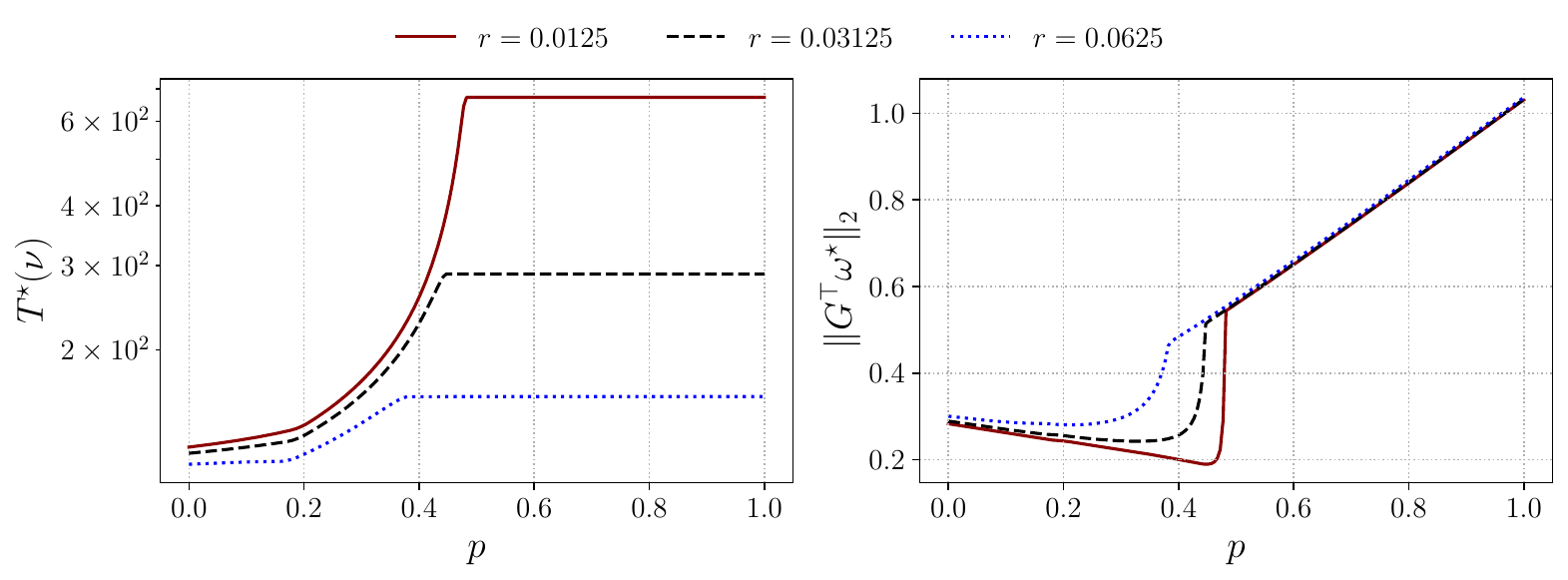}
    \caption{On the left: characteristic time of the loopy star graph in the alternative setting with $K=5$ and $r=\alpha\frac{1-2q}{4(K-1)},\alpha\in\{0.1,0.25,0.5\}$  (the best arm is $v=1$). On the right: plot of $\|m^\star\|$ as a function of $\omega^\star$.}
    \label{fig:loopystaralt_characteristic_time_app}

    \includegraphics[width=0.4\linewidth]{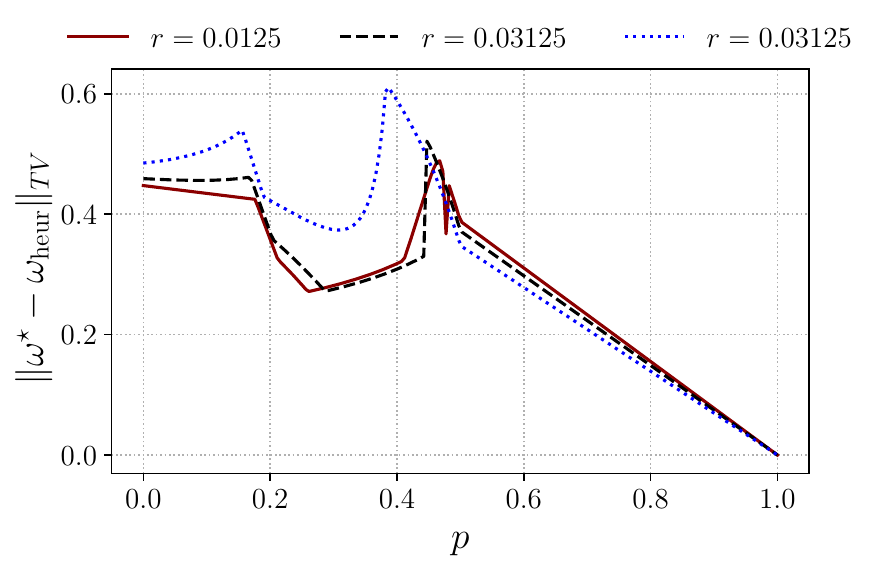}
    \caption{Difference between $\omega^\star$ and $\omega_{\rm heur}$ in the loopy star graph (same setting as in \cref{fig:loopystaralt_characteristic_time_app}).}
    \label{fig:approx_sol_loopystaralt}
\end{figure}

\subsubsection{Ring graph}  
In this graph   (\cref{fig:ring_graph_details}) each node is  connected to two adjacent nodes. For each node $u$, the feedback from the node on the right $v_r$ (clockwise direction) is "seen" with probability $p$, i.e., $G_{u,v_r}=p$; on the other hand, the feedback from the node on the "left" $v_l$ (anti-clockwise direction) is "seen" with probability $1-p$, i.e., $G_{u,v_l}=1-p$. All other edges have $0$ probability. We used a value of $p=0.3$ and rewards are Gaussian (with variance $1$), with mean values linearly distributed in $[0,1]$ across the $K$ arms.

In \cref{fig:ring_characteristic_time_app} and \cref{fig:approx_sol_ring} we depict $T^\star(\nu),\|G^\top \omega^\star\|$ and the difference $\|\omega^\star - \omega_{\rm heur}\|$ for this setting.  In this case the approximate solution $\omega_{\rm heur}$ converges to the optimal solution for $p \to 0$ or $p\to 1$.

\begin{figure}[h]
    \centering
    \includegraphics[width=0.7\linewidth]{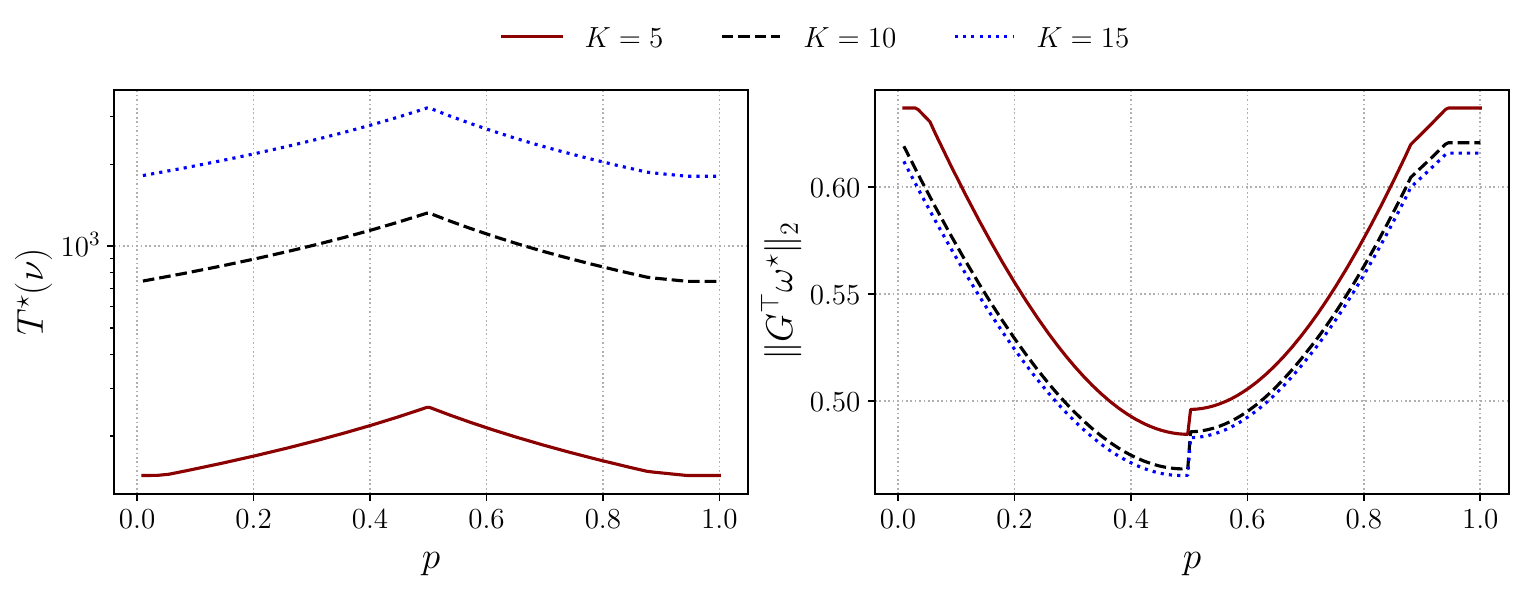}
    \caption{On the left: characteristic time of the ring graph for different values of $K,p$. On the right: plot of $\|m^\star\|$ as a function of $\omega^\star$.}
    \label{fig:ring_characteristic_time_app}

    \includegraphics[width=0.4\linewidth]{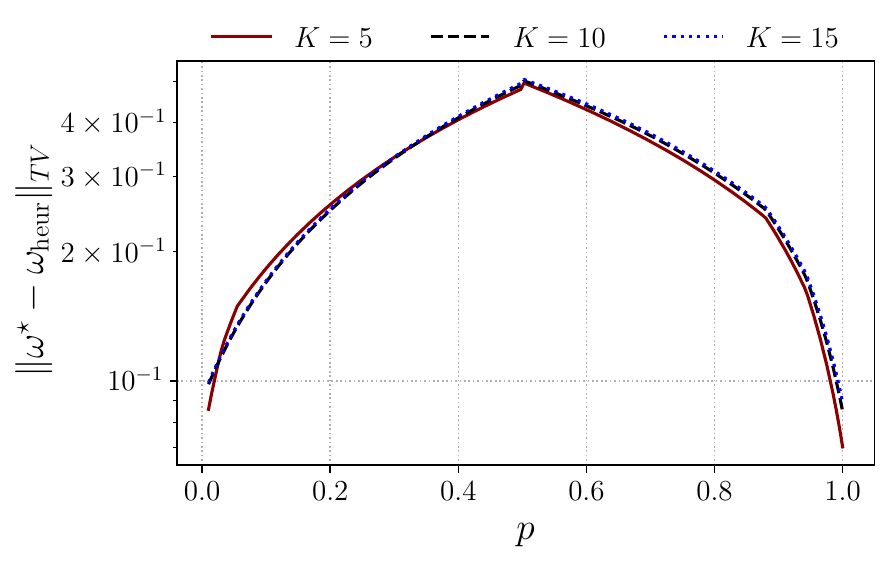}
    \caption{Difference between $\omega^\star$ and $\omega_{\rm heur}$ in the ring graph (same setting as in \cref{fig:ring_characteristic_time_app}).}
    \label{fig:approx_sol_ring}
\end{figure}

\subsubsection{Loopless clique graph} 

\begin{figure}[h]
    \centering
    \includegraphics[width=0.7\linewidth]{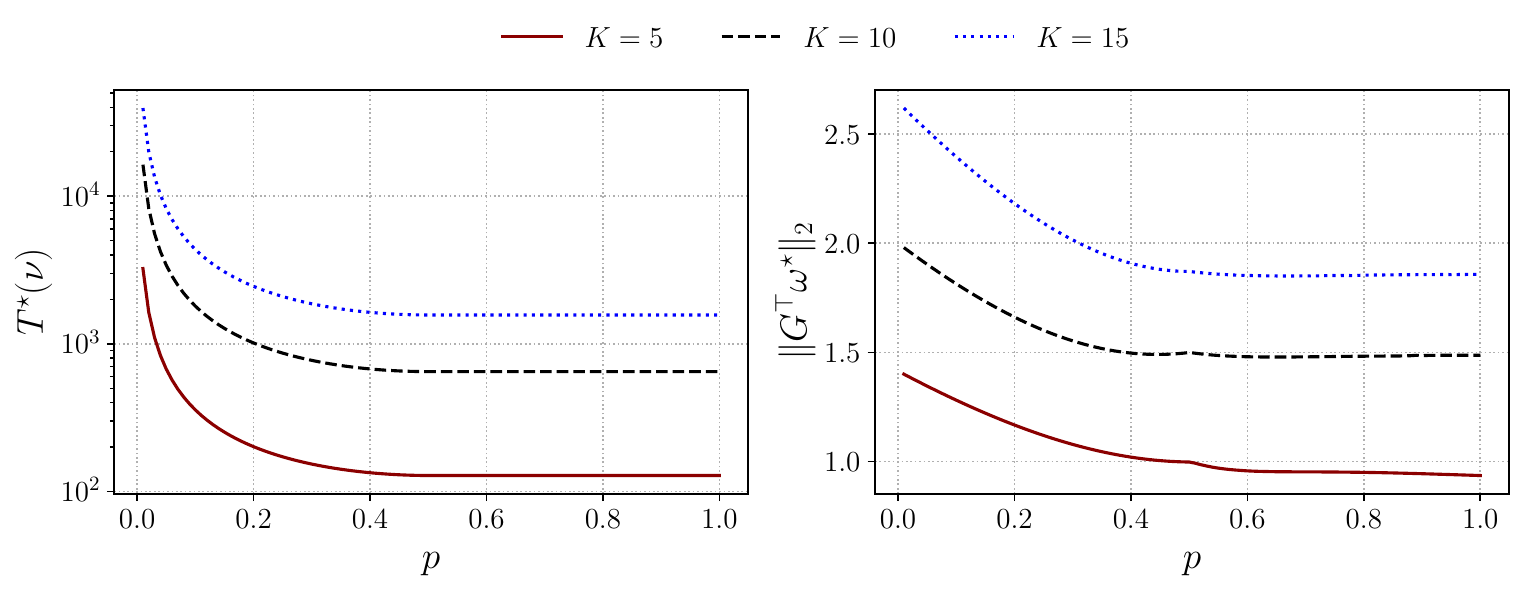}
    \caption{On the left: characteristic time of the loopless graph  for different values of $K,p$. On the right: plot of $\|m^\star\|$ as a function of $\omega^\star$.}
    \label{fig:looplessclique_characteristic_time_app}

    \includegraphics[width=0.4\linewidth]{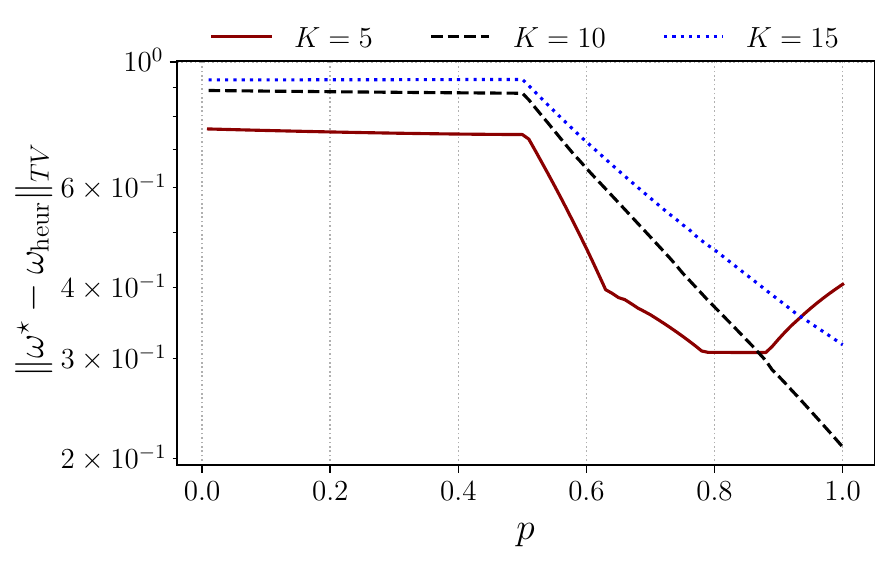}
    \caption{Difference between $\omega^\star$ and $\omega_{\rm heur}$ in the loopless graph (same setting as in \cref{fig:looplessclique_characteristic_time_app}).}
    \label{fig:approx_sol_looplessclique}
\end{figure}

This graph  (\cref{fig:looplessclique_graph_details}) is  fully connected  without any self-loops (the probabilities in the figure are omitted to avoid cluttering). Assuming the vertices $u\in V$ are numbered according to the natural numbers $V=\{1,2,\dots,K\}$, the edge probabilities are
\[
G_{u,v} = \begin{cases}
0 & v=u,\\
p/u & \forall v\neq u \wedge v\notin 2\mathbb{N},\\
1-(p/u) & \hbox{otherwise}.
\end{cases}
\]
where $2\mathbb{N}$ is the set of even numbers. We see that in this type of graph it may be better to choose a vertex according to its index depending on what type of feedback the learner seeks.
We used a value of $p=0.5$ and rewards are Gaussian (with variance $1$), with mean values linearly distributed in $[0,1]$ across the $K$ arms.

In \cref{fig:looplessclique_characteristic_time_app} and \cref{fig:approx_sol_looplessclique} we depict $T^\star(\nu),\|G^\top \omega^\star\|$ and the difference $\|\omega^\star - \omega_{\rm heur}\|$ for this setting.  In this case the approximate solution $\omega_{\rm heur}$ does not seem to converge to $\omega^\star$ for any value of $p$.

\subsection{Numerical Results}
In this section we present additional details on the numerical results.

\subsubsection{Algorithms and results}
In addition to \algoname{} and heuristic \algoname{} we also used {\sc EXP3.G} \citep{alon2015online}, an algorithm for regret minimization in the adversarial case. We also compare with two variant of {\sc UCB}: (1) {\sc UCB-FG-E}, which acts greedily with respect to the upper confidence bound of $\left(\hat G(t)\hat \mu(t)\right)$; (2) {\sc UCB-FG-V}, which selects $\argmax_v \hat G_{v, \hat a^{\rm ucb}(t)}^{\rm ucb}(t)$, where $\hat G^{\rm ucb}(t), \hat a^{\rm ucb}(t)$ are, respectively, the UCB estimates of $G$ and $a^\star$ at time $t$.
For all algorithms, the graph estimator $\hat G(t)$ was initialized in an optimistic way, i.e., $\hat G_{u,v}(1)=1$ for all $u,v\in V$.

\paragraph{EXP3.G algorithm.} {\sc EXP3.G} \cite{alon2015online} initializes two vectors $p,q\in \mathbb{R}^K$ uniformly, so that $p_i=q_i=1/K$ for $i=1,\dots,K$. At every time-step, an action $V_t$ is drawn from $V_t\sim p_t$, where $p_t\gets  (1-\eta)q_t + \eta {\cal U}$ with $\eta\in (0,1)$ being an exploration facotr and  ${\cal U}$ is the uniform distribution over $\{1,\dots,K\}$.

After observing the feedback, the algorithm sets
\[
\hat q_t \gets q_t \exp( -\eta x_t) \hbox{ and } q_t \gets \hat q_t / \sum_u \hat q_{t,u},
\]
where $x_{t,u} = -Z_{t,u}/ \sum_{v \in N_{in}(u)} p_{t,v}$.  In the experiments, we let $\eta=3/10$.

\paragraph{UCB-FG-E algorithm.} This method is  a variant of UCB that acts greedily  with respect to the upper confidence bound of $\left(\hat G(t)\hat \mu(t)\right)$. In practice, we let $\hat \mu^{\rm ucb}_u(t)= \hat \mu_u(t) +\sqrt{\frac{2 \ln(1+t)}{M_u(t)}}$, $\hat G_{u,v}^{\rm ucb}(t)=\hat G_{u,v}(t) + \sqrt{\frac{\ln(1+t)}{2N_{u}(t)}}$ and select the action to take according to  $V_t = \argmax_u \hat G^{\rm ucb}(t) \hat \mu^{\rm ucb}(t)$.

\paragraph{UCB-FG-V algorithm.}  This method is  a variant of UCB that selects $V_t=\argmax_v \hat G_{v, \hat a^{\rm ucb}(t)}^{\rm ucb}(t)$, where $\hat G^{\rm ucb}(t), \hat a^{\rm ucb}(t)$ are, respectively, the UCB estimates of $G$ and $a^\star$ at time $t$ (note that $a^{\rm ucb}(t) = \argmax_u \hat \mu_u^{\rm ucb}(t)$).

In \cref{fig:sample_complexity_1}  we depict the sample complexity of the algorithms for different values of $K,\delta$. Note that the sample complexity $\tau$ is not  normalized, and results were computed over $100$ seeds.

\begin{figure}
    \centering
    \includegraphics[width=.8\linewidth]{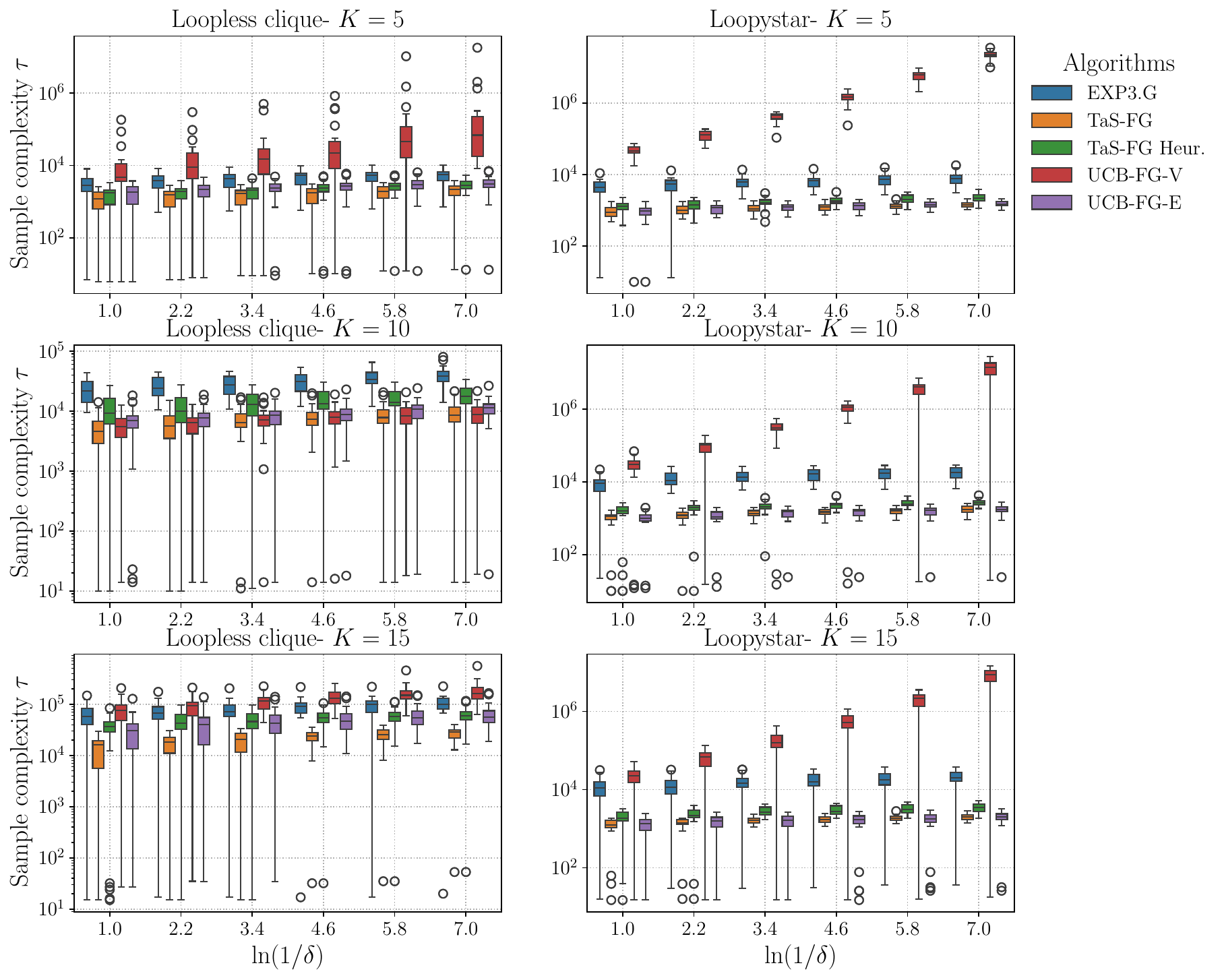}
    \includegraphics[width=.8\linewidth]{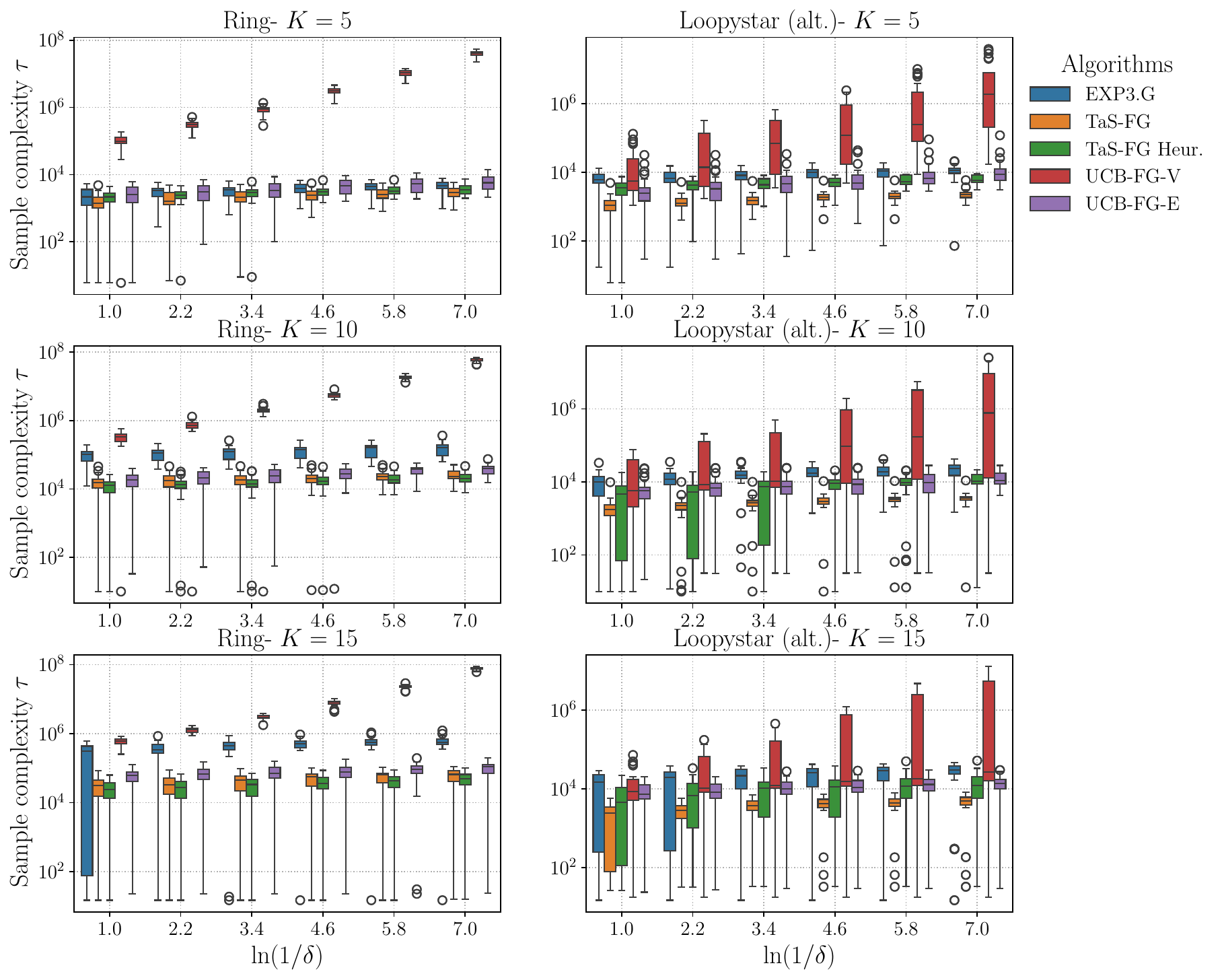}
    \caption{Sample complexity results, from top-left (clockwise): loopless clique, loopystar, ring and alternative loopystar graphs.
    Box plots are not normalized, and were computed over $100$ seeds. Boxes indicate the interquartile range, while the median   and mean values are, respectively, the solid  line and the $+$  sign in black.}
    \label{fig:sample_complexity_1}
\end{figure}

\subsubsection{Libraries and computational resources}

\paragraph{Libraries used in the experiments.} We set up our experiments using Python 3.10.12 \citep{van1995python} (For more information, please refer to the following link \url{http://www.python.org}), and made use of the following libraries: NumPy \citep{harris2020array}, SciPy \citep{2020SciPy-NMeth},  Seaborn \citep{michael_waskom_2017_883859}, Pandas \citep{mckinney2010data}, Matplotlib \citep{hunter2007matplotlib}, CVXPY \citep{diamond2016cvxpy}. As numerical optimizer we used Gurobi 10.0.1 \citep{gurobi}.

New code is published under the MIT license. To run the code, please, read the attached {\tt README.md} file for instructions.

\paragraph{Computational resources.} Experiments were run on a shared cluster node featuring a Linux OS with $2$ fourteen-core $2.6$ GHz Intel Gold 6132 and 384GB of ram. The total computation time per core to design the experiments, debug the code and obtain the final results  was roughly of 82 hours/core. To obtain the final results over $100$ seeds we estimated that 10/hours/core are sufficient.

\newpage
\section{Feedback Graph Properties}
In the following we list some properties for feedback graphs. 
\begin{lemma}\label{lemma:so_greater_than_alpha}
    For a strongly observable graph $G$ we have that $|SO(G)| \geq \max(\sigma(G),\alpha(G))$.
\end{lemma}
\begin{proof}
First, note that $\sigma(G) \leq |SO(G)|$, since all vertices with a self-loop are strongly observable.

To prove that $|SO(G)|\geq \alpha(G)$,
    by contradiction, assume that $
    \alpha(G) > |SO(G)|$. First, for any $I\in {\cal I}, u\in I$, $u$ is strongly observable. Since $|I|=\alpha(G)$, then at-least $|SO(G)|\geq\alpha(G)$, which is a contradiction.  
\end{proof}

\begin{lemma}
    Assume $|{\cal I}(G)| > 1$ and let $ I_1,I_2 \in {\cal I}(G) $, with $I_1\neq I_2$. Then, for any $u\in I_1$ we have that there exists $v\in I_2\setminus I_1$ satisfying $G_{uv}>0$ or $G_{vu}>0$. 
\end{lemma}
\begin{proof}
Let $I_1,I_2\in {\cal I}(G)$ satisfying $I_1\neq I_2\Rightarrow \exists v\in I_2\setminus I_1$. Let $u\in I_1$. By contradiction, for all $v \in I_2\setminus I_1$  we have that $G_{u,v} =0$ and $G_{v,u}=0$. In that case, we can construct a new a new set $\tilde I = I_1\cup \{v\}$ such that $G_{u,v}=0$ and $G_{v,u}=0$ for any $u,v\in \tilde I$, implying that $\tilde I\in {\cal I}(G)$.
But $|I_1|< |\tilde I|$, which contradicts the fact that $I_1\in {\cal I}(G)$.
\end{proof}
\begin{lemma}\label{lemma:so_selfloops}
    For a strongly observable graph  if $\alpha(G)>1$ then  $\forall I\in {\cal I}(G),\forall u\in I$ we have $\{u\}\in N_{in}(u)$, that is, all vertices in $I$ have self-loops. As a corollary, we have that $\sigma(G)\geq \alpha(G)$ if $\alpha(G)>1$.
\end{lemma}
\begin{proof}
    By contradiction, assume there exists $I_0\in {\cal I}(G)$ with $u\in I_0$ such that $\{u\}\notin N_{in}(u)$. Since $u$ is strongly observable, it means that $V\setminus\{u\} \in N_{in}(u)$. 

    Now, consider the case $\alpha(G)>1$. If that is the case, then let $v\in I_0$. By strong observability of $u$, we have $\{v\}\in N_{in}(u)$, which contradicts the fact that $I_0$ is an independent set.

    The latter statement is a consequence of the fact that all vertices in every $I$ have self-loops.
\end{proof}
\begin{corollary}
    Consider a strongly observable graph with $\alpha(G)>1$. Then,  $\forall I \in {\cal I} $ there exists $I_0\in {\cal I}$ such that $I\ll I_0$.
\end{corollary}
\begin{proof}
    By contradiction assume that $\exists I\in {\cal I}(G)$ such that $\forall I_0 \in {\cal I}(G)$  there exists $u(I_0)\in I$ such that $N_{in}(u(I_0))\cap I_0=\emptyset$, where $u: {\cal I}(G)\to V$. However, since the graph is strongly observable, we have that either (1) $\{u(I_0)\}\in N_{in}(u(I_0))$, or (2) $V\setminus\{u(I_0)\}\subset N_{in}(u(I_0))$ or (3) both.

    Consider the case $\alpha(G)>1$. By \cref{lemma:so_selfloops} then each vertex in $I$ has a self-loop.  Hence taking $I_0=I$ leads to $N_{in}(u(I)) \cap I \ni \{u(I)\}$, which is a contradiction.
\end{proof}

\begin{lemma}
    Consider a strongly observable graph. Then $|SO(G)|=\sigma(G)=\alpha(G)$ only for bandit feedback graphs. As a corollary, for non-bandit feedback graphs we have $|SO(G)|>\alpha(G)$.
\end{lemma}
\begin{proof}
The first part of the lemma is easy to prove as $\alpha(G)$ is maximal when all the vertices have only self-loops, thus $\alpha(G)=|SO(G)|=K$.

To prove the second part, note that it always holds true that $K=|SO(G)|\geq \alpha(G)$ by \cref{lemma:so_greater_than_alpha}. However, equality is reached only for bandit feedback graphs. Therefore, for other feedback graphs it holds that $|SO(G)|>\alpha(G)$.
\end{proof}

\begin{lemma}\label{lemma:simple_domination_selfloops}
    Consider a graph $(V,G)$, with $\alpha(G)<\sigma(G)$. For any subset $W\subset L(G)$ of size $|W|=\alpha(G)+1$, at most $\alpha(G)$ vertices are needed to dominate $W$.
\end{lemma}
\begin{proof}
    The proof is simple: first, by \cref{lemma:domination_number_inclusion}, we have $\alpha(W)\leq \alpha(G)$. Therefore $W$ is not an independent set, and there must exists $v,u\in W$ such that $u\in N_{out}(v)$. Therefore $W\setminus\{u\}$ dominates $W$.
\end{proof}

\begin{lemma}\label{lemma:domination_number_inclusion}
    For any set $G$ satisfying $\alpha(G)=k$, we must have $\max(1, k+|V|-|G|)\leq\alpha(V)\leq k$ for any subset $V$ of $G$.
\end{lemma}
\begin{proof}
The right hand-side is trivial since any subset  $V\subset G$ can have at most $k$ independent vertices.  The left hand-side follows from the fact that   removing an element from $G$ can at most reduce the number of independent vertices by $1$. 
\end{proof}

\subsection{Domination Number of the Set of Strongly Observable Vertices}
We now provide one of the main results that shows an upper bound on the number of vertices with self-loops needed to dominate the set of strongly observable vertices.

In the proofs we denote by $(V|_A, G|_A)$ the restriction of a graph $(V,G)$ to a set of vertices $A\subseteq V$. Practically, we have that $V|_A=A$, and $G|_A \in \mathbb{R}^{|A|\times |A|}$ with $(G|_A)_{u,v}= G_{u,v}$ for $u,v\in A$

We begin with the following preliminary lemma that studies the case $\alpha(G)=1$.
\begin{lemma}\label{lemma:domination_strongly_observable_1}
Let $E(G)=SO(G)\setminus L(G)$ be the set of strongly observable vertices that do not have a self-loop. Assume that $\alpha(G)=1$ and $|SO(G)|>0$. Then, at most $\sigma(G)-\left\lfloor \frac{\sigma(G)}{2}\right \rfloor$  vertices in $L(G)$ are needed to dominate the set of strongly observable vertices $SO(G)$.
\end{lemma}
\begin{proof}
    The  idea of the proof  is to find the least number of vertices in $L(G)$ that dominates $SO(G)$, and we prove this by induction. 
    
    The reason why we can do that, is that since any $v\in E(G)$ lacks a self-loop, strong observability forces $v$ to have edges from every other vertex. Hence, any vertex in $L(G)$ has an out-edge to $v\in E(G)$, and thus dominates  $E(G)$. Therefore we only need to find the domination number of $L(G)$.
    
    We begin by considering the case $\alpha(G)=1$, with $\sigma(G)=1,\sigma(G)=2,\sigma(G)=3$ and then the general case $\sigma(G) >1 $.

    First, note that $|SO(G)| \geq |L(G)| = \sigma(G)$. If $v\in E(G)$, then by definition $v\notin L(G)$, and so $v$ must have in-edges from every other vertex, that is  $N_{in}(v)=V\setminus\{v\}$.

\paragraph{Case $\sigma(G)=1$.} There is exactly one vertex $v$ with a self-loop. This single vertex trivially dominates itself and also dominates any vertex in $E(G)$. Hence the domination number of $SO(G)$ is $1$.

 \paragraph{Case $\sigma(G)=2$.} The two self-loop vertices in $L(G)=\{u,v\}$ must have a directed edge between them in some direction (or both) since $\alpha(G)=1$. Therefore, it must either be $G_{u,v}>0$ or $G_{v,u}>0$ for $L(G)=\{u,v\}$. Hence, there exists one vertex that dominates $SO(G)$.

\paragraph{Case $\sigma(G)=3$.} We can partition $L(G)$ into two sets $L_1(G),L_2(G)$, each of size $2$, such that $L_1(G)\cup L_2(G) = L(G)$ and $|L_1(G)|=|L_2(G)|$.  Since  $\alpha(L_1(G))=\alpha(L_2(G))=1$, we can repeatedly  apply the logic from the case $\sigma(G)=2$  to obtain that at most $2$ vertices are required to dominate $SO(G)$.

 \paragraph{Case $\sigma(G)\geq 4$.} 
 
 For any $\sigma(G)$, form pairs of self‐loop vertices. If $\sigma(G)$ is even, we get $\sigma(G)/2$ disjoint pairs; each pair needs just one dominator (since $\alpha(G)=1$ forces an edge in each pair). If $\sigma(G)$ is odd, then we can construct $(\sigma(G)-1)/2$ pairs plus a singleton, which requires one more dominator. Hence $(\sigma(G)-1)/2 +1=\sigma(G) - \frac{\sigma(G)-1}{2}=\sigma(G)-\lfloor\sigma(G)/2\rfloor$ suffice.

\end{proof}

\begin{theorem}\label{lemma:domination_strongly_observable}
Let $E(G)=SO(G)\setminus L(G)$ be the set of strongly observable vertices that do not have a self-loop. Assume that $|SO(G)|>0$. Then:
\begin{itemize}
    \item (Case 1) If $SO(G)\setminus E(G) = \emptyset$, then $\min(|SO(G)|,2)$ vertices are needed to dominate the set of strongly observable vertices $SO(G)$.
    \item (Case 2) If $SO(G)\setminus E(G) \neq \emptyset$, then at most $\sigma(G)-\left\lfloor \frac{\sigma(G)}{\alpha(G)+1}\right \rfloor$  vertices in $L(G)$ are needed to dominate the set of strongly observable vertices $SO(G)$.
\end{itemize}
\end{theorem}
\begin{proof}
Recall that $L(G) = \{v \in V: \{v\} \in N_{in}(v)\}$ is the set of vertices that do  have a self-loop. It follows that $L(G)\cup E(G) = SO(G)$. Moreover, if $v\in E(G)$, then by definition $v\notin L(G)$, and so $v$ must have in-edges from every other vertex, that is  $N_{in}(v)=V\setminus\{v\}$.

\begin{case}
In the first case, since $E(G)=SO(G)$, if $|SO(G)|=1$ we only need one vertex to dominate $SO(G)$. Instead, if $|SO(G)|>1$, it is possible to dominate $SO(G)$ with two vertices (just pick two vertices from $E(G)$).
\end{case}
    \begin{case} For the second case, since any $v\in E(G)$ lacks a self-loop, strong observability forces $v$ to have edges from every other vertex. Hence, any vertex in $L(G)$ has an out-edge to $v\in E(G)$, and thus dominates  $E(G)$. Therefore we only need to find the domination number of $L(G)$.

    \begin{enumerate}
        \item The case for $\alpha(G)=1$ was proved in \cref{lemma:domination_strongly_observable_1}.

        \item Now, assume $1<\alpha(G)$ and $\alpha(G)\geq \sigma(G)$. Note that we can always dominate $L(G)$ with $\sigma(G)$ elements, and since $\sigma(G)/(\alpha(G)+1)<1$, we have that the statement of case 1 holds.
        \item In the last case we have $1<\alpha(G)<\sigma(G)$ .

        Because  $\alpha(G)<\sigma(G)$, then $L(G)$ is not an independent set--there must be edges among its vertices. Hence, by \cref{lemma:simple_domination_selfloops}, in any subset of $\alpha(G)+1$ vertices from $L(G)$, it is possible to dominate that subset with at most $\alpha(G)$ vertices.
        Therefore, consider the two cases below:

\begin{itemize}
    \item If $\sigma(G)$ is divisible by $\alpha(G)+1$, then divide $L(G)$ into $m=\lfloor \sigma(G)/(\alpha(G)+1)\rfloor$  sets $L_i(G)$, each of size $\alpha(G)+1$, so that $\cup_i L_i(G)=L(G)$.

        In each set  $L_i(G)$ at most $\alpha(G)$ vertices are needed to dominate it. 
    \item Alternatively, if $\sigma(G)$ is \emph{not} divisible by $\alpha(G)+1$, then there is a leftover block of size $r<\alpha(G)+1$, with
    \[
    r + m (\alpha(G)+1)=\sigma(G).
    \]
    Note that at most $r$ vertices are needed to dominate this set.
\end{itemize}
    
Hence, in either case, the dominating‐set size of $L(G)$ is at most $m\alpha(G)+r$, which is
\[
m\alpha(G)+r = m\alpha(G)+\sigma(G)- m(\alpha(G)+1)= \sigma(G) - m=\sigma(G)-\left\lfloor \frac{\sigma(G)}{\alpha(G)+1} \right\rfloor.
\]
        
        \end{enumerate}
        \end{case}
\end{proof}
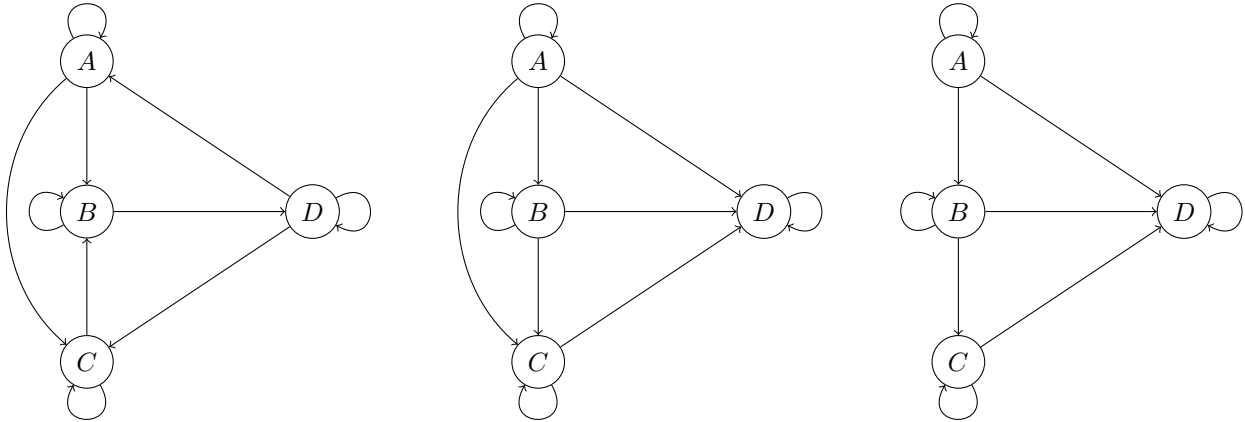
\begin{figure}[t]

\begin{minipage}[t]{0.3\textwidth}
    
\begin{tikzpicture}[
  vertex/.style = {circle, draw, minimum size=0.7cm},
  loop/.style = {looseness=3, in=60, out=120, min distance=7mm},
  every edge/.style = {draw, thick}
]

\node[vertex] (A1) at (0,0) {$A$};
\node[vertex] (B1) at (0,-2) {$B$};
\node[vertex] (C1) at (0, -4) {$C$};
\node[vertex] (D1) at (3,-2) {$D$};

\draw[->,loop] (A1) to (A1);
\draw[->,loop, in=150, out=210] (B1) to (B1);
\draw[->,loop, in=240, out=300] (C1) to (C1);
\draw[->,loop, in=330, out=30] (D1) to (D1);
\draw[->] (A1) -- (B1);
\draw[->] (C1) -- (B1);

\draw[->] (D1) -- (A1);
\draw[->] (B1) -- (D1);
\draw[->] (D1) -- (C1);

\draw[->, bend right=50] (A1) to (C1);

\end{tikzpicture}

\end{minipage}
\hfill
\begin{minipage}[t]{0.3\textwidth}
    
\begin{tikzpicture}[
  vertex/.style = {circle, draw, minimum size=0.7cm},
  loop/.style = {looseness=3, in=60, out=120, min distance=7mm},
  every edge/.style = {draw, thick}
]

\node[vertex] (A1) at (0,0) {$A$};
\node[vertex] (B1) at (0,-2) {$B$};
\node[vertex] (C1) at (0, -4) {$C$};
\node[vertex] (D1) at (3,-2) {$D$};

\draw[->,loop] (A1) to (A1);
\draw[->,loop, in=150, out=210] (B1) to (B1);
\draw[->,loop, in=240, out=300] (C1) to (C1);
\draw[->,loop, in=330, out=30] (D1) to (D1);
\draw[->] (A1) -- (B1);
\draw[->] (B1) -- (C1);

\draw[->] (A1) -- (D1);
\draw[->] (B1) -- (D1);
\draw[->] (C1) -- (D1);
\draw[->, bend right=50] (A1) to (C1);
\end{tikzpicture}
\end{minipage}
\hfill
\begin{minipage}[t]{0.3\textwidth}
    
\begin{tikzpicture}[
  vertex/.style = {circle, draw, minimum size=0.7cm},
  loop/.style = {looseness=3, in=60, out=120, min distance=7mm},
  every edge/.style = {draw, thick}
]

\node[vertex] (A1) at (0,0) {$A$};
\node[vertex] (B1) at (0,-2) {$B$};
\node[vertex] (C1) at (0, -4) {$C$};
\node[vertex] (D1) at (3,-2) {$D$};

\draw[->,loop] (A1) to (A1);
\draw[->,loop, in=150, out=210] (B1) to (B1);
\draw[->,loop, in=240, out=300] (C1) to (C1);
\draw[->,loop, in=330, out=30] (D1) to (D1);
\draw[->] (A1) -- (B1);
\draw[->] (B1) -- (C1);

\draw[->] (A1) -- (D1);
\draw[->] (B1) -- (D1);
\draw[->] (C1) -- (D1);
\end{tikzpicture}

\end{minipage}

    \caption{Example of strongly observable graphs and their domination number. {\bf On the left}: a graph with $\sigma(G)=4$ and $\alpha(G)=1$. The smallest sets of vertices that dominate this graph are $\{A,B\},\{B,D\}$. The maximally independent sets are ${\cal I}=\{\{A\},\{B\},\{C\},\{D\}\}$. Note that $\sigma(G)-\lfloor\sigma(G)/(\alpha(G)+1)\rfloor = 4-\lfloor 4/2 \rfloor = 2$.
     {\bf In the middle}: a graph with $\alpha(G)=1,\sigma(G)=4$. The smallest set of vertices that dominate the graph is $\{A\}$. The maximally independent sets are ${\cal I}=\{\{A\},\{B\},\{C\},\{D\}\}$. We have $\sigma(G)-\lfloor\sigma(G)/(\alpha(G)+1)\rfloor = 2$.
      {\bf On the right}: a graph with $\alpha(G)=2,\sigma(G)=4$. The smallest sets of vertices that dominate the graph are $\{A,B\},\{A,C\}$. The maximally independent sets are ${\cal I}=\{\{A,C\}\}$. We have $\sigma(G)-\lfloor\sigma(G)/(\alpha(G)+1)\rfloor = 4-\lfloor 4/3\rceil =3$.
     }
    \label{fig:example_graphs_proof_scaling}
\end{figure}

\newpage
\section{Sample Complexity Lower Bounds}
The sample complexity analysis delves on the required minimum amount of \emph{evidence} needed to discern between different hypotheses (e.g., vertex $v$ is optimal vs vertex $v$ is not optimal). The evidence is quantified by the log-likelihood ratio of the observations under the true model and a \emph{confusing model}. This confusing model, is usually, the model that is \emph{statistically} closer to the true model, while admitting a different optimal vertex.

To state the lower bounds, we first define the concept of \emph{absolute continuity} between two models.
For any two  models $\nu=\{G,(\nu_u)_{u\in V}\},\nu'=\{G',(\nu_u')_{u\in V}\}$ with the same number of vertices  we say that $\nu$ is absolutely continuous w.r.t $\nu'$, that is $\nu\ll \nu'$, if for all $(v,u)\in V^2$ we have $\nu_{v,u} \ll \nu_{v,u}'$

Given this definition of absolute continuity, we can define the set of confusing models as follows
\[{\rm Alt}(\nu)\coloneqq \{\nu': a^\star(\nu)\neq a^\star(\nu'), \nu \ll \nu'\},\]
which is the set of models for which $a^\star(\nu)$ is not optimal.  We also denote by ${\rm Alt}_u(\nu) = \{\nu': \mu_{u}'> \mu_{a^\star}'\}$ the set of models where $u\neq a^\star(\nu)$ may be the optimal vertex in $\nu'$.
\subsection{The Uninformed Setting}
\label{subsec:proof_lb}
In this section we prove \cref{thm:lb_general} and \cref{prop:unidentifiability_bernoulli}. We start by stating a general expression of the lower bound.
\subsubsection{General lower bound expression}
In \cref{thm:lb_general_proof} we state a general expression for $T^\star(\nu)$ and then show  in the next sections the proofs of \cref{thm:lb_general} and \cref{prop:unidentifiability_bernoulli}.
\begin{theorem}\label{thm:lb_general_proof}
     For any $\delta$-PC algorithm and any model $\nu$ satisfying \cref{assump:nontrivial_problem}, we have that
     \begin{equation}
         \mathbb{E}_\nu[\tau]\geq T^\star(\nu){\rm kl}(\delta, 1-\delta),
     \end{equation}
     where $(T^\star(\nu))^{-1}$ is equivalent to the following two expressions
     \begin{equation}
         (T^\star(\nu))^{-1}=\begin{cases}\sup_{\omega\in \Delta(V)}\inf_{\nu'\in {\rm Alt}(\nu)}\sum_{v\in V}  \omega_v \sum_{u\in N_{out}(v)}{\rm KL}(\nu_{v,u},\nu_{v,u}') ,\\ 
         \sup_{\omega\in \Delta(V)}\inf_{\nu'\in {\rm Alt}(\nu)}\sum_{u\in V}   \sum_{v\in N_{in}(u)}\omega_v{\rm KL}(\nu_{v,u},\nu_{v,u}').
         \end{cases}
     \end{equation}
     We refer to the first expression as “pull-based", partitioning the log‐likelihood by which vertex was chosen. We refer to the second expression as  “observation‐based", partitioning by which vertex  was observed.
 \end{theorem}
 \begin{proof}

Consider two bandit models $\nu=\{G,(\nu_u)_u\},\nu'=\{G',(\nu_u')_u\}$ with the same number of vertices and unique optimal vertex in both models,  such that $\nu\ll\nu'$. For each $v$ there exists a measure $\lambda_{v}$ such that $\nu_{v}$ and $\nu_v'$ have, respectively, densities $f_v$ and $f_v'$. Similarly, for $\nu_{v,u}$ and $\nu_{v,u}'$ there exists a measure $f_{v,u}$ and $f_{v,u}'$ respectively.

\paragraph{First expression (pull-based).}
Hence, consider the log-likelihood ratio between $\nu$ and $\nu'$ of the data observed up to time $t$, and consider writing it in terms of the out-neighborhood of the vertices selected by the algorithm:
\begin{align*}
L_t &= \ln \frac{{\rm d}\mathbb{P}_\nu(V_1,Z_1,\dots, V_t, Z_t)}{{\rm d}\mathbb{P}_{\nu'}(V_1,Z_1,\dots, V_t, Z_t)},\\
&=\sum_{n=1}^t \sum_{u\in N_{out}(V_n)} \ln\left(\frac{f_{v,u}(Z_{n,u})}{f_{v,u}'(Z_{n,u})}\right), \\
&=\sum_{n=1}^t \sum_{v\in V} \sum_{u\in N_{out}(v)} \mathbf{1}_{\{V_n=v\}}\ln\left(\frac{f_{v,u}(Z_{n,u})}{f_{v,u}'(Z_{n,u})}\right),\\
&=\sum_{v\in V} \sum_{u\in N_{out}(v)}\sum_{s=1}^{N_v(t)}\ln\left(\frac{f_{v,u}(W_{s,(v,u)})}{f_{v,u}'(W_{s,(v,u)})}\right),\\
 \end{align*}
 where  $(W_{s,(v,u)})_s$  is an i.i.d. sequence of samples observed from $\nu_{v,u}$. Hence, if we take the expectation with respect to $\nu$, by Wald's lemma, we have that
 \begin{align*}
 \mathbb{E}_\nu[L_t] &= \sum_{v\in V} \sum_{u\in N_{out}(v)} \mathbb{E}_\nu[N_v(t)]{\rm KL}(\nu_{v,u},\nu_{v,u}'),\\
 &=\sum_{v\in V}  \mathbb{E}_\nu[N_v(t)]\sum_{u\in N_{out}(v)}{\rm KL}(\nu_{v,u},\nu_{v,u}').
 \end{align*}
Therefore, applying  \cite[Lemma 1]{kaufmann2016complexity} at $t=\tau$, we find that for any $\delta$-PC algorithm we have that
\[
\sum_{v\in V}  \mathbb{E}_\nu[N_v(\tau)]\sum_{u\in N_{out}(v)}{\rm KL}(\nu_{v,u},\nu_{v,u}') \geq {\rm kl}(\delta,1-\delta).
\]
Consider the set of confusing models
${\rm Alt}(\nu)= \{\nu'=(\mu',G'): a^\star(\mu)\neq a^\star(\mu'), \nu \ll \nu'\},$
and define the selection rate of a vertex $v$ as
$\omega_v=\mathbb{E}_\nu[N_v(\tau)]/\mathbb{E}_\nu[\tau]$. Then, by minimizing over the set of confusing models, and then optimizing  $\omega=(\omega_v)_{v\in V}$ over the simplex $\Delta(V)$, we obtain
\[
\mathbb{E}_\nu[\tau]\underbrace{\sup_{\omega\in \Delta(V)}\inf_{\nu'\in {\rm Alt}(\nu)}\sum_{v\in V}  \omega_v \sum_{u\in N_{out}(v)}{\rm KL}(\nu_{v,u},\nu_{v,u}')}_{\eqqcolon (T^\star(\nu))^{-1}} \geq {\rm kl}(\delta,1-\delta).
\]
and therefore $\mathbb{E}_\nu[\tau]\geq T^\star  (\nu){\rm kl}(\delta,1-\delta)$.

\paragraph{Second expression (obsrvation-based).}
The second version of $T^\star(\nu)$ comes from considering the  in-neighborhood of $v$ for each vertex:

\begin{align*}
L_t &= \sum_{n=1}^t \sum_{u\in V} \sum_{v\in N_{in}(u)}\mathbf{1}_{\{V_n=v\}} \ln\left(\frac{f_{v,u}(Z_{n,u})}{f_{v,u}'(Z_{n,u})}\right),\\
&= \sum_{u\in V} \sum_{v\in N_{in}(u)}\sum_{s=1}^{N_v(t)}\ln\left(\frac{f_{v,u}(W_{s,(v,u)})}{f_{v,u}'(W_{s,(v,u)}))}\right)
\end{align*}

Hence
\begin{align*}
    \mathbb{E}_\nu[L_t] &= \sum_{u\in V}\sum_{v\in N_{in}(u)}\mathbb{E}_\nu[N_v(t)]{\rm KL}(\nu_{v,u},\nu_{v,u}'),
\end{align*}
from which we can  immediately conclude the proof by following the same steps as for the previous expression.
\end{proof}

\subsubsection{Continuous vs discrete rewards}\label{subsubsec:continuous_vs_Discrete}
Before proceeding further in our analysis, we rewrite the KL-divergence in terms of the associated Radon-Nykodim derivatives.
Note that for a product random variable $Z=XY$ with $Z\sim \nu_{X,Y}$,  we have that $\mathbb{P}_{Z}(A) = \int_A f_Z(z) {\rm d}\mu(z)$ with respect to some dominating measure $\mu(z)$. We consider some cases:
\begin{itemize}
    \item {\bf Continuous case:} for $Y$ distributed as a Bernoulli of parameter $p$, and $X$ as a continuous r.v. with density $f_X(x)$ we have that $\mathbb{P}_Z(A) = (1-p)\mathbf{1}_{\{0\in A\}} + p \int_A f_X(z) {\rm d}\lambda(z)$ where $\lambda$ is the Lebesgue measure Let
$\mu(A) = \delta_0(A) + \lambda(A)$ be the dominating measure. To find $f_Z(z)$ we can apply  the Radon-Nykodim derivative  in $z=0$, which tells us that
\[
\mathbb{P}_Z(0) = 1-p = \int_{\{0\}} f_Z(z) {\rm d}\mu(z) = f_Z(0).
\]
On the other hand for $z\neq 0$ we have
\[
\mathbb{P}_Z(A) = p\int_A f_X(z) {\rm d}\lambda(z)= \int_A f_Z(z) {\rm d}\mu(z)=\int_A f_Z(z) {\rm d}\lambda(z)  \Rightarrow f_Z(z) = p f_X(z) \hbox{ a.e.}
\]
Therefore
\[
f_Z(z) = (1-p)\mathbf{1}_{\{z=0\}} + p f_X(z)\mathbf{1}_{\{z\neq 0\}}.
\]
In words, the "continuous" part has no contribution to the overall probability mass when $z=0$, since the Lebesgue measure of $\{X=0\}$ is $0$, while for $z\neq 0$ the main contribution comes from the continuous part.  In the setting studied in this paper, the intuition is that when we observe $0$, then almost surely we know its due to the edge not being activated.
\item {\bf Discrete case:} for $Y$ distributed as a Bernoulli of parameter $p$, and $X$ as a categorical r.v. over $\{0,\dots, N\}$ with probabilities $\{q_0,\dots, q_N\}$ we have that 
$\mu(A) = \sum_{i=0}^N\delta_i(A)$, and
\[
f_Z(z) = (1-p)\mathbf{1}_{\{z=0\}} + p\left[\sum_{i=0}^N q_i \mathbf{1}_{\{z=i\}}\right],
\]
hence $\mathbb{P}_Z(Z=0) = 1-p + pq_0$ and $\mathbb{P}(Z=i)=pq_i$ for $i\in\{1,\dots,N\}$.
\end{itemize}

\subsubsection{The continuous case: proof of \cref{thm:lb_general}}\label{subsubsec:sample_complexity_lb_continuous_case}
We now consider the continuous case. From the second expression (observation-based) in  \cref{thm:lb_general_proof} we derive the result of \cref{thm:lb_general}.

\begin{proof}[Proof of \cref{thm:lb_general}]
We continue from the result of \cref{thm:lb_general_proof}.
    Note that ${\rm Alt}(\nu)=\{\nu'=(G',\{\nu_u'\}_u)\mid \exists v_0\neq a^\star: \mu_{v_0}' > \mu_{a^\star}'\} $, where $a^\star=a^\star(\mu)$. Hence, letting ${\rm Alt}_v(\nu) = \{\nu'\mid \mu_{v}'> \mu_{a^\star}'\}$, we have ${\rm Alt}(\nu)=\cup_{v\neq a^\star} {\rm Alt}_v(\nu)$. Therefore, due to the properties of the KL divergence we obtain
    \begin{align*}
        &\inf_{\nu'\in {\rm Alt}(\nu)}\sum_{u\in V}\sum_{v\in N_{in}(u)}\omega_v{\rm KL}(\nu_{v,u},\nu_{v,u}')\\
        &\qquad = \min_{u\neq a^\star}\inf_{\nu'\in {\rm Alt}_u(\nu): \mu_v'>\mu_{a^\star}'}\sum_{u\in V}\sum_{v\in N_{in}(u)}\omega_v{\rm KL}(\nu_{v,u},\nu_{v,u}').
    \end{align*}
    
    Following the discussion in \cref{subsubsec:continuous_vs_Discrete},  we can write
\begin{align*}
{\rm KL}(\nu_{v,u},\nu_{v,u}') &= \mathbb{E}_{Z\sim \nu_{v,u}}\left[ \ln \frac{{\rm d}\mathbb{P}_{\nu_{v,u}}(Z)/{\rm d}\mu(Z)}{{\rm d}\mathbb{P}_{\nu_{v,u}}'(Z)/{\rm d}\mu(Z)}\right],\\
&= \mathbb{E}_{Z\sim \nu_{v,u}}\left[ \ln \frac{f_{v,u}(Z)}{f_{v,u}'(Z)}\right],\\
&= (1-G_{v,u})\ln \frac{1-G_{v,u}}{1-G_{v,u}'} + G_{v,u}\mathbb{E}_{Z\sim \nu_{u}}\left[ \ln \frac{G_{v,u}f_{u}(Z)}{G_{v,u}'f_{u}'(Z)}\right],\\
&= {\rm kl}(G_{v,u},G_{v,u}') + G_{v,u}\mathbb{E}_{Z\sim \nu_{u}}\left[ \ln \frac{f_{u}(Z)}{f_{u}'(Z)}\right],\\
&= {\rm kl}(G_{v,u},G_{v,u}') + G_{v,u}{\rm KL}(\nu_u,\nu_u').
\end{align*}
Therefore, noting that the constraint involves only the pair $(\nu_u',\nu_{a^\star}')$ through their parameters $(\mu_u',\mu_{a^\star}')$, we conclude that
 \begin{align*}
        &\inf_{\nu'\in {\rm Alt}(\nu)}\sum_{u\in V}\sum_{v\in N_{in}(u)}\omega_v{\rm KL}(\nu_{v,u},\nu_{v,u}')\\
         &\qquad = \min_{u\neq a^\star}\inf_{\nu'\in {\rm Alt}_u(\nu): \mu_v'>\mu_{a^\star}'}\sum_{u\in V}\sum_{v\in N_{in}(u)}\omega_v \left(  {\rm kl}(G_{v,u},G_{v,u}') + G_{v,u}{\rm KL}(\nu_u,\nu_u').\right),\\
        &\qquad = \min_{u\neq a^\star}\inf_{\nu': \mu_u'\geq\mu_{a^\star}'}\sum_{v\in N_{in}(u)}\omega_v G_{v,u}{\rm KL}(\nu_u,\nu_u') + \sum_{w\in N_{in}(a^\star)}\omega_w G_{w,a^\star}{\rm KL}(\nu_{a^\star},\nu_{a^\star}'),\\
        &\qquad = \min_{u\neq a^\star}\inf_{\nu': \mu_u'\geq\mu_{a^\star}'} m_u{\rm KL}(\nu_u,\nu_u') + m_{a^\star}{\rm KL}(\nu_{a^\star},\nu_{a^\star}'),
    \end{align*}
    where $m_u\coloneqq \sum_{v\in N_{in}(u)}\omega_v G_{v,u}$ and $m_{a^\star}=\sum_{\omega\in N_{in}(a^\star)}\omega_w G_{w,a^\star}$.
    
    Therefore, by optimizing over $\nu'$ as in \cite[Lemma 3]{garivier2016optimal} we obtain
    \[
    (T^\star(\nu))^{-1} = \sup_{\omega\in \Delta(V)} \min_{u\neq a^\star} (m_u +m_{a^\star}) I_{\frac{m_{a^\star}}{m_u+m_{a^\star}}}(\nu_{a^\star},\nu_u) \hbox{ s.t. } m_u = \sum_{v\in N_{in}(u)}\omega_v G_{v,u}.
    \]
\end{proof}

\subsubsection{The discrete case: proof of \cref{prop:unidentifiability_bernoulli}}\label{subsubsec:sample_complexity_discrete_case}

We find that if $(\nu_u)_{u\in V}$ are Bernoulli distributed, then we obtain that it is not possible, in general, to estimate the best vertex. The reason is simple: without knowing which edge was activated, the learner does not know how to discern between the randomness of the reward and the randomness of the edge.

\begin{proof}[Proof of \cref{prop:unidentifiability_bernoulli}]
    Let $n\in \mathbb{R}_+^K$ and consider the optimization problem
    \[\inf_{\nu'\in {\rm Alt}(\nu)}\sum_{u\in V} \sum_{v\in N_{in}(u)} n_v{\rm KL}(\nu_{v,u},\nu_{v,u}').\]
    If $\nu_{v,u}$ is a Bernoulli distribution of parameter $q_{v,u}\coloneqq G_{v,u}\mu_u$ (sim. for $\nu_{v,u}'$ with $q_{v,u}'\coloneqq G_{v,u}'\mu_u'$), then we can write
    \begin{align*}
    \inf_{\nu'\in {\rm Alt}(\nu)}\sum_{u\in V}  \sum_{v\in N_{in}(u)}n_v{\rm KL}(\nu_{v,u},\nu_{v,u}')=\inf_{\nu'\in {\rm Alt}(\nu)}\sum_{u\in V}  \sum_{v\in N_{in}(u)}n_v{\rm kl}(q_{v,u},q_{v,u}').
    \end{align*}
    Using the fact that $\nu'\in {\rm Alt}(\nu)$, and that $\nu'$ is observable, it must imply that there exists $(v,u)\in V^2$ such that $\mu_u'>\mu_{a^\star}'$ and $G_{v,u}'>0$. Similarly, from the observability of $\nu$ $\exists w\in V$ s.t. $G_{w,a^\star}>0$. Hence, by absolute continuity, we have $G_{w,a^\star}'>0$, otherwise the event ${\cal E}=\{Z_{w,a^\star}=1\}$ would satisfy $\mathbb{P}_\nu({\cal E})>0$ and $\mathbb{P}_{\nu'}({\cal E}) = 0$.
    
    Then, similarly as in previous results,  we have that
\begin{align*}
    &\inf_{\nu'\in {\rm Alt}(\nu)}\sum_{u\in V}  \sum_{v\in N_{in}(u)}n_v{\rm kl}(q_{v,u},q_{v,u}'),\\
    &\qquad=\min_{u\neq a^\star}\min_{v\in N_{in}(u),w\in N_{in}(a^\star)}\inf_{\mu',G': \mu_{u}'\geq\mu_{a^\star}', G_{v,u}',G_{w,a^\star}'\geq 0} n_v{\rm kl}(q_{v,u},q_{v,u}')  + n_w{\rm kl}(q_{w,a^\star},q_{w,a^\star}').
    \end{align*}
    Therefore, for $u\neq a^\star$, $v\in N_{in}(u), w\in N_{in}(a^\star)$, we are interested in the following non-convex problem
    \begin{equation*}
    \begin{aligned}
     \min_{\mu', G'>0} \quad & n_v{\rm kl}(q_{v,u},q_{v,u}')  + n_w{\rm kl}(q_{w,a^\star},q_{w,a^\star}')\\
    \textrm{s.t.} \quad 
    &q_{v,u}'=G_{v,u}'\mu_{u}'  \qquad \forall (v,u)\in V^2,\\
      &\mu_u' \geq \mu_{a^\star}'.
    \end{aligned}
\end{equation*}
However, the solution, as one may expect, is $0$. Simply note one can take $q_{v,u}'=q_{v,u}$ by choosing $\mu_{u}'=\mu_{a^\star}$ and $G_{v,u}'=q_{v,u}/\mu_{a^\star}$ (which is $\leq 1$). Similarly, we have $q_{w,a^\star}'=q_{w,a^\star}$ by choosing $ \mu_{a^\star}'=\mu_{a^\star}$ and $G_{w,a^\star}'= q_{w,a^\star}/\mu_{a^\star}= G_{w,a^\star}$. Henceforth, in the Bernoulli  case we have that $(T^\star(\nu))^{-1} = 0$.
\end{proof}

\subsection{The Informed Setting}\label{subsubsec:informed_setting_appendix}

We now give a proof of \cref{thm:lb_informed}, which follows closely the one in the uninformed case.

\begin{proof}[Proof of \cref{thm:lb_informed}]
First note that knowing the set of edges that was activated is equivalent to knowing the value of $(Y_{t,(V_t,u)})_u$. Then, denote the density associated to $Y_{t,(V_t,u)}$ by $g_{v,u}$ under $\nu$ (sim. $g_{v,u}'$ for $\nu'$). We can write the log-likelihood ratio as
    \begin{align*}
L_t &= \sum_{n=1}^t \sum_{u\in V} \sum_{v\in N_{in}(u)}\mathbf{1}_{\{V_n=v\}} \ln\left(\frac{{\rm d}\mathbb{P}_\nu(Z_{n,u}, Y_{n,(v,u)} V_n)}{{\rm d}\mathbb{P}_{\nu'}(Z_{n,u}, Y_{n,(v,u)}, V_n)}\right),\\
 &= \sum_{n=1}^t \sum_{u\in V} \sum_{v\in N_{in}(u)}\mathbf{1}_{\{V_n=v\}}\Bigg[\mathbf{1}_{\{Y_{n,(v,u)}=1\}} \ln\left(\frac{f_u(Z_{n,u}) g_{v,u}(1)}{f_u'(Z_{n,u}) g_{v,u}'(1)}\right)\\
&\qquad\qquad \qquad +\mathbf{1}_{\{Y_{n,(v,u)}=0\}} \ln\left(\frac{g_{v,u}(0)}{ g_{v,u}'(0)}\right)\Bigg],\\
 &=  \sum_{u\in V} \sum_{v\in N_{in}(u)}\sum_{n=1}^{N_v(t)}\Bigg[\mathbf{1}_{\{Y_{n,(v,u)}=1\}} \ln\left(\frac{f_u(W_{n,u}) g_{v,u}(1)}{f_u'(W_{n,u}) g_{v,u}'(1)}\right)  +\mathbf{1}_{\{Y_{n,(v,u)}=0\}} \ln\left(\frac{g_{v,u}(0)}{ g_{v,u}'(0)}\right)\Bigg].
\end{align*}
where $(W_{n,u})_n$ is a sequence of i.i.d. random variables distributed according to $\nu_u$. 
Hence
\begin{align*}
\mathbb{E}_\nu[L_\tau] &= \sum_{u\in V} \sum_{v\in N_{in}(u)}\mathbb{E}_\nu\Bigg[ \sum_{n=1}^{\infty}\mathbf{1}_{\{N_v(\tau)\geq n\}}\Big[\mathbf{1}_{\{Y_{n,(v,u)}=1\}} \ln\left(\frac{f_u(W_{n,u}) }{f_u'(W_{n,u}) }\right)\\
&\qquad\qquad \mathbf{1}_{\{Y_{n,(v,u)}=1\}}\ln\left(\frac{ g_{v,u}(1)}{ g_{v,u}'(1)}\right) +\mathbf{1}_{\{Y_{n,(v,u)}=0\}} \ln\left(\frac{g_{v,u}(0)}{ g_{v,u}'(0)}\right)\Big] \Bigg].
\end{align*}
Note that $\{N_v(\tau)\geq n\} = \{N_v(\tau) \leq n-1\}\in {\cal F}_{n-1}$, therefore $W_{n,u}$ and $Y_{n,(v,u)}$ are independent of that event. Hence
\begin{align*}
\mathbb{E}_\nu[L_\tau] &= \sum_{u\in V} \sum_{v\in N_{in}(u)} \sum_{n=1}^{\infty}\mathbb{P}_\nu(N_v(\tau)\geq n) \left[G_{v,u} {\rm KL}(\nu_u,\nu_u') + {\rm kl}(G_{v,u}, G_{v,u}')\right],\\
&= \sum_{u\in V} \sum_{v\in N_{in}(u)} \mathbb{E}_\nu[N_v(\tau)] \left[G_{v,u} {\rm KL}(\nu_u,\nu_u') + {\rm kl}(G_{v,u}, G_{v,u}')\right].
\end{align*}

Therefore, applying  \cite[Lemma 1]{kaufmann2016complexity} at $t=\tau$,
$
\mathbb{E}_\nu[L_\tau] \geq {\rm kl}(\delta,1-\delta)$.

Consider the set of confusing models
${\rm Alt}(\nu)= \{\nu'=(\mu',G'): a^\star(\mu)\neq a^\star(\mu'), \nu \ll \nu'\},$
and define the selection rate of a vertex $v$ as
$\omega_v=\mathbb{E}_\nu[N_v(\tau)]/\mathbb{E}_\nu[\tau]$. Then, by minimizing over the set of confusing models, and then optimizing over $\omega=(\omega_v)_{v\in V}$ over the simplex $\Delta(V)$, we obtain
\[
\mathbb{E}_\nu[\tau]\underbrace{\sup_{w\in \Delta(V)}\inf_{\nu'\in {\rm Alt}(\nu)}\sum_{u\in V} \sum_{v\in N_{in}(u)} \omega_v \left[G_{v,u} {\rm KL}(\nu_u,\nu_u') + {\rm kl}(G_{v,u}, G_{v,u}')\right]}_{\eqqcolon (T^\star(\nu))^{-1}} \geq {\rm kl}(\delta,1-\delta).
\]
Hence, consider now the expression $\inf_{\nu'\in {\rm Alt}(\nu)}\sum_{u\in V} \sum_{v\in N_{in}(u)} w_v \left[G_{v,u} {\rm KL}(\nu_u,\nu_u') + {\rm kl}(G_{v,u}, G_{v,u}')\right]$ and observe that it simplifies as in the proof of \cref{thm:lb_general}:
\begin{align*}
    \inf_{\nu'\in {\rm Alt}(\nu)}&\sum_{u\in V} \sum_{v\in N_{in}(u)} \omega_v \left[G_{v,u} {\rm KL}(\nu_u,\nu_u') + {\rm kl}(G_{v,u}, G_{v,u}')\right]\\
     &=\min_{u\neq a^\star}\inf_{G',\nu_u',\nu_{a^\star}': \mu_u'\geq\mu_{a^\star}'}\sum_{v\in N_{in}(u)} \omega_v G_{v,u} {\rm KL}(\nu_u,\nu_u')  + \sum_{v\in N_{in}(a^\star)} \omega_v G_{v,a^\star} {\rm KL}(\nu_{a^\star},\nu_{a^\star}') ,\\
     &=\min_{u\neq a^\star}\inf_{G',\nu_u',\nu_{a^\star}': \mu_u'\geq\mu_{a^\star}'} m_u {\rm KL}(\nu_u,\nu_u')  + m_{a^\star} {\rm KL}(\nu_{a^\star},\nu_{a^\star}'),
\end{align*}
as in the proof of \cref{thm:lb_general}. For the known graph case, note that the set of confusing models becomes $
{\rm Alt}'(\nu) \coloneqq \{\nu'=(\mu', G): a^\star(\mu)\neq a^\star(\mu'), \nu_u\ll \nu_u' \ \forall u\in V\}
$, from which one can conclude the same result.
\end{proof}

\subsection{Scaling Properties}\label{app:subsec_scaling}

To gain a better intuition of the characteristic time in \cref{thm:lb_general}, we can focus on the Gaussian case where $\nu_u = {\cal N}(\mu_u,\lambda^2)$. For this case we have that ${\rm KL}(\nu_u,\nu_v) = (\mu_u-\mu_v)^2/(2\lambda^2)$, and $I_\alpha(\nu_u,\nu_v) = \dfrac{\alpha(1-\alpha)(\mu_u-\mu_v)^2}{2\lambda^2}$.
Therefore $T^\star(\nu)$ can be computed by solving the following convex problem
          \begin{equation}\label{eq:characteristic_time_gaussian}
    \begin{aligned}
   T^\star(\nu) =& \mspace{-15mu}  \inf_{m\in \mathbb{R}_+^K,\omega\in \Delta(V)} \max_{u\neq a^\star}\left(m_u^{-1} + m_{a^\star}^{-1}\right) \frac{2\lambda^2}{\Delta_u^2}\\
   &\hbox{ s.t. } m_u = \sum_{v\in N_{in}(u)}\omega_v G_{v,u} \quad \forall u\in V.
    \end{aligned}
\end{equation}

\subsubsection{General case}
We restate here \cref{prop:scaling_weakly_observable_gaussian} and provide a proof.
\begin{proposition}
    For an observable model $\nu=(\{\nu_u\}_u, G)$ with a graph  $G$ satisfying $\delta(G)+\sigma(G)>0$ and  Gaussian random rewards $\nu_u={\cal N}(\mu_u, \lambda^2)$ we can upper bound $T^\star$ as
    \begin{equation*}
        T^\star(\nu) \leq \frac{4\left[\delta(G)+\sigma(G)-\left\lfloor \frac{\sigma(G)}{\alpha(G)+1}\right \rfloor\right]\lambda^2}{\min_{u\neq a^\star} \min(\bar G_u, \bar G_{a^\star}) \Delta_u^2},
    \end{equation*}
    where $\bar G_u \coloneqq \max\left(  \max_{v\in D(G)} G_{v,u},\min_{v\in L(G): G_{v,u}>0}G_{v,u} \right )$ (sim. $\bar G_{a^\star}$).
\end{proposition}
\begin{proof}
Since $\sigma(G)+\delta(G)>0$, the graph has at least one of the following: (i) a positive weak‐domination number $\delta(G)$ or (ii) at least one strongly observable vertex with a self‐loop $\sigma(G)$ (or both properties hold at the same time).

Consider  the expression of $T^\star(\nu)$ in \cref{eq:characteristic_time_gaussian}. Note
     that for each weakly observable vertex $u$ there exists $v\in N_{in}(u)\cap D(G)$, where $D(G)$ is the smallest set dominating the set of weakly observable vertices (see \cref{def:graph_dependent_quantities}).

Observe the following two properties:
\begin{enumerate}
    \item  For any strongly observable vertex  by \cref{lemma:domination_strongly_observable} we  need at-most $\sigma(G)-\left\lfloor \frac{\sigma(G)}{\alpha(G)+1}\right \rfloor$ vertices in $L(G)$ to dominate $|SO(G)|$ if $\sigma(G)>0$. 
    \item If $\sigma(G)=0$, then by definition all strongly observable vertices lack self‐loops, and $\delta(G)$ must be positive; thus $D(G)$ dominates  the entire graph.
\end{enumerate}

    Therefore we need at-most $\kappa=\delta(G)+\sigma(G)-\left\lfloor \frac{\sigma(G)}{\alpha(G)+1}\right \rfloor$ vertices to dominate the graph. Using this information, we allocate probability mass uniformly across these vertices.

    \begin{enumerate}
        \item For any vertex $u\in W(G)$ let $\omega_v = 1/\kappa$ for all $v\in D(G)$. This fact allows us to lower bound $m_u$ as
    \begin{align*}
    m_u &=  \sum_{v\in N_{in}(u)\setminus D(G)}\omega_v G_{v,u} + \sum_{w\in D(G)\cap N_{in}(u)} \omega_w G_{w,u},\\
    &\geq   \sum_{w\in D(G)\cap N_{in}(u)} \omega_w G_{w,u},\\
    &\geq   \sum_{w\in D(G)\cap N_{in}(u)} \frac{1}{\kappa} G_{w,u},\\
    &\geq   \max_{w\in D(G)\cap N_{in}(u)} \frac{1}{\kappa} G_{w,u},\\
    &\geq   \max_{w\in D(G)} \frac{1}{\kappa} G_{w,u}.
    \end{align*}

    \item For any vertex $u\in L(G)$ with a self-loop we can lower bound $m_u$ as 
    \[
    m_u \geq \min_{v\in L(G): G_{v,u}>0} \frac{1}{\kappa} G_{v,u}.
    \]

    \end{enumerate}
    Therefore, for any $u\in V$ we have 
    \[
    \frac{1}{\kappa m_{u}} \leq \begin{cases}
        \dfrac{1}{\min_{v\in L(G): G_{v,u}>0}G_{v,u}} & \hbox{ if } u\in L(G),\\
        \dfrac{1}{\max_{v\in D(G)} G_{v,u}} & \hbox{ otherwise},
    \end{cases}
    \]
    Then, let $\bar G_u \coloneqq \max\left(  \max_{v\in D(G)} G_{v,u},\min_{v\in L(G): G_{v,u}>0}G_{v,u} \right )$. In other words, $\bar G_u$  captures the relevant edge‐activation probability from either a dominating vertex or a self‐loop vertex.
    We conclude that
    \begin{align*}
    T^\star(\nu) &\leq \max_{u\neq a^\star}\kappa\left(\bar G_{u}^{-1} + \bar G_{a^\star}^{-1} \right) \frac{2\lambda^2}{\Delta_u^2},\\
     &\leq \frac{4\kappa\lambda^2}{\min_{u\neq a^\star} \min(\bar G_{u}, \bar G_{a^\star}) \Delta_u^2},
    \end{align*}
    where in the last expression we used $a+b\leq 2\max(a,b)$.
\end{proof}

\subsubsection{The loopless clique}
We now consider the scaling for the case  $\delta(G)+\sigma(G)=0$, which corresponds to the loopless clique. We restate here \cref{prop:scaling_looplessclique} and provide a proof.
\begin{proposition}
    For an observable model $\nu=(\{\nu_u\}_u, G)$ with $\delta(G)+\sigma(G)=0$, and  Gaussian random rewards $\nu_u={\cal N}(\mu_u, \lambda^2)$, we can upper bound $T^\star$ as
    \begin{equation*}
        T^\star(\nu) \leq 
            \dfrac{4\bar G \lambda^2}{\Delta^2_{\rm min}}, 
    \end{equation*}
    where \[\bar G \coloneqq  \min_{(v,w)\in V^2: v\neq w} \max_{u\neq a^\star}\frac{1}{G_{v,w}(u)} + \frac{1}{G_{v,w}(a^\star)},\]
    and $ G_{v,w}(u)\coloneqq G_{v,u} + G_{w,u}$.
\end{proposition}
\begin{proof}
The proof uses the idea that only $2$ vertices are needed to dominate the graph. Hence, we find two vertices $(v,w)$ that allow to bound $ \max_{u \neq a^*} (m_u^{-1} + m_{a^*}^{-1}) $ in the expression of $T^\star$.

 First,  $\delta(G)=0$ implies that there are no weakly observable vertices. Since we also  have $\sigma(G)=0$ then $V=E(G)$ (where $E(G)$ is the set of strongly observable vertices without self-loops). Hence, by definition, the graph  is the loopless clique. By \cref{lemma:domination_strongly_observable} we  need $2$ vertices to dominate the graph. 
        Define the following set
        \[
        {\cal A} \coloneqq \argmin_{(v,w)\in V^2: v\neq w} \max_{u\neq a^\star}  G_{v,w}(u)^{-1} + G_{v,w}(a^\star)^{-1}, \hbox{ where } G_{v,w}(u)\coloneqq G_{v,u} + G_{w,u},
        \]
        and denote by $\bar G \coloneqq  \min_{(v,w)\in V^2: v\neq w} \max_{u\neq a^\star}G_{v_0,w_0}(u)^{-1} + G_{v_0,w_0}(a^\star)^{-1} $ the optimal value for any $(v_0,w_0)\in {\cal A}$.
        Then, let $\omega_{v_0} = \omega_{w_0}=1/2$ for a generic pair $(v_0,w_0)\in {\cal A}$. For any $u\in V$ we have $m_u $
        \begin{align*}
            m_u &= \sum_{v\in V} w_v G_{v,u} \\
            &\geq (  G_{v_0, u} +   G_{w_0, u})/2.
        \end{align*}
        Note that by Lemma \ref{lemma:domination_strongly_observable}, \(m_u > 0\) for any \(u \in V\). Therefore,
        \begin{align*}
            m^{-1}_u \leq 2 G_{v_0,w_0}(u)^{-1}.
        \end{align*}
        Since the bound hold for any $(v_0,w_0)\in {\cal A}$, we conclude that
        \begin{align*}
             T^*(\nu) &\leq \max_{u \neq a^*} (m_u^{-1} + m_{a^*}^{-1}) \frac{2\lambda^2}{\Delta^2_u}, \\
             &\leq \max_{u \neq a^*} \left(G_{v_0,w_0}(u)^{-1} + G_{v_0,w_0}(a^\star)^{-1}\right) \frac{4\lambda^2}{\Delta_{\rm min}^2 },\\
             &=  \frac{4\bar G \lambda^2}{\Delta^2_{\rm min}}.
        \end{align*}
\end{proof}

\subsubsection{Heuristic solution: scaling and spectral properties}\label{subsubsec:app_heuristic_sol}
In general we find it  hard to characterize the scaling of $T^\star(\nu)$ in terms of the spectral properties of $G$ without a more adequate analysis. Furthermore, we also wonder if it is possible to find a closed-form solution that can be easily used.

Intuition suggests that, by exploiting the underlying topology of the graph, a good solution $\omega$ should be sparse (which, in turns, helps to minimize the sample complexity). However, it may be hard to find a simple sparse solution. 

To that aim, we can gain some intuition from the BAI problem for the classical multi-armed bandit setup. From the analysis in \citet[Appendix A.4]{garivier2016optimal}, an approximately optimal solution in the Gaussian case is given by $\omega_u^\star \propto 1/\Delta_u^2$, with $\Delta_{a^\star} = \Delta_{\rm min}$.

Hence, we also propose  that $m_u \propto  1/\Delta_u^2$, with $m=G^\top \omega$. At this point we could try to minimize the MSE loss between $m$ and the vector $\Delta^{-2}\coloneqq (1/\Delta_u^2)_{u\in V}$, subject to $\|w\|_1=1$. However, in this problem we are more interested in the directional alignment between $m$ and $\Delta^{-2}$, rather than in their magnitude, since the magnitude of $\omega$ is constrained \footnote{Note that also the MSE problem $\argmin_\beta \|y-A\beta\|_2^2 = \argmin_\beta  -2y^\top A\beta+ \|A\beta\|_2^2$ also tries to solve a problem of alignment through the term $2y^\top A\beta$, while the second term $\|A\beta\|_2^2$ can be considered a form of regularization.}.  Therefore, one may be interested in maximizing teh similarity $m^\top \Delta^{-2}$, or rather
\[
\max_{\omega: \|\omega\|_2\leq \alpha}  \omega^\top G\Delta^{-2}
\]
for some constraint $\alpha$.
Obviously the solution is  $\omega= \alpha G\Delta^{-2}/\|G\Delta^{-2}\|_2$ in the classical Euclidean space with the $\ell_2$ norm. However, such solution is not a distribution. To that aim, we project $G\Delta^{-2}$ on the closest distribution in the KL sense, defined as ${\rm Proj}_{\rm KL}(x) \coloneqq \min_{p\in {\cal P}} {\rm KL}(p,x)$ for any $x$ such that $x_u\geq 0$ for every $u$.
\begin{lemma}
    The projection of $G\Delta^{-2}$ in the KL sense is given by $\omega_{\rm heur}=G\Delta^{-2}/\|G\Delta^{-2}\|_1$.
\end{lemma}
\begin{proof}
    For some vector $x$ satisfying $x\geq 0$, write the Lagrangian of $\min_{p\in {\cal P}} {\rm KL}(p,x)$:
    \[{\cal L}(p,\lambda)=\sum_u  p_u \ln \left( \frac{p_u}{x_u} \right) + \lambda \left(1-\sum_u p_u\right).\]
    Then, we check the first order condition $\partial {\cal L}/\partial p_u =0 \Rightarrow \ln \left( \frac{p_u}{x_u} \right) +1  - \lambda =0$, implying that $p_u=x_ue^{\lambda-1}$. Using the fact that $\sum_u p_u =1$ we also obtain $e^{1-\lambda}=\sum_u x_u= \|x\|_1$. Therefore $\lambda=1-\ln(\|x\|_1)$, from which we conclude that $p_u = x_u/\|x\|_1$.
\end{proof}

Such allocation $w_{\rm heur} \coloneqq \frac{G \Delta^{-2}}{\|G\Delta^{-2}\|_1}$ makes intuitively sense: a vertex $u$ will be chosen with probability proportional to $ \sum_{v\in V} G_{uv} \Delta_{v}^{-2}$, thus assigning higher preference to vertices that permit the learner to sample arms with small sub-optimality gaps.

 For this heuristic allocation $w_{\rm heur}$ we can provide the following upper bound on its scaling.

\begin{lemma}
    For a model $\nu=(\{\nu_u\}_u, G)$ with an observable graph  $G$ and  Gaussian random rewards $\nu_u={\cal N}(\mu_u, \lambda^2)$ we can upper bound $T(\omega_{\rm heur};\nu)$  as
    \begin{equation*}
        T^\star (\nu) \leq T(\omega_{\rm heur};\nu)\leq \frac{\|G\Delta^{-2}\|_1}{\sigma_{\min}(G)^2} \cdot 4\lambda^2 \leq \frac{K\min\left(\sqrt{K}\sigma_{\max}(G), \sum_i\sigma_i(G)\right)}{\Delta_{\min}^2\sigma_{\min}(G)^2} \cdot 4\lambda^2.
    \end{equation*}
    where $\sigma_{\min}(G), \sigma_{\max}(G), \sigma_i(G)$ are, respectively the minimum singular value of $G$,  the maximum singular value of $G$ and the $i$-th singular value of $G$ (to not be confused with $\sigma(G)$!).
\end{lemma}
\begin{proof}
    Again, we prove this corollary by lower bound each \(m_u\). Note that we have $m=G^\top G \omega_{\rm heur}$. Denote by $G_u$ the $u$-th row of $G$, then 
    \begin{align*}
        m_u &= \frac{\sum_{v \in V} (G^T G)_{u,v} \Delta_{v}^{-2}}{\|G\Delta^{-2}\|_1} \\
        &\geq \frac{\|G_u\|^2_2 \Delta_{u}^{-2}}{\|G\Delta^{-2}\|_1} \\
        &\geq  \frac{\|G_u\|^2_2 \Delta_u^{-2}}{\|G\Delta^{-2}\|_1}.
    \end{align*}
    And therefore, using that $\|G_u\|_2^2 \geq \sigma_{\min}(G)^2$, we observe that
    \begin{align*}
        T(\omega_{\rm heur};\nu) &\leq \max_{u \neq a^{\star}} \left(\frac{\Delta_u^2}{\|G_u\|^2_2 } + \frac{\Delta_{\rm min}^2}{\|G_{a^{\star}}\|^2_2 }\right) \frac{2 \|G\Delta^{-2}\|_1\lambda^2}{\Delta_u^2},\\
         &\leq \max_{u \neq a^{\star}} \left(\Delta_u^2+ \Delta_{\rm min}^2\right) \frac{2 \|G\Delta^{-2}\|_1\lambda^2}{\sigma_{\min}(G)^2\Delta_u^2},\\
         &\leq  \frac{\|G\Delta^{-2}\|_1}{\sigma_{\min}(G)^2} \cdot 4\lambda^2.
    \end{align*}

  Lastly, denoting by $G_u$ the $u$-th row of $G$, we obtain that $\|G\Delta^{-2}\|_1 \leq \sum_u |G_u^\top \Delta^{-2}| \leq  \|\Delta^{-2}\|_2 \sum_u \|G_u\|_2 \leq \frac{\sqrt{K}}{\Delta_{\rm min}^2} \sum_u \|G_u\|_2$. We conclude by noting that $\sum_u \|G_u\|_2\leq K \sigma_{\max}(G)$, and thus $\|G\Delta^{-2}\|_1 \leq  \frac{K^{3/2}\sigma_{\max}(G)}{\Delta_{\rm min}^2}$

   Additionally, we also note that $\|G\Delta^{-2}\|_1 \leq \sum_u |G_u^\top \Delta^{-2}| \leq  \|\Delta^{-2}\|_\infty \sum_u \|G_u\|_1$ by Holder's inequality. Now, let $\|\cdot\|_*$ denote the Schatten-1 norm. Using that $\|{\rm vec}(G)\|_1 \leq K \|G\|_* = K\sum_i \sigma_i(G)$ we have $\|G\Delta^{-2}\|_1 \leq \frac{K\sum_i \sigma_i(G)}{\Delta_{\rm min}^2}$. 
\end{proof}

\paragraph{An alternative approach that is sparse.} We note that the above analysis does not take fully advantage of the graph structure. An alternative approach that yields a better scaling is to instead consider the similarity problem
\[
\max_{\omega: \|\omega\|_1=1}  \omega^\top G\Delta^{-2}.
\]
The optimal solution then is simply $\omega_u = \mathbf{1}_{\{u\in {\cal G}\}}/|{\cal G}|$, where ${\cal G}=\argmax_{v} (G\Delta^{-2})_v$. This is an efficient allocation, since it scales as $O\left(\frac{|{\cal G}|}{\Delta_{\min}^2 \max_{u\in {\cal G}}\min_v G_{u,v}} \right)$. However, such solution is admissible only if it guarantees that  $V\ll {\cal G}$ , i.e., this set of vertices ${\cal G}$ dominates the graph.

Alternatively, one can choose the top $k$ vertices ${\cal U}=\{u_1,u_2,\dots, u_k\}$ ordered according to $(G\Delta^{-2})_{u_1}\geq\dots \geq (G\Delta^{-2})_{u_k}\geq \dots (G\Delta^{-2})_{u_K}\ $ satisfying $V\ll \{u_1,\dots, u_k\}$ (i.e., these vertices dominate the graph). Then, one can simply let $\omega_u = \mathbf{1}_{\{u\in {\cal U}\}}/|{\cal U}|$. Since these are also the vertices that maximize the average information collected from the graph, we believe this solution to be sample efficient. A simple analysis, shows that the worst case scaling in this scenario is $O\left( \frac{|{\cal U|}}{\Delta_{\rm min}^2 \max_{u\in {\cal U}} \min_v G_{u,v}}\right)$.
\newpage
\section{Analysis of \algoname}\label{sec:app_algorithm}
In this section we provide an analysis of \algoname{} for an observable model in the uninformed case (with continuous rewards) or in the informed case. In \cref{subsec:app_sampling_rule} we analyse the sampling rule. In \cref{subsec:app_stopping_rule} we analyse the stopping rule. Lastly, in \cref{subsec:algorithm_as_sample_complexity_app}, we analyse the sample complexity.

\subsection{Sampling Rule}\label{subsec:app_sampling_rule}
The proof of the tracking proposition \cref{prop:tracking} is inspired by D-tracking \cite{garivier2016optimal,jedra2020optimal}. In \cite{degenne2019pure} they show that classical D-tracking \cite{garivier2016optimal} may fail to converge when $C^\star(\nu)$ is a convex set of possible optimal allocations. However, using a modified version it is possible to prove the convergence. We take inspiration from \cite{jedra2020optimal}, where they applied a modified D-tracking to the linear bandit case and showed convergence of $ w^\star(t)$.

The intuition behind the proof is that tracking the average of the converging sequence $( \omega^\star(t))_t$, which is a convex combination, converges to a stable point in the convex set $C^\star(\nu)$. The proof makes use of the following result, from \cite{Bonsall1963CB}.

\begin{theorem}[Maximum theorem \cite{Bonsall1963CB}]\label{thm:berge}
Let $C^\star(\nu)=\arginf_{\omega\in \Delta(V)}T(\omega;\nu)$. Then $T^\star(\nu)\coloneqq T(\omega^\star;\nu), \omega^\star\in C^\star(\nu),$ is continuous at $\nu$ (in the sense of $(G,\{\mu_u\}_u)\in [0,1]^{K\times K} \times [0,1]^K$), $C^\star(\nu)$ is convex, compact and non-empty.
Furthermore, we have that for any open neighborhood ${\cal V}$ of $C^\star(\nu)$, there exists an open neighborhood ${\cal U}$ of $\nu$, such that for all $\nu'\in {\cal U}$ we have $C^\star(\nu')\subseteq {\cal V}$.
\end{theorem}

Here we state, and prove, a more general version,  of \cref{prop:tracking} (which follows by taking $\alpha_{t,n}=1/t$ in the next proposition).

\begin{proposition} \label{prop::tracking_general}
    Let $S_t= \{u\in V: N_u(t) < \sqrt{t}-K/2\}$. The D-tracking rule, defined as
    \begin{equation}
        V_t \in \begin{cases}
            \argmin_{u\in S_t} N_u(t) & S_t\neq \emptyset\\
            \argmin_{u\in V} N_u(t) - t\sum_{n=1}^t \alpha_{t,n}\omega_u^\star(n)& \hbox{otherwise}
        \end{cases},
    \end{equation}
    where, for every $t\geq 1$, the sequence $\alpha_t=(\alpha_{t,n})_{n=1}^t $ satisfies: (1) $\alpha_{t,n}\in [0,1]$; (2) for every fixed $n\in\{1,\dots,t\}$ we have  $\alpha_{t,n} =o(1)$ in $t$ ; (3) for all $t$, $\sum_{n=1}^t \alpha_{t,n}=1$.

    Such tracking rule
    ensures that for all $\epsilon>0$ there exists $t(\epsilon)$ such that for all $t \geq t(\epsilon)$ we have
    \[
    \|N(t)/t - \bar v(t)\|_\infty \leq 5(K-1)\epsilon,
    \]
    where $\bar v(t) \coloneqq \arginf_{\omega\in C^\star(\nu)} \|\omega-\sum_{n=1}^t \alpha_n \omega^\star(n)\|_\infty$, and 
    $
   \lim_{t\to\infty}\inf_{\omega\in C^\star(\nu)}\| N(t)/t -\omega\|_\infty 
 \to 0 
    $ almost surely.
\end{proposition}
\begin{proof}
Define the projection of $x\in \mathbb{R}^n$ onto $C$ as ${\rm Proj}_C(x) \coloneqq \arginf_{\omega\in C} \|x - w\|_\infty$, which is guaranteed to exists if $C$ is convex and compact. 

Let  $n(t)\coloneqq {\rm Proj}_{C^\star(\nu)}(N_t/t)$. The proof lies showing that $\| N_t/t - n(t)\|_\infty \to 0$ as $t\to \infty$. To that aim, we first need to show that $\inf_{w\in C^\star(\nu)}\|\bar \omega^\star(t)-w\|_\infty \to 0$, where $\bar \omega^\star(t) \coloneqq \sum_{n=1}^t \alpha_n \omega^\star(n)$ is a convex combination of the estimated optimal allocations up to time $t$.

Begin by defining the following quantities 
    \[\bar v(t) \coloneqq {\rm Proj}_{C^\star(\nu)}(\bar \omega^\star(t)) \hbox{ and }  v(t) \coloneqq {\rm Proj}_{C^\star(\nu)}( \omega^\star(t)),\]
    which are, respectively, the projection onto $C^\star(\nu)$ of the average estimated allocation and the projection of the last estimated allocation.
    
By the forced exploration step, we have that $N_u(t)\to\infty$ for every $u\in V$. Since the model is observable,  we can invoke the law of large number and guarantee that $\mathbb{P}_\nu(\lim_{t\to \infty} \hat \nu(t)=\nu)=1$ in the sense that $(\hat G(t),\hat \mu(t))\to (G,\mu)$ almost surely. Then, by continuity of the problem (see \cref{thm:berge}) we have that $\forall \epsilon >0 \exists t_0(\epsilon): \sup_{\omega\in C^\star(\hat\nu(t))}\|\omega-{\rm Proj}_{C^\star(\nu)}(\omega)\|_\infty \leq \epsilon$ for all $t\geq t_0(\epsilon)$.

Henceforth, for $t\geq t_0(\epsilon)$ we have $\|  \omega^\star(t) -  v(t)\|_\infty\leq \sup_{w\in C^\star(\hat\nu(t))}\|\omega-{\rm Proj}_{C^\star(\nu)}(\omega)\|_\infty \leq \epsilon$, thus we derive
    \begin{align*}
    \left\| \sum_{n=1}^t \alpha_{t,n} v(n) - \bar \omega^\star(n) \right\|_\infty &\leq \sum_{n=1}^{t_0(\epsilon)} \alpha_{t,n}\| v(n) -  \omega^\star(n)\|_\infty + \sum_{n=t_0(\epsilon)+1}^t\alpha_{t,n}\| v(n) -  \omega^\star(n)\|_\infty,\\
    &\leq 
    t_0(\epsilon) \bar \alpha_{t,t_0(\epsilon)}+  \epsilon,\\
    \end{align*}
    where $ \bar \alpha_{t,t_0(\epsilon)}= \max_{1\leq n\leq t_0(\epsilon)} \alpha_{t,n}$ and  we used the fact that $\sum_n \alpha_{t,n}=1$.
    Hence, for any $x\in C^\star(\nu)$ note that $\| \bar \omega^\star(t) - \bar v(t)\|_\infty \leq \|\bar \omega^\star(t)-x\|_\infty$. For a fixed $t$, one can choose $x=\sum_{n=1}^t \alpha_{t,n}  v(n) $ since every $ v(n)\in C^\star(\nu)$, and also a convex combination belongs to $C^\star(\nu)$ by convexity. Henceforth
    \[
    \| \bar \omega^\star(t) - \bar v(t)\|_\infty \leq \left\| \bar \omega^\star(t) - \sum_{n=1}^t \alpha_{t,n} v(n) \right\|_\infty  \leq t_0(\epsilon) \bar \alpha_{t,t_0(\epsilon)}+ \epsilon.
    \]
    Since for every fixed $n$ we have  $\bar \alpha_{t,n} = o(1)$ in $t$, there exists $t_1(\epsilon)$ such that for $t\geq t_1(\epsilon)$ we have $\bar \alpha_{t,t_0(\epsilon)} \leq \epsilon/t_0(\epsilon)$,
    which implies that  $\| \bar \omega^\star(t) - \bar v(t)\|_\infty \leq 2\epsilon$ for $t\geq \max(t_0(\epsilon),t_1(\epsilon))$. 

    Now we prove that $\| N_t/t - n(t)\|_\infty \to 0$ as $t\to \infty$.
    Define $t_2(\epsilon)\coloneqq \max(t_0(\epsilon),t_1(\epsilon))$, and observe  that  $\| N_t/t - n(t)\|_\infty \leq \| N_t/t - \bar v(t)\|_\infty$.

    Therefore we are interested in bounding the quantity $\| N_t/t - \bar v^\star(t)\|_\infty$, and we use similar arguments as in \cite[Lemma 17]{garivier2016optimal}. Let $t\geq t_2(\epsilon)$. Define $E_u(t) = N_u(t) - t \bar v_u(t)$ for every $u\in V$ and note that $\sum_u E_u(t)=0 \Rightarrow \min_u E_u(t) \leq 0$. We want to show that $\sup_u |E_u(t)/t|$ is bounded.

    We begin by showing the following
    \[
    \{U_{t+1}=u\} \subseteq {\cal E}_1(t)\cup{\cal E}_2(t) \subseteq \{E_{u}(t)\leq 2t\epsilon\},
    \]
    where ${\cal E}_1(t) = \{u=\argmin_{u\in V} N_u(t) - t \bar \omega_u^\star(t)\}$ and ${\cal E}_2(t) = \{N_u(t) \leq g(t)\}$ with $g(t)=\max(0,(\sqrt{t}-K/2))-1$.

    For the first part, if $\{U_{t+1}=u\}\subseteq {\cal E}_1(t)$, since $t\geq t_2(\epsilon)$ we have 
    \begin{align*}
    E_u(t) &=  N_u(t) - t \bar v_u(t) \pm t \bar \omega_u^\star(t),\\
    &\leq   N_u(t) - t \bar \omega_u^\star(t) +2t\epsilon,\\
    &=   \min_{v} N_v(t) - t \bar \omega_v^\star(t) +2t\epsilon,\\
    &\leq  \min_{v} E_v(t) +4t\epsilon,\\
    &\leq 4t\epsilon.
    \end{align*}
    where in the second equality we used the fact that $\{U_{t+1}=u\}\subset {\cal E}_1(t)$ and in the last inequality that  $\min_v E_v(t) \leq 0$.

     For the second part,  as shown in \cite[Lemma 17]{garivier2016optimal}, there exists $t_3(\epsilon)$ such that for $t\geq t_3(\epsilon)$ then $g(t)\leq 4t\epsilon$ and $1/t\leq \epsilon$. Then if $\{U_{t+1}=u\}\subseteq {\cal E}_2(t)$ we have that
    \[
    E_u(t) \leq g(t) - t\bar v_u(t) \leq 4t\epsilon.
    \]

    Therefore, as in \cite[Lemma 17]{garivier2016optimal}, one can conclude that for $t\geq t'\coloneqq\max(t_2(\epsilon), t_3(\epsilon))$
    \[
      E_{u}(t) \leq \max(E_{u}(t'), 4t\epsilon+1).
    \]
    Using that $\sum_u E_u(t)=0$, and that for all $t\geq t', E_{u}(t')\leq t'$ and $1/t \leq \epsilon$ we have that
    \[
    \sup_i |E_{u}(t)/t| \leq (K-1) \max(t'/t, 4\epsilon+1/t) \leq (K-1)\max(5\epsilon, t'/t).
    \]
    Hence, there exists $t''\geq t'$ such that for all $t\geq t''$ we have $\sup_i |E_{u}(t)/t| \leq 5(K-1)\epsilon$. Letting $\epsilon\to 0$ concludes the proof.
\end{proof}
Hence, \cref{prop:tracking} follows by choosing $\alpha_{t,n}=1/t$ in the previous proposition. Another possible choice is the exponential smoothing factor $\alpha_{t,n} = \kappa_t \lambda^{t-n}$ with $\lambda \in (0,1)$ and $\kappa_t = \frac{1-\lambda}{1-\lambda^t}$. In the next subsection we investigate which choice of $\alpha_{t,n}$ is better.
\subsubsection{Is the average of the allocations the best convex combination?}

A natural question that arises is which factor $\alpha_{t,n}$ to use. Why is $\alpha_{t,n}=1/t$ a good choice?
This question is related to the fluctuations of the underlying process, and to the stability of the exploration process.
We try to give an answer by looking at the variance of the resulting allocation $\bar w^\star(t)$.

\paragraph{The i.i.d. case.} We begin by considering the i.i.d. case, which supports the fact that a simple average is a good approach to minimize variance. 
\begin{lemma}
    Consider an i.i.d. sequence of Gaussian random variables $\{X_n\}_n$ with $0$ mean and variance  $\sigma^2$. Let $\bar X_t(w) = \sum_{n=1}^t w_n  X_n$, with $\{w_n\}\in \Delta(\{1,\dots, t\})=\Delta([t])$.
    Then
    \[
      \min_{w\in \Delta([t])} {\rm Var}(X_t(w)) = \frac{\sigma^2}{t},
    \]
    which is achieved for $w_i=1/t$, for all $i\in [t]$.
\end{lemma} 
\begin{proof}
    Note that ${\rm Var}(X_t(w)) = \sum_{n=1}^t w_n^2 \sigma^2$. Introduce the Lagrangian ${\cal L}(w,\lambda) =\sum_{n=1}^t w_n^2 \sigma^2 + \lambda(1-\sum_{n=1}^t w_n)$.
    Checking the first order condition yields $d{\cal L}/d w_n= 2w_n \sigma^2 = \lambda$, hence $w_n = \lambda/ (2\sigma^2 )$. Since $\sum_{n} w_n=1$ we must have $t\lambda/(2\sigma^2)=1 \Rightarrow \lambda = 2\sigma^2/t$. Therefore $w_n=1/t$. We conclude that $\min_w {\rm Var}(X_t(w)) = \sigma^2/t$.
\end{proof}
More in general, the weighting should be inversely proportional to the variance of the underlying random variable, according to the inverse variance weighting  principle \cite{hartung2011statistical}. That is, one can use the same approach as in the previous lemma to easily derive that in case $X_n\sim{\cal N}(0,\sigma_n^2)$, then the optimal weighting is $w_n= \frac{1}{\sigma_n^2} \left( \sum_{k=1}^t \frac{1}{\sigma_k^2} \right)^{-1}$.

\paragraph{A more complex case.} While the i.i.d. case seems  to indicate that taking a simple average is a good approach to minimize the variance, it may not always be the case. In fact, we may expect the random variable $ w^\star(t)$ to have smaller fluctuations as $t$ grows larger. We try to give a more complete picture by also looking at  a more complex case.

Consider a process $X_n=X_{n-1}+\xi_n$, where $\xi_n$ is an i.i.d. zero-mean process with variance $\mathbb{E}[\xi_n^2]\leq C/n^{1+\alpha}$  for some $\alpha,C >0$ and $n\geq 2$. And let $X_1=\xi_1$, with $\mathbb{E}[\xi_1]\leq \sigma^2$.

 We define $S_t^{avg}$  to be be the simple average
\[
 S_t^{avg}= \frac{1}{t}(X_1+X_2+\dots +X_t) 
\]

Similarly,  we define the exponentially smoothed average with $\kappa_t = (1-\lambda)/(1-\lambda^t)$:
\begin{align*}
S_t^{exp}&= \kappa_t \sum_{n=1}^t\lambda^{t-n} X_n.
\end{align*}
Then, we obtain the following result on the variance of the two averages.
\begin{lemma}
    Consider the simple average $S_t^{avg}$ and the exponentially smoothed average $S_t^{exp}$ with factor $\lambda\in (0,1)$. Then ${\rm Var}(S_t^{avg}) \leq {\rm Var}(S_t^{exp})  $.
\end{lemma}
\begin{proof}
Let $S_t^{avg}$ be the simple average, and note the following rewriting:
\[
 S_t^{avg}= \frac{1}{t}(X_1+X_2+\dots +X_t) = X_1 + \frac{(t-1)}{t}\xi_2 + \dots + \xi_t = \sum_{n=1}^t \frac{t-n+1}{t} \xi_n.
\]
Also rewrite the exponentially smoothed average as follows
\begin{align*}
S_t^{exp}&= \kappa_t \sum_{n=1}^t\lambda^{t-n} X_n,\\
&=\kappa_t \sum_{n=1}^t\lambda^{t-n}  \sum_{i=1}^n \xi_i,\\
&=\kappa_t \sum_{i=1}^t  \xi_i\sum_{n=i}^t \lambda^{t-n} ,\\
&=\kappa_t \sum_{i=1}^t  \xi_i\sum_{n=0}^{t-i} \lambda^{n} ,\\
&=\kappa_t \sum_{i=1}^t  \xi_i\frac{1-\lambda^{t-i+1}}{1-\lambda} ,\\
&=\frac{1}{1-\lambda^t} \sum_{i=1}^t  (1-\lambda^{t-i+1}) \xi_i.
\end{align*}

 Due to the properties of $X_n$ we have that $\mathbb{E}[\xi_n] = \mathbb{E}[\mathbb{E}[\xi_n|{\cal F}_{n-1}]]=0$ and $\mathbb{E}[\xi_j \xi_n]=0$. Therefore, we can write the variance of the averages as follows:
\begin{align*}
{\rm Var}(S_t^{avg}) &\leq \sigma^2 +C\sum_{n=2}^t \frac{(t-n+1)^2}{t^2 n^{1+\alpha}},\\
{\rm Var}(S_t^{exp}) &\leq \sigma^2+ \frac{C}{(1-\lambda^t)^2}\sum_{n=2}^t  \frac{(1-\lambda^{t-n+1})^2}{n^{1+\alpha}}.
\end{align*}
To show that ${\rm Var}(S_t^{avg}) \leq{\rm Var}(S_t^{exp})$ we can check if the following inequality holds for all $n\in\{2,\dots,t\}$:
\[
\frac{t-n+1}{t} \leq \frac{1-\lambda^{t-n+1}}{1-\lambda^t}.
\]
We can prove that the function $h(x) = \frac{1-\lambda^x}{x}$ is decreasing in $x\in [1,t]$. To that aim, compute the derivative $h'(x) = \frac{-\lambda^x x\ln(\lambda) -1 +\lambda^x}{x^2}$. We are interested in checking if the numerator is negative. Then
\[
-\lambda^x x\ln(\lambda) -1 +\lambda^x \leq 0 \Rightarrow \lambda^x(1- x\ln(\lambda)) \leq 1.
\]
Rewrite as $e^{x\ln(\lambda)}(1- x\ln(\lambda)) \leq 1$ and let $y=-x\ln(\lambda)$. Then
\[
e^{-y}(1+y)\leq 1 \Rightarrow  1+y\leq e^y,
\]
which is always true for $y\geq 0$.

Hence, we have shown that $h(x) \geq h(t)$ for $x\leq t$. Letting $x=t-n+1$, with $n=2,\dots,t$, concludes the proof.
\end{proof}
Unfortunately considering an approach that minimizes the variance is rather difficult. However, numerical experiments seem to suggest that a simple empirical average is an effective approach.

\subsection{Stopping Rule}\label{subsec:app_stopping_rule}
The (fixed-confidence) Best Arm Identification problem in multi-armed bandit models can be seen as a hypothesis testing problem, where we are testing if $\mu \in {\cal H}_k$, where ${\cal H}_k = \{\mu': \mu_{k}' > \max_{j\neq k} \mu_j'\}$. That is, we are testing if the optimal action in $\mu$ is $k$.

Denoting by $n_a(t)$ the number of times the outcome of a certain action $a$ is observed, the generalized likelihood ratio statistics (GLR) for such problem can be written as
\[
\inf_{\lambda \in {\rm Alt}(\hat \mu(t))} \sum_a n_a(t) {\rm KL}(\hat \nu_a(t), \lambda_a),
\]
where $\nu_a(t)$ is the estimated distribution of rewards for arm $a$, which depends solely on $\hat \mu_a(t)$ due to the assumption that $\nu_a$ is a single-parameter exponential distribution, and ${\rm Alt}(\hat \mu(t))=\{\lambda: \argmax_a \lambda_a \neq \argmax_a \}$ is the set of confusing model.

Taking inspiration from such approach, define the following  GLR statistic
\begin{align*}
\Lambda(t) 
\coloneqq \min_{u\neq \hat a_t} \inf_{\lambda: \lambda_u\geq \lambda_{\hat a_t}}\sum_{v\in V} M_v(t) {\rm KL}(\hat \nu_v(t), \lambda_v),
\end{align*}
where $\lambda$ is an alternative reward parameter.
 Then, note that $\Lambda(t)$ can be conveniently rewritten as
\begin{align*}
    \Lambda(t)  &= \min_{u\neq \hat a_t} (M_u(t) +M_{\hat a_t}(t)) I_{\frac{M_{\hat a_t}(t)}{M_u(t) +M_{\hat a_t}(t)}}(\hat \nu_{\hat a_t},\hat \nu_u(t)),\\
    &=\min_{u\neq \hat a_t} M_{\hat a_t}(t) {\rm KL}(\hat \nu_{\hat a_t}(t), \hat \nu_{\hat a_t, u}) + M_{u}(t) {\rm KL}(\hat \nu_{u}(t), \hat \nu_{\hat a_t, u}),
\end{align*}
where $\hat \nu_{\hat a_t, u}$ is a distribution of rewards depending on the parameter $\hat \mu_{\hat a_t,u}(t)$ defined as 
\[
\hat \mu_{a,b}(t) = \frac{M_a(t)}{M_a(t)+M_b(t)} \hat \mu_a(t) + \frac{M_b(t)}{M_a(t)+M_b(t)} \hat \mu_b(t).
\]

One can then show that $tT(N_t/t; \hat \nu(t))^{-1}$ is equivalent to $\Lambda(t)$. First, observe that
\begin{equation*}
     T(N_t/t;\hat \nu(t))^{-1}= \min_{u\neq \hat a_t} (m_u(t) +m_{\hat a_t}(t) )I_{\frac{m_{\hat a_t}(t)}{m_u(t)+m_{\hat a_t}(t)}}(\hat \nu_{\hat a_t}(t),\hat \nu_u(t)),
\end{equation*}
where $m_u(t) = \sum_v \hat G_{v,u}(t) \frac{N_v(t)}{t} = \sum_v \frac{N_{v,u}(t)}{N_v(t)}\frac{N_v(t)}{t}= M_u(t)/t$. Using this latter fact, and noting that $\frac{m_{\hat a_t}(t)}{m_u(t)+m_{\hat a_t}(t)}=\frac{M_{\hat a_t}(t)}{M_u(t) +M_{\hat a_t}(t)}$, we get
\begin{equation*}
     t T(N_t/t;\hat \nu(t))^{-1}= \min_{u\neq \hat a_t} (M_u(t) +M_{\hat a_t}(t) )I_{\frac{M_{\hat a_t}(t)}{M_u(t) +M_{\hat a_t}(t)}}(\hat \nu_{\hat a_t}(t),\hat \nu_u(t)) = \Lambda(t).
\end{equation*}

We can now provide the proof of the stopping rule.

\begin{proof}[Proof of \cref{prop:threshold_error_rate}]
First note that the event $\{\hat a_\tau \neq a^\star(\mu)\}\subset \{\nu \in {\rm Alt}(\hat \nu(\tau))\}$. From the discussion above, using the notation $\Lambda(t)=t T(N_t/t; \hat \nu(t))^{-1}$, we  observe that under  $\{\nu \in {\rm Alt}(\hat \nu(\tau))\}$ then the following inequalities hold
\begin{align*}
    \mathbb{P}_\nu(\tau<\infty, \hat a_\tau \neq a^\star(\mu)) &\leq \mathbb{P}_\nu(\exists t \in \mathbb{N}: \hat a_t \neq a^\star(\mu), t T(N_t/t;\hat \nu(t))^{-1} > \beta(t,\delta)),\\
    &\leq \mathbb{P}_\nu\left(\exists t \in \mathbb{N}:\hat a_t \neq a^\star(\mu),\min_{u\neq \hat a_t} \inf_{\lambda: \lambda_u\geq \lambda_{\hat a_t}}\sum_{v\in V} M_v(t) {\rm KL}(\hat \nu_v(t), \lambda_v) > \beta(t,\delta) \right),\\
    &\leq \mathbb{P}_\nu\left(\exists t \in \mathbb{N},\exists u\neq \hat a_t: \hat a_t \neq a^\star(\mu),\inf_{\lambda: \lambda_u\geq \lambda_{\hat a_t}}\sum_{v\in V} M_v(t) {\rm KL}(\hat \nu_v(t), \lambda_v) > \beta(t,\delta) \right),\\
    &\leq \mathbb{P}_\nu\left(\exists t \in \mathbb{N},\exists u\neq \hat a_t: \sum_{v\in \{u,\hat a_t\}} M_v(t) {\rm KL}(\hat \nu_v(t), \nu_v) > \beta(t,\delta) \right).
\end{align*}
Now, from \citep[Theorem 7]{kaufmann2021mixture}, we know that
\[
\mathbb{P}_\nu\left(\exists t \in \mathbb{N}, \exists u\neq \hat a_t: \sum_{v\in \{u,\hat a_t\}} M_v(t) {\rm KL}(\hat \nu_v(t), \nu_v) > 2  {\cal C}_{\rm exp}\left(\frac{\ln\left(\frac{K-1}{\delta}\right)}{2}\right) + 3\sum_{v\in \{\hat a_t, u\}} \ln(1+\ln(M_v(t)))\right) \leq \delta.
\]
where we applied \citep[Theorem 7]{kaufmann2021mixture} over $K-1$ subsets $\{\underbrace{(u,\hat a_t)}_{{\cal S}_u}\}_{u\neq \hat a_t}$ of size $2$ each, and took a union bound. Finally, using Jensen's inequality we also have that 
\begin{align*}
 \ln(1+\ln(M_v(t)))+ \ln(1+\ln(M_u(t))) &\leq 2\ln\left(\frac{1+\ln(M_v(t))}{2} + \frac{1+\ln(M_u(t))}{2}\right),\\
 &= 2\ln\left(1 + \frac{\ln(M_v(t))+\ln(M_u(t))}{2}\right),\\
  &\leq 2\ln\left(1 + \ln\left(\frac{M_v(t)+M_u(t)}{2}\right)\right).
\end{align*}
Note that $M_v(t)+M_u(t)$ cannot exceed $2t$ (just consider a full feedback graph where $G_{u,v}=1$ so that $M_v(t)=t$ for every $v$). Hence, this implies that the GLR statistics is $\delta$-PC with the threshold
\[\beta(t,\delta)=2  {\cal C}_{\rm exp}\left(\frac{\ln\left(\frac{K-1}{\delta}\right)}{2}\right) + 6\ln(1+\ln(t)).\]
\end{proof}

\paragraph{Definition of ${\cal C}_{\rm exp}(x)$.}
Last, but not least,  we briefly explain the definition of ${\cal C}_{\rm exp}(x)$. We define ${\cal C}_{\rm exp}(x)$ as \citep[Theorem 7]{kaufmann2021mixture} ${\cal C}_{\rm exp}(x)\coloneqq 2\tilde{h}_{3/2}\left(\frac{h^{-1}(1+x)+\ln(2\zeta(2))}{2}\right)$, where: $\zeta(s) = \sum_{n\geq 1} n^{-s}$; $h(u)=u-\ln(u)$ for $u\geq 1$ ; lastly, for for any $z\in [1,e]$ and  $x\geq 0$:
\[
\tilde{h}_z(x) = \begin{cases}
h^{-1}(x)e^{1/h^{-1}(x)}& \hbox{ if } x \geq h(1/\ln z),\\
z(x-\ln\ln z)& \hbox{otherwise.}
\end{cases}
\]
\subsection{Sample Complexity Analysis}\label{subsec:algorithm_as_sample_complexity_app}

The following sample complexity analysis follows the analysis of \cite{garivier2016optimal} while adopting necessary changes for our problem setup. We first prove part (1) and (2) of theorem \ref{thm:sample_complexity}, and prove part (3) of theorem \ref{thm:sample_complexity} separately. In the end of this section, we provide the proof for corollary \ref{thm:sample_complexity_heuristic}.
\begin{proof} [Proof of part (1) and (2) of theorem \ref{thm:sample_complexity}:]
    Let \(\mathcal{E}\) be the event:
    \begin{align*}
        \mathcal{E} = \left\{\inf_{w \in C^{\star}(\nu)}\left \|\frac{N(t)}{t} - \omega\right\|_{\infty} \xrightarrow{t \to \infty} 0, \hat{\nu}(t) \xrightarrow{t \to \infty} \nu\right\}.
    \end{align*}
    By Proposition \ref{prop:tracking} and the law of large number, we have \(\mathcal{E}\) holds with probability 1. On \(\mathcal{E}\), with the continuity property of function \(T(\omega, \nu)^{-1}\) at \((\omega^{\star}(\nu), \nu)\) and proposition \ref{prop:tracking}, for every \(\omega^{\star}(\nu) \in C^{\star}(\nu)\), we have for all \(\epsilon > 0\) there exists \(t_0 \in \mathbb{N}\) such that for all \(t \geq t_0\):\begin{align*}
        T(N(t)/t; \hat{\nu}(t))^{-1} \geq \frac{1}{1+\epsilon} T^{\star} (\nu)^{-1}.
    \end{align*}
    Therefore, for \(t \geq t_0\):
    \begin{align*}
        L(t) = t T(N(t)/t; \hat{\nu}(t))^{-1} \geq \frac{t}{1+\epsilon} T^{\star}(\nu)^{-1}.
    \end{align*}
    Hence,
    \begin{align*}
        \tau &= \inf \left\{t \in \mathbb{N}: L(t) \geq \beta(t, \delta)\right\}, \\
        &\leq t_0 \vee \inf \left\{t \in \mathbb{N}: \frac{t}{1+\epsilon} T^{\star}(\nu)^{-1} \geq \beta(t, \delta)\right\}.
    \end{align*}
    Recall that \(\beta(t,\delta)\coloneqq2  {\cal C}_{\rm exp}\left(\frac{\ln\left(\frac{K-1}{\delta}\right)}{2}\right) + 6\ln(1+\ln(t))\). Note that there exists a universal constant \(B\) such that \(\beta(t, \delta) \leq \ln (Bt/\delta)\). Hence,
    \begin{align*}
        \tau \leq t_0 \vee \inf \left\{t \in \mathbb{N}: \frac{t}{1+\epsilon} T^{\star}(\nu)^{-1} \geq \ln(Bt/\delta)\right\}.
    \end{align*}
    Applying Lemma 18 of \cite{garivier2016optimal} by letting \(\alpha = 1\):
    \begin{align*}
        \tau \leq t_0 \vee (1 + \epsilon) T^{\star}(\nu) \left[\ln \left(\frac{Be(1+\epsilon)T^{\star}(\nu)}{\delta} \right) + \ln \ln \left(\frac{B(1+\epsilon)T^{\star}(\nu)}{\delta}  \right) \right].
    \end{align*}
    Thus, \(\tau\) is finite with probability 1. And
    \begin{align*}
        \limsup_{\delta \to 0} \frac{\tau}{\ln(1/\delta)} \leq (1 + \epsilon) T^*(\nu).
    \end{align*}
    We conclude the proof of part (2) by letting \(\epsilon \to 0\).
\end{proof}
\begin{proof} [Proof of part (3) of theorem \ref{thm:sample_complexity}:]

    Let \(T \in \mathbb{N}\), for \(\epsilon > 0\), define \(\mathcal{E}_T \coloneqq \bigcap_{t=T^{\frac{1}{4}}}^{T} (\hat{\nu}(t) \in \mathcal{I}_{\epsilon})\), where \(\mathcal{I}_{\epsilon} \coloneqq \{\nu': \|\nu' - \nu\|_{\infty} \leq \epsilon\}\) and \(\|\nu' - \nu\|_{\infty} \coloneqq \max\{\|G' - G\|_{\infty}, \|\mu' - \mu\|_{\infty}\}\). Following the same argument as Lemma 19 of \cite{garivier2016optimal}, one can show that there exist two constant \(B\) and \(C\) (that depend on \(\nu\) and \(\epsilon\)) such that \(\mathbb{P}_{\nu}(\mathcal{E}_T^c) \leq B T \exp(-C T^{\frac{1}{8}})\).
    Denote
    \begin{align*}
        C_{\epsilon}^{\star}(\nu) = \inf_{\substack{\omega' : \|\omega' - {\rm Proj}_{C^{\star}(\nu)}(\omega')\|_{\infty} \leq 5(K-1)\epsilon\\\nu' : \|\nu' - \nu\|_{\infty} \leq \epsilon}} T(\omega', \nu')^{-1}.
    \end{align*}
    By proposition \ref{prop::tracking_general}, for any \(\epsilon\), there exists \(T(\epsilon)\) such that for any \(T \geq T(\epsilon)\) and \(t \geq \sqrt{T}\), \(\|N(t)/t - {\rm Proj}_{C^{\star}(\nu)}(N(t)/t)\|_\infty \leq 5(K-1)\epsilon\). With this fact, on \(\mathcal{E}_T\), for \(T \geq T(\epsilon)\) and \(t \geq \sqrt{T}\), one has:
    \begin{align*}
        L(t) = t T(N(t)/t; \hat{\nu}_t)^{-1} \geq tC_{\epsilon}^{\star}(\nu).
    \end{align*}
    Therefore, let \(T \geq T(\epsilon)\), on \(\mathcal{E}_T\),
    \begin{align*}
        \min \{\tau_{\delta}, T\} &\leq \sqrt{T} + \sum_{t=\sqrt{T}}^T \mathbf{1}_{(\tau_{\delta} > t)}, \\
        &= \sqrt{T} + \sum_{t=\sqrt{T}}^T \mathbf{1}_{(L(t) \leq \beta(t, \delta))}, \\
        &\leq \sqrt{T} + \sum_{t=\sqrt{T}}^T \mathbf{1}_{(tC_{\epsilon}^{\star}(\nu) \leq \beta(t, \delta))}, \\
        &\leq \sqrt{T} + \sum_{t=\sqrt{T}}^T \mathbf{1}_{(tC_{\epsilon}^{\star}(\nu) \leq \beta(T, \delta))}, \\
        &\leq \sqrt{T} + \frac{\beta (T, \delta)}{C_{\epsilon}^{\star}(\nu)}.
    \end{align*}
    Denote \(T_0 = \inf\left\{T \in \mathbb{N}: \sqrt{T} + \frac{\beta (T, \delta)}{C_{\epsilon}^{\star}(\nu)} \leq T\right\}\). Thus, one has for \(T \geq \max\{T_0, T(\epsilon)\}\), \(\mathcal{E}_T \subseteq (\tau_{\delta} \leq T)\). Hence,
    \begin{align*}
        \mathbb{E[\tau_{\delta}]} &\leq \max\{T_0, T(\epsilon)\} + \sum_{T = \max\{T_0, T_{\epsilon}\}}^{\infty} \mathbb{P}(\tau_{\delta} > T), \\
        &\leq T_0 + T(\epsilon) + \sum_{T = \max\{T_0, T_{\epsilon}\}}^{\infty} B T \exp(-CT^{\frac{1}{8}}).
    \end{align*}
    We then upper bound \(T_0\). By introducing a constant \(C(\eta) = \inf \left\{T \in \mathbb{N}: T - \sqrt{T} \geq \frac{T}{1 + \eta}\right\}\), one has
    \begin{align*}
        T_0 &\leq C(\eta) + \inf \left\{T \in \mathbb{N}: \frac{\beta (T, \delta)}{C_{\epsilon}^{\star}(\nu)} \leq \frac{T}{1 + \eta}\right\}, \\
        &\leq C(\eta) + \inf \left\{T \in \mathbb{N}: \frac{C^{\star}_{\epsilon}(\nu) T}{1 + \eta} \geq \ln(BT/\delta)\right\}.
    \end{align*}
    Applying Lemma 18 of \cite{garivier2016optimal} again:
    \begin{align*}
        T_0(\delta) \leq C(\eta) + (1 + \eta) C_{\epsilon}^{\star}(\nu)^{-1} \left[\ln \left(\frac{Be(1+\eta)}{C_{\epsilon}^{\star}(\nu) \delta} \right) + \ln \ln \left(\frac{B(1+\eta)}{C_{\epsilon}^{\star}(\nu) \delta}  \right) \right].
    \end{align*}
    Therefore,
    \begin{align*}
        \limsup_{\delta \to 0} \frac{\mathbb{E}[\tau_{\delta}]}{\ln(1/\delta)} \leq \frac{(1+\eta)}{C^{\star}_{\epsilon}(\nu)}.
    \end{align*}
    From the continuity property of function \(T(\omega, \nu)^{-1}\) at \((\omega^{\star}(\nu), \nu)\) for each \(\omega^{\star}(\nu)\) in \(C^{\star}(\nu)\), one has
    \begin{align*}
        \lim_{\epsilon \to 0} C^{\star}_{\epsilon} (\nu) = T^{\star} (\nu)^{-1},
    \end{align*}
    Letting \(\eta\) go to 0:
    \begin{align*}
        \limsup_{\delta \to 0} \frac{\mathbb{E}[\tau_{\delta}]}{\ln(1/\delta)} \leq T^{\star}(\nu).
    \end{align*}
\end{proof}

\begin{proof} [Proof of Corollary \ref{thm:sample_complexity_heuristic}]
    The proof relies on the following proposition, which is a direct application of Proposition \ref{prop::tracking_general}.

    \begin{proposition} \label{prop::tracking_heur}
    Let $S_t= \{u\in V: N_u(t) < \sqrt{t}-K/2\}$. The D-tracking rule, defined as
    \begin{equation}
        V_t \in \begin{cases}
            \argmin_{u\in S_t} N_u(t) & S_t\neq \emptyset\\
            \argmin_{u\in V} N_u(t) - t\sum_{n=1}^t \alpha_{t,n}\omega_{\rm heur}(n)& \hbox{otherwise}
        \end{cases},
    \end{equation}
    where, for every $t\geq 1$, the sequence $\alpha_t=(\alpha_{t,n})_{n=1}^t $ satisfies: (1) $\alpha_{t,n}\in [0,1]$; (2) for every fixed $n\in\{1,\dots,t\}$ we have  $\alpha_{t,n} =o(1)$ in $t$ ; (3) for all $t$, $\sum_{n=1}^t \alpha_{t,n}=1$.

    Such tracking rule
    ensures that for all $\epsilon>0$ there exists $t(\epsilon)$ such that for all $t \geq t(\epsilon)$ we have
    \[
    \|N(t)/t - \omega_{\rm heur}\|_\infty \leq 5(K-1)\epsilon,
    \]
    and 
    $
   \lim_{t\to\infty}\| N(t)/t -\omega_{\rm heur}\|_\infty 
 \to 0 
    $ almost surely.
\end{proposition}

\begin{proof}
    Proposition \ref{prop::tracking_heur} can be proved using the same analysis as in Proposition \ref{prop::tracking_general}, except that one needs to replace \(C^{\star}(\nu)\) with \(\left\{\omega_{\rm heur}\right\}\). We thus skip the full proof.
\end{proof}

Corollary \ref{thm:sample_complexity_heuristic} can then be proved using the same analysis as in Theorem \ref{thm:sample_complexity} combined with Proposition \ref{prop::tracking_heur}.
\end{proof}

We also provide an almost-surely lower bound for this heuristic algorithm, as established in the following Proposition \ref{prop::lower_bound_heur}.
\begin{proposition} \label{prop::lower_bound_heur}
    \algoname{} with $\omega^\star(t) = \omega_{\rm heur}(t)$ satisfies that $\mathbb{P}_\nu\left(\liminf_{\delta \to 0} \frac{\tau}{\ln(1/\delta)} \geq T^{\star}(\nu)\right)=1$.
\end{proposition}
\begin{proof}
     By the continuity property of function \(T(\omega, \nu)^{-1}\) at \((\omega_{\rm heur}, \nu)\) and Proposition \ref{prop::tracking_heur}, for any \(\epsilon > 0\), with probability 1, there exists \(t_1\) such that for any \(t \geq t_1\),
    \begin{align*}
        L(t) = t T(N(t)/t; \hat{\nu}(t))^{-1} \leq \frac{t}{1-\epsilon} T(\omega_{\rm heur},\nu)^{-1} \leq \frac{t}{1-\epsilon}  T^{\star}(\nu)^{-1}.
    \end{align*}
    By the definition of the stopping time,
    \begin{align*}
        \tau &= \inf \left\{t \in \mathbb{N}: L(t) \geq \beta(t, \delta)\right\}, \\
        &\geq \inf \left\{t \in \mathbb{N}: L(t) \geq \ln (1/\delta)\right\}.
    \end{align*}
    When \(\delta \to 0\), \(\inf \left\{t \in \mathbb{N}: L(t) \geq \ln (1/\delta)\right\} > t_1\). Therefore,
    \begin{align*}
        \liminf_{\delta \to 0} \frac{\tau}{\ln (1/\delta)} &\geq \liminf_{\delta \to 0} \frac{\inf \left\{t \in \mathbb{N}: t T^{\star}(\nu)^{-1} \geq (1-\epsilon)\ln (1/\delta)\right\}}{\ln (1/\delta)}, \\
        &\geq T^{\star}(\nu) (1-\epsilon).
    \end{align*}

    Letting \(\epsilon\) go to 0 concludes the proof.
\end{proof}
\end{document}